\documentclass{article}

\usepackage[main, final]{neurips_2026}

\usepackage[utf8]{inputenc} 
\usepackage[T1]{fontenc}    
\usepackage{url}            
\usepackage{booktabs}       
\usepackage{amsfonts}       
\usepackage{nicefrac}       
\usepackage{microtype}      
\usepackage{xcolor}         

\usepackage{pifont}
\usepackage{fontawesome5}

\usepackage[colorlinks, urlcolor=mydarkred, linkcolor=mydarkred, citecolor=gray]{hyperref}
\definecolor{mydarkred}{rgb}{0.75,0,0}

\usepackage[table]{xcolor}
\usepackage{graphicx}
\usepackage{mathtools}
\usepackage{amssymb}
\usepackage{bbm}
\usepackage{amsmath}
\usepackage{amsthm}
\usepackage{tabularx}
\usepackage{multirow, threeparttable, booktabs, makecell, adjustbox}
\usepackage{caption}
\usepackage{tcolorbox}
\usepackage{graphicx}
\usepackage{subcaption}
\newtheorem{theorem}{Theorem}[section]

\newtheorem{corollary}{Corollary}[section]
\newtheorem{lemma}[theorem]{Lemma}
\newcommand\bs\boldsymbol
\tcbuselibrary{theorems}
\usepackage{enumitem}
\usepackage[T1]{fontenc}
\usepackage{microtype}
\usepackage{inconsolata}
\usepackage{graphicx}
\usepackage{enumitem}
\usepackage{wrapfig}
\usepackage{makecell}
\usepackage{tabularx}
\newcounter{mybox}
\renewcommand{\themybox}{\thesection.\arabic{mybox}}

\newtcolorbox{mybox}[2][]{
  colback=gray!5,
  colframe=violet,
  fonttitle=\bfseries,
  phantom={\refstepcounter{mybox}\label{#2}},
  title={Box~\themybox\if\relax\detokenize{#1}\relax\else: #1\fi},
}

\title{Cyclical Entropy Eruption: Entropy Dynamics of Agent Reinforcement Learning}

\author{%
  Wendi Li \quad Shawn Im \quad Sharon Li \\
  Department of Computer Sciences\\
  University of Wisconsin-Madison\\
  \texttt{\{wli679, shawnim, sharonli\}@cs.wisc.edu} \\
}

\begin{document}

\maketitle

\begin{abstract}
Agentic large language models are increasingly used to solve real-world tasks by reasoning over goals, invoking tools, and interacting with external environments. Reinforcement learning provides a natural framework for improving these behaviors, and recent agent RL methods have achieved strong results across domains. However, the training dynamics of agent RL remain poorly understood, limiting our ability to diagnose instabilities and design more effective training algorithms.
In this work, we identify a previously underexplored phenomenon in agent RL, which we term \emph{cyclical entropy eruption}. Unlike single-turn reasoning RL, where entropy typically collapses and stays low, agent RL training exhibits unique recurring cycles of sharp entropy eruption and gradual subsidence. We decompose this dynamic into three phases and provide theoretical and empirical analyses of each, explaining the mechanisms underlying its cyclical oscillation. We further show that degenerate patterns such as sentence duplication and hallucination, once acquired during eruption, can persist and accumulate across cycles. Motivated by these findings, we propose SEAL (Separation-Enhanced Agent Learning), a lightweight auxiliary loss that separates correct and incorrect trajectories in representation space, directly targeting the root cause of entropy eruption. Experiments across multiple benchmarks, models, and RL algorithms demonstrate that SEAL stabilizes training and yields stronger downstream agent performance. Code:~\href{https://github.com/WindyLee0822/SEAL}{\faGithub}
\end{abstract}

\vspace{-0.5em}
\section{Introduction}

Large language model agents are rapidly emerging as a practical interface between foundation models and real-world tasks, enabling models to reason over goals, invoke tools, and interact with external environments. This capability has already created value in settings such as web navigation \citep{web1,web2,web3}, software engineering \citep{sweagent,swe1,swe2,swe3}, scientific assistance \citep{sci1,sci2,sci3}, and embodied decision making \citep{emb1,emb2,emb3}, where agents can decompose complex objectives into actionable intermediate steps.
Reinforcement learning for agents provides a natural framework for improving such behaviors, since it optimizes long-horizon decision quality. Consequently, a growing class of agent RL algorithms has emerged \citep{searchr1,webagentr1,gigpo}, many of which adapt techniques originally developed for LLM reasoning by maximizing advantage-weighted objectives over trajectories. The policy can learn to favor action sequences that improve planning, tool use, and error correction.

\begin{figure}[t]
    \centering
    \includegraphics[width=\linewidth]{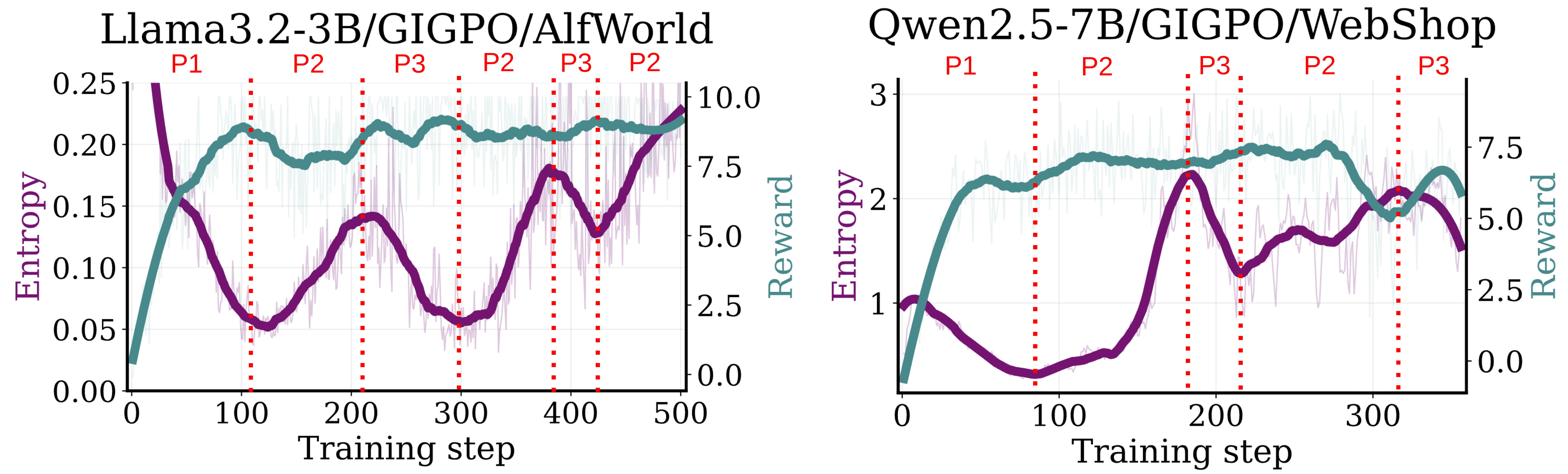}
    \vspace{-1.5em}
    \caption{
    Cyclical entropy eruption in agent RL across models, tasks, and algorithms. P1, P2, and P3 represent three phases, respectively. 
    See more examples in Figure~\ref{fig:entropy-eruption-apd}.}
    \vspace{-1em}
    \label{fig:entropy-eruption}
\end{figure}
Despite the rapid progress of agent RL, our understanding of its training dynamics remains limited. Most existing work focuses on improving benchmark performance or designing new optimization objectives, while offering relatively little insight into how the policy distribution evolves throughout training. Moreover, training-dynamics analyses developed for single-turn LLM RL \citep{entropy_mechanism_llm,entropycontrol_llmrl} may not directly transfer to agent settings, where learning unfolds over long-horizon trajectories with structured actions, tool interactions, and validity constraints. Understanding these dynamics is important not only for explaining empirical successes and failures, but also for guiding algorithm design and diagnosing instability in increasingly complex agent environments. 

In this work, we provide an in-depth understanding of agent RL by identifying and analyzing a previously unexplored phenomenon, which we term \emph{\textbf{cyclical entropy eruption}}. As shown in Figure~\ref{fig:entropy-eruption}, the policy entropy in agent RL does not simply decrease and stabilize as in single-turn reasoning RL~\citep{entropy_mechanism_llm, entropycontrol_llmrl}; instead, it undergoes unique recurring cycles of eruption and subsidence throughout training. We decompose this dynamic into three phases:
\ding{182} \emph{Entropy descent}: the entropy descends at the beginning phase of training, where the policy model learns rapidly about how to use tools, call functions, and output in the correct format. 
\ding{183} \emph{Entropy eruption}: an entropy surge phase in which high representation similarity among trajectories causes gradient interference that suppresses valid-trajectory likelihood. As a result, this flattens the policy distribution and amplifies degenerate patterns such as sentence duplication and hallucination.  
\ding{184} \emph{Entropy subsidence}: As the distribution flattens, the policy samples more diverse trajectories, which in turn reduces representation similarity. As a result, the training gradually reinforces correct responses by increasing the likelihood of correct trajectories, which then triggers the next eruption. The dynamics therefore form a self-perpetuating cycle: convergence raises similarity $\rightarrow$ similarity causes eruption $\rightarrow$ eruption increases diversity $\rightarrow$ diversity enables subsidence $\rightarrow$ and subsidence leads back to convergence.

Motivated by our theoretical understanding of entropy dynamics, we propose \textsc{SEAL} ({S}eparation-{E}nhanced {A}gent {L}earning), a lightweight modification to the original agentic RL objective, with no additional cost at inference time. The objective encourages the model to separate correct and incorrect trajectories in representation space. 
Therefore, harmful gradient interference is reduced: incorrect and correct trajectories become less likely to induce conflicting updates when their representations are no longer highly similar, which mitigates the likelihood decrease of valid trajectories and helps suppress entropy eruption. While the primary contribution of this paper is the dynamical analysis, SEAL demonstrates that our theoretical understanding can guide practice: even a simple intervention motivated by the analysis meaningfully stabilizes training, improves valid-trajectory likelihood, and yields stronger downstream agent performance across different tasks, backbones, and RL algorithms. 
For example, SEAL improves the average AlfWorld accuracy of Qwen2.5-7B with GRPO by 2.81\% and raises its WebShop success rate by 3.13\%. Moreover, on the Llama-series model, where GRPO training collapses and yields a 0\% success rate, adding SEAL recovers training and boosts performance to 79.69\%. 
We expect future work to deepen this line of inquiry with more sophisticated methods. 
Our main contributions are: 
\begin{enumerate}[leftmargin=*,itemsep=0.pt,topsep=0.5pt]
    \item We identify \emph{cyclical entropy eruption}, a recurring training instability in agent RL in which policy entropy repeatedly erupts and subsides. This phenomenon is qualitatively not observed in single-turn reasoning RL.
    \item We provide in-depth theoretical and empirical analyses of each phase, explaining the mechanisms underlying its cyclical oscillation.
    \item We propose \textsc{SEAL}, an agent RL  loss motivated directly by our analysis, and demonstrate that it stabilizes training and improves agent performance across multiple benchmarks, models, and RL algorithms.
\end{enumerate}
\section{Preliminary}
\label{sec:prelim}

\paragraph{LLM agents and structured trajectories.}
Given a task prompt~$x$, an LLM agent autoregressively generates a \emph{trajectory} $y = (y_1, \dots, y_T)$ that may contain intermediate reasoning, tool calls, returned observations, and a final answer. Because such tasks typically cannot be completed from parametric knowledge alone, agents invoke external tools such as search engines, calculators, code interpreters, databases, or domain-specific APIs. Tool use is usually implemented through structured action formats, \emph{e.g.} JSON arguments, XML-style tags, or function-calling schemas. Dedicated special tokens demarcate each stage of the trajectory, enabling the environment to parse and execute tool calls. A trajectory~$y$ is called {valid} if every tool call conforms to the required schema and formatting constraints (correct syntax, expected argument types, proper use of special tokens).
We use $r_{\mathrm{fmt}}(y)\in\{0,1\}$ as the indicator of format validity. 
A trajectory~$y$ is called {semantically correct} if the trajectory produces the right final answer. We use $r_{\mathrm{sem}}(y)\in\{0,1\}$ as the indicator of semantic correctness. 

\paragraph{Agent reinforcement learning.}
Reinforcement learning provides a natural objective for improving the agentic capabilities of LLMs, since it can directly optimize long-horizon task success rather than next-token prediction alone. In particular, RL enables the model to learn effective reasoning, tool selection, and action sequencing from trajectory-level feedback.
For a given prompt~$x$, let $y_1,\dots,y_G$ denote a group of $G$ trajectories sampled independently from the current policy $\pi_\theta(\cdot\mid x)$. Each trajectory $y_i$ is assigned a normalized advantage $A_i$ computed from the within-group rewards, satisfying $\sum_{i=1}^{G} A_i = 0$ in the standard centered case. The RL training objective is formulated as
\begin{equation}\label{eq:rl-obj}
  \mathcal{L}_{\mathrm{RL}}(\theta)
  \;=\;
  -\sum_{i=1}^{G} A_i\,\ell_i(\theta),
  \qquad
  \ell_i(\theta) := \log\pi_\theta(y_i\mid x)
  = \sum_{k=1}^{|y_i|} \log\pi_\theta\!\bigl(y_{i,k}\mid x,\,y_{i,<k}\bigr).
\end{equation}
Minimizing $\mathcal{L}_{\mathrm{RL}}$ encourages the policy to increase the likelihood of trajectories whose advantage exceeds the group mean and to suppress those that fall below it. This within-group relative update underlies several recent agent RL algorithms, including GRPO~\citep{grpo} and GIGPO~\citep{gigpo}, and serves as the foundation for the training-dynamics analysis in Section~\ref{sec:dynamics}.

\section{Cyclical Entropy Eruption in Agent Reinforcement Learning}
\label{sec:dynamics}
Despite rapid progress in agent RL, relatively little attention has been paid to \emph{how} the policy distribution evolves over the course of training.  Understanding the dynamics matters practically: entropy dynamics govern the exploration-exploitation balance during training, and pathological entropy behavior can cause reward stagnation, the reinforcement of degenerate outputs, or outright training collapse.  In this section, we identify a special and recurring entropy pattern in agent RL that we term \emph{cyclical entropy eruption}.

\subsection{Overview: The Cyclical Entropy Eruption Phenomenon}

\paragraph{Experimental setup.}
We conduct agent RL training with GRPO~\citep{grpo} and GIGPO~\citep{gigpo} on two tasks, AlfWorld~\citep{alfworld} and WebShop~\citep{webshop}\footnote{The implementation details follow \texttt{Agent-RL}~\citep{gigpo} with the \texttt{VeRL} ~\citep{verl} framework, detailed in Section~\ref{sec:experimental-settings} and Appendix~\ref{apd:sec:implementation-details}}. AlfWorld is a text-based embodied benchmark for evaluating LLM agents on multi-step household tasks, where agents must interpret goals and interact through multiple turns to complete activities such as picking, cleaning, heating, cooling, and placing objects. WebShop is a realistic web-interaction benchmark that tests LLM agents in online shopping scenarios, requiring them to search, browse, compare, and purchase products on a simulated e-commerce website with a large and diverse product catalog. 
Additional results on a search-augmented QA task~\citep{searchr1} and experiments with other backbone models are in Appendix~\ref{apd:sec:entropy-eruption}.

\paragraph{Cyclical entropy eruption in agent RL.}

As shown in Figure~\ref{fig:entropy-eruption}, the policy entropy in agent RL often undergoes a striking non-monotonic evolution: after an initial decrease, the entropy erupts sharply and then enters a recurring cycle of eruption and subsidence throughout training. This behavior is qualitatively different from the entropy dynamics typically observed in single-turn reasoning tasks such as math \citep{entropy_mechanism_llm, entropycontrol_llmrl}. Empirically, entropy eruption is often accompanied by reward stagnation or even degradation (Figure~\ref{fig:entropy-eruption}), as well as the emergence of toxic generation patterns such as sentence duplication and the incorrect use of format-specific special tokens (Figure~\ref{fig:toxic}). Understanding this phenomenon is therefore important for achieving stable and continual policy improvement in agent RL. 
We divide the entropy dynamics into three phases: (1) \emph{entropy descent},  (2) \emph{entropy eruption}, and (3) \emph{entropy subsidence}. 
The remainder of this section analyzes each phase in turn, 
and explains the mechanism underlying its cyclical oscillation.

\subsection{Phase 1: Entropy Descent and Validity Surge}

In the first phase of training, agent RL exhibits a sharp entropy decrease as the policy rapidly reallocates probability mass toward a narrower set of valid trajectories. 
We explain this below. 

\begin{wrapfigure}{r}{0.4\textwidth}
    \centering
    \vspace{-1em}
    \includegraphics[width=0.39\textwidth]{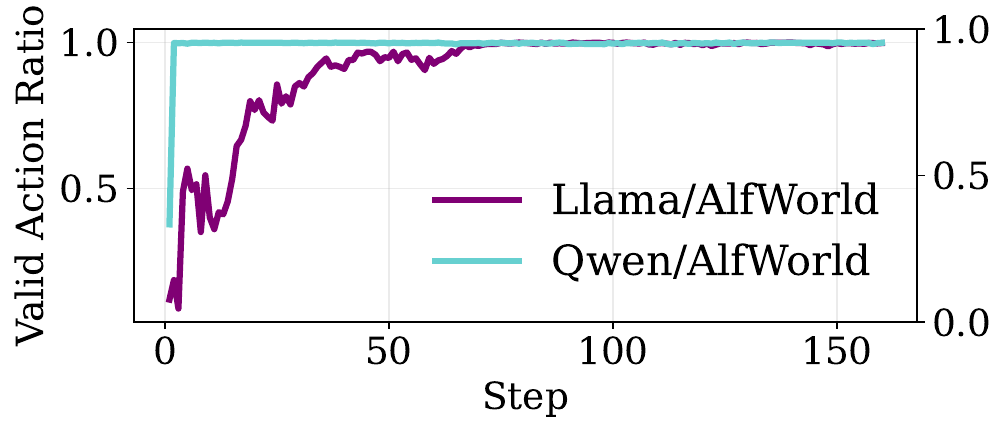}
    \vspace{-1em}
    \includegraphics[width=0.39\textwidth]{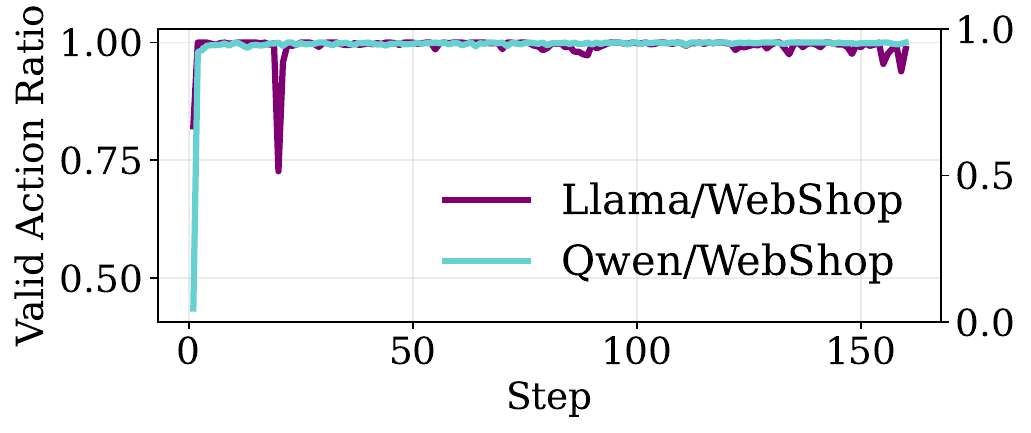}
    \caption{Valid action ratio surges at the beginning phase of agent RL.}
    \label{fig:validaction}
    \vspace{-1em}
\end{wrapfigure}
\paragraph{Format learning dominates early training.} We measure the fraction of sampled trajectories that satisfy formatting constraints and perform valid actions. As shown in Figure~\ref{fig:validaction}, we observe that the fraction of sampled trajectories that satisfy format-validity constraints surges and quickly saturates in the opening steps of training. This indicates that the policy first acquires protocol-level competence, including how to invoke tools and functions correctly, and how to satisfy schema and formatting constraints. This behavior is natural because while semantic competence can be partially inherited from pretraining, agent-specific formatting conventions—such as tool-call syntax, schema tokens, and special control markers—are often absent or weakly represented in the pretrained distribution. Consequently, the model begins RL with a much larger deficiency in format validity than in semantic understanding, so early gradient updates are dominated by correcting this format bottleneck, which in turn suppresses further gains from semantic learning until valid trajectories become sufficiently likely.
We formalize this intuition into Lemma~\ref{lem:format_gated_learning}, which shows that when the policy is weaker on agent-specific formatting than on semantic correctness, RL updates are dominated by increasing the format validity of trajectories. 
As a result, probability mass is rapidly shifted from a diverse set of malformed or invalid trajectories to a relatively concentrated subset of valid ones. We also found that mid-training can shrink phase 1 since mid-training already instills the basic formatting and validity constraints. We analyze this phenomenon in Appendix~\ref{apd:sec:mid-training}.

\subsection{Phase 2: Entropy Eruption}
\label{sec:eruption-phase2}

After the policy acquires format validity in Phase~1, training enters a regime in which entropy rises sharply rather than continuing to fall (see Figure~\ref{fig:entropy-eruption}). This behavior is distinctive to agent RL and stands in contrast with single-turn reasoning RL, where entropy typically remains low after its initial collapse.  Below, we explain this unique phenomenon by two linked mechanisms: high representation similarity among trajectories and the gradient interference it induces.

\begin{figure}[t!]
    \centering
    \includegraphics[width=0.32\linewidth]{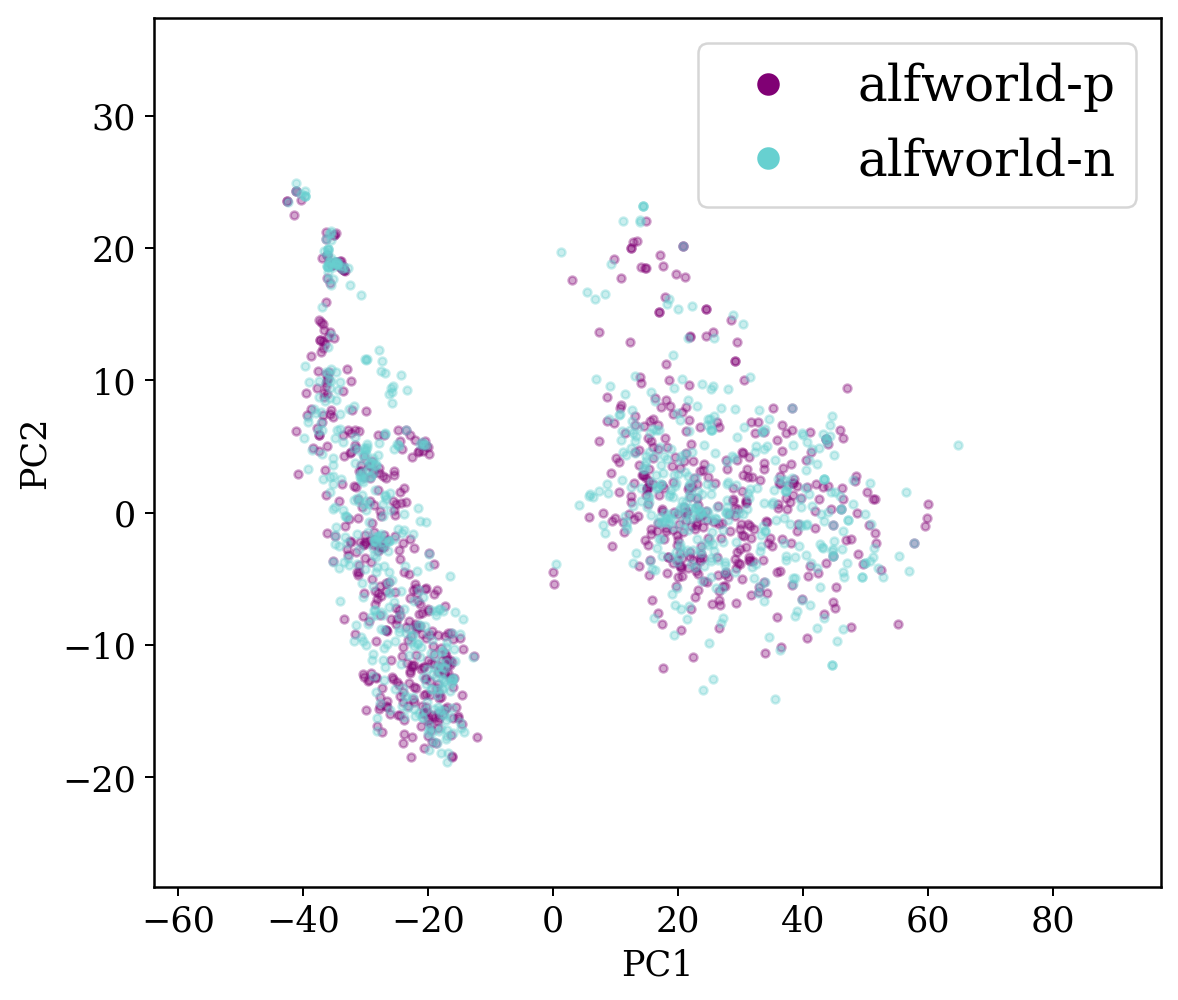}
    \includegraphics[width=0.32\linewidth]{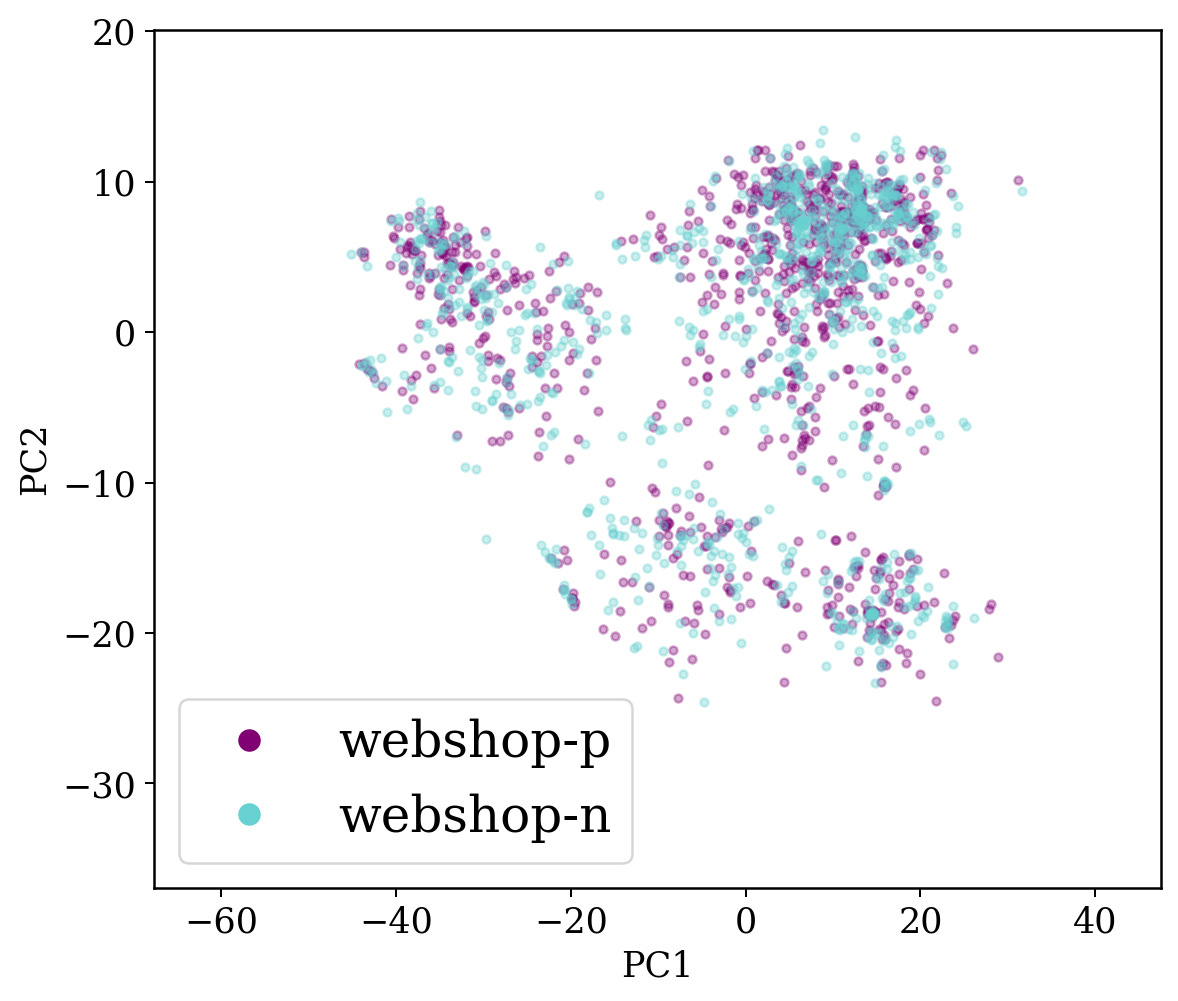}
    \includegraphics[width=0.32\linewidth]{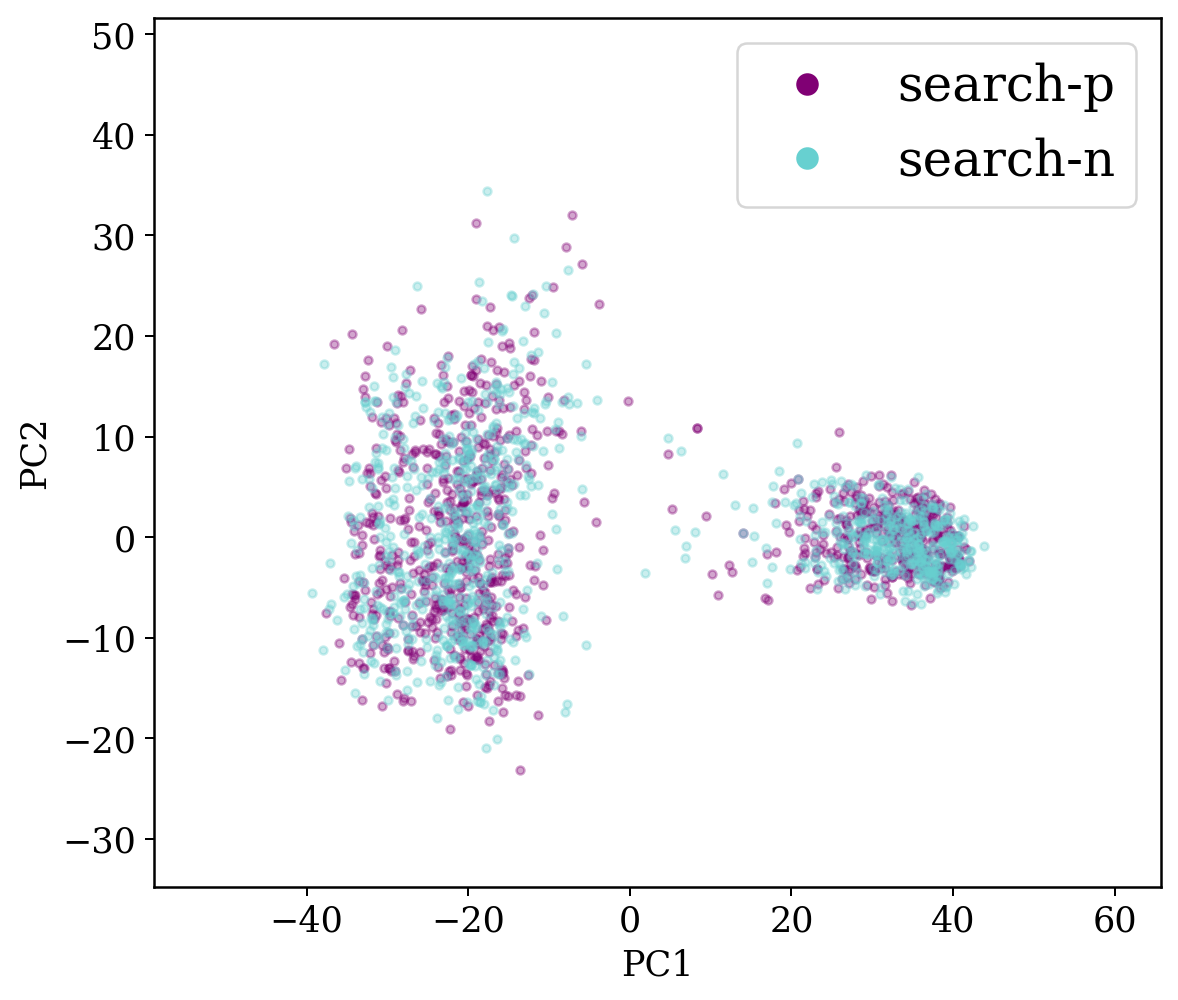}
    \vspace{-1em}
    \caption{Representation distribution of \textcolor[HTML]{800074}{correct trajectories} and \textcolor[HTML]{67d0d0}{wrong trajectories} in three agent tasks. -p and -n denote correct and wrong trajectories in the figures.}
    \vspace{-1em}
    \label{fig:geometry}
\end{figure}

\paragraph{Agent trajectories cluster together in the representation space.}

We begin by showing that \emph{correct} and \emph{wrong} agent trajectories are not well separated in the LLM representation space. Figure~\ref{fig:geometry} visualizes the last-layer hidden representations of Qwen-2.5-7B, averaged over response tokens, for sampled trajectories from three agent tasks: AlfWorld, WebShop, and Search-augmented QA. Across all three tasks, correct and wrong trajectories exhibit substantial overlap rather than forming clearly separated clusters, indicating strong representation similarity between positively and negatively rewarded trajectories. In Appendix~\ref{apd:sec:representation-similarity}, we further show that this representation overlap is markedly more pronounced in agent tasks than in other LLM tasks.
This high similarity results in important optimization consequences: negatively rewarded trajectories can interfere with correct trajectories during RL updates, thereby reducing the likelihood of valid trajectories and contributing to the onset of entropy eruption.

\begin{figure}[t!]
    \centering
    \includegraphics[width=0.42\linewidth]{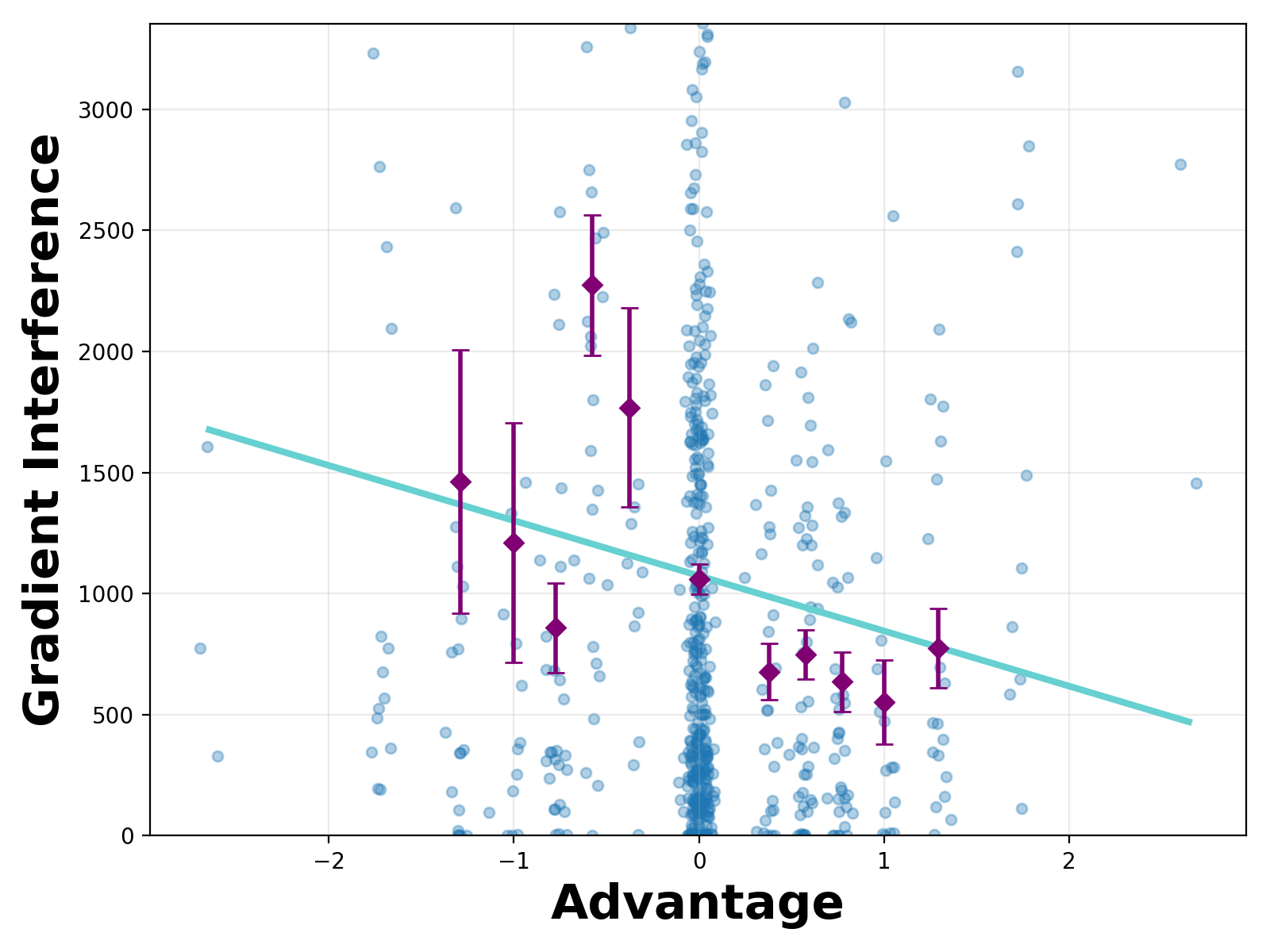}
    \includegraphics[width=0.42\linewidth]{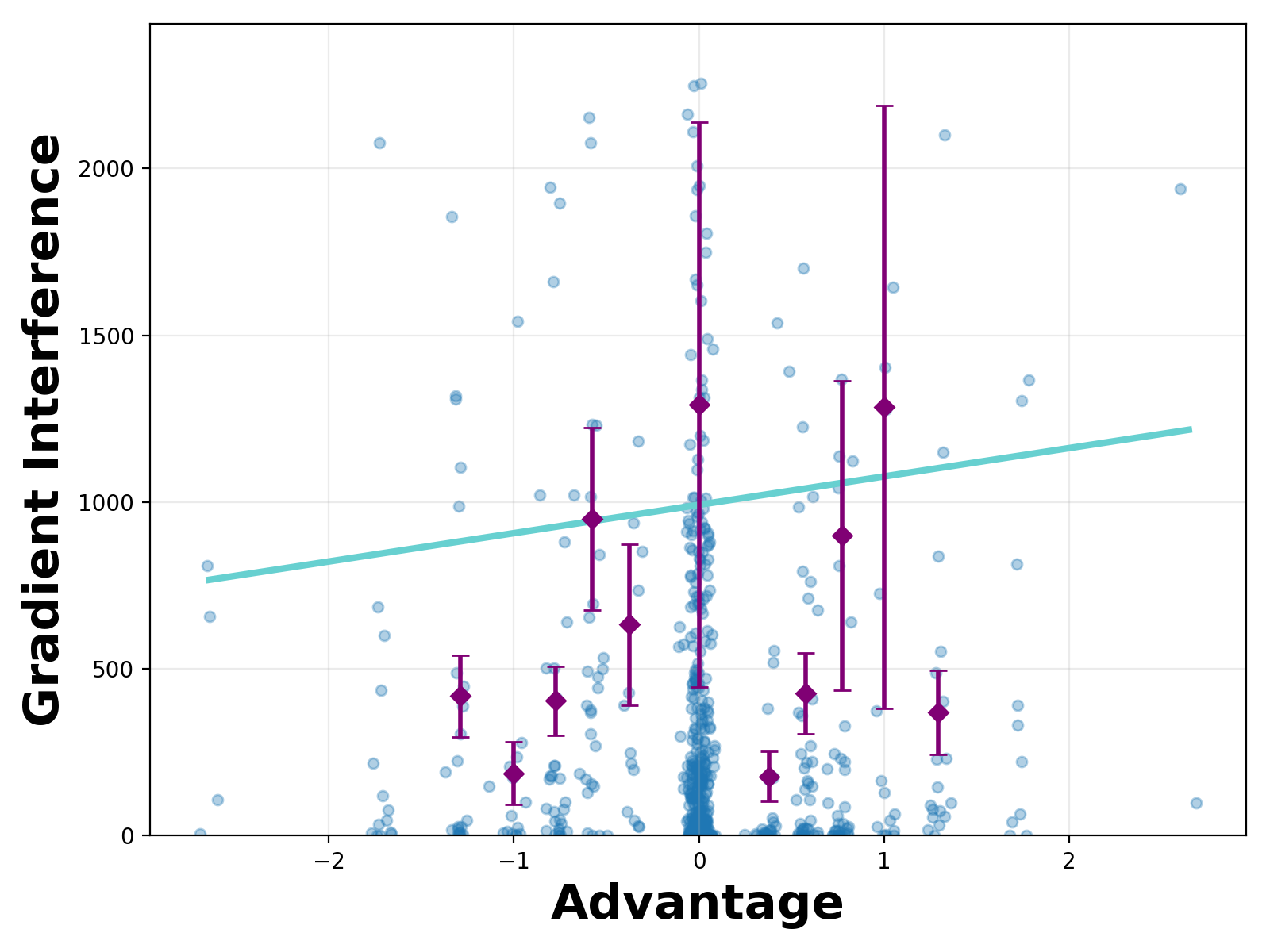}
    \vspace{-1em}
    \caption{Gradient interference of Llama3.2-3B against trajectory advantage immediately before (left) and after (right) an entropy-eruption event on WebShop. Following the definition of $c_j$ in Lemma~\ref{lem:likelihood-decrease}, each point corresponds to trajectory $j$, with the x-axis representing its normalized advantage and the y-axis representing its average gradient inner product with all other trajectories in the rollout group}
    \vspace{-1.5em}
    \label{fig:grad-correlation}
\end{figure}

\paragraph{Representation similarity causes gradient interference.} When trajectories are close in representation space, their per-token gradients align, so that suppressing one trajectory's likelihood inadvertently suppresses others. The following lemma makes this precise.

\begin{lemma}[High representation similarity induces gradient interference]
\label{lem:gradient-interference}
Fix a prompt $x$, and let $\theta = (W,\phi)$, where $W \in \mathbb{R}^{|V|\times d}$ is the output matrix and $\phi$ denotes all remaining model parameters. For any trajectory $y_i = (y_{i,1},\dots,y_{i,|y_i|})$, let $h_{i,k} := h_\phi(x,y_{i,<k}) \in \mathbb{R}^d$, $p_{i,k} := \pi_\theta(\,\cdot \mid x,y_{i,<k}) \in \mathbb{R}^{|V|}$, and let $e_{y_{i,k}} \in \mathbb{R}^{|V|}$ be the one-hot vector of token $y_{i,k}$. Define $\ell_i(\theta) := \sum_{k=1}^{|y_i|} \log p_{i,k}$. 
If $g_i := \nabla_\theta \ell_i(\theta)$ is the full gradient of the trajectory log-likelihood, then for any two trajectories $y_i$ and $y_j$, their gradient interference can be formulated as
$$\langle g_i, g_j \rangle=
\underbrace{\sum_{k=1}^{T_i}\sum_{k'=1}^{T_j} \alpha_{i,k;j,k'} \, \langle h_{i,k}, h_{j,k'} \rangle}_\text{representation similarity} + \underbrace{\sum_{k=1}^{T_i}\sum_{k'=1}^{T_j}\Bigl\langle J_{i,k}^{\top} W^{\top} q_{i,k},\,J_{j,k'}^{\top} W^{\top} q_{j,k'}\Bigr\rangle}_\text{relevant to LLM backbone geometry}.$$
where $\alpha_{i,k;j,k'} := \langle q_{i,k}, q_{j,k'}\rangle = \langle  e_{y_{i,k}}-p_{i,k},\, e_{y_{j,k'}}-p_{j,k'} \bigr\rangle$, and $J_{i,k} := \frac{\partial h_{i,k}}{\partial \phi} \in \mathbb{R}^{d \times \dim(\phi)}$ denote the Jacobian of the hidden state with respect to the non-output parameters.
\end{lemma}
\vspace{-0.5em}
\noindent This lemma shows that high cosine similarity between hidden states $h_{i,k}$ and $h_{j,k'}$ directly amplifies gradient interference between the two trajectories. This prediction is also supported empirically: in Figure~\ref{fig:grad-correlation}, we compare gradient interference before and after entropy eruption, and show that negatively advantaged trajectories exhibit substantially stronger interference before the eruption, consistent with the proposed mechanism.
\vspace{-0.5em}
\paragraph{Gradient interference suppresses valid trajectories.}
Strong gradient interference means that when the RL objective suppresses a negatively-advantaged trajectory, it simultaneously pulls down the likelihood of other trajectories that share similar representations---including correct ones.
The next lemma formalizes the likelihood reduction effect.

\begin{lemma}[Gradient interference leads to likelihood decrease]
Fix a prompt~$x$ and a group of sampled trajectories $y_1, \dots, y_G$ with advantages $A_1, \dots, A_G$. For each response $y_i$, let $g_i := \nabla_\theta \ell_i(\theta) = \nabla_\theta \log \pi_\theta(y_i \mid x)$ be the gradient of the log-likelihood of response $y_i$. We further define $c_j := \frac{1}{G}\sum_{i=1}^{G}\langle g_i, g_j \rangle$, which measures how strongly trajectory~$y_j$'s gradient aligns with the group on average. Under a gradient step $\frac{d}{dt}\theta = -\eta\,\nabla_\theta \mathcal{L}_{\mathrm{RL}}(\theta)$ with $\eta > 0$, the average log-likelihood change of the group satisfies \[ \frac{1}{G}\sum_{i=1}^{G} \frac{d}{dt}\ell_i(\theta) \;=\; \sum_{j=1}^{G} A_j\, c_j. \] This quantity is negative, i.e.\ the group's average likelihood decreases, when $\sum_{j=1}^{G} A_j\, c_j < 0$. 
\label{lem:likelihood-decrease}
\end{lemma}

The lemma implies that the average likelihood of the response group could decrease when the gradient of a negatively advantaged trajectory $y_j$ strongly interferes with other trajectories, namely when $c_j$ is large. In that case, suppressing \(y_j\) also suppresses other trajectories that are semantically similar. Consequently, even the likelihood of correct responses in the group may be pulled down if they share strong representation similarity with negatively advantaged trajectories.

\paragraph{Likelihood decrease drives entropy growth.}
We next show, through the following lemma, that this likelihood decrease on valid trajectories increases entropy and thus drives the policy toward a flatter distribution.

\begin{lemma}[Likelihood decline on high-probability responses increases entropy]\label{lem:entropy-increase} 
Fix a prompt $x$, and let $\pi_t(y):=\pi_{\theta_t}(y\mid x)$ be the distribution at the training step $t$. If at time $t$, $\mathrm{Cov}(\log\pi_t(y),\frac{\mathrm d}{\mathrm d t}\ell_t(y))<0$, then $\frac{\mathrm d}{\mathrm dt}H_t>0$.
\end{lemma}

\noindent 
Once the valid-action ratio saturates in Phase~1, valid trajectories tend to occupy the high-likelihood region of the policy, but their likelihood then decreases during training (Figure~\ref{fig:likelihood}), as demonstrated in Lemma~\ref{lem:likelihood-decrease}. Thus, $\mathrm{Cov}(\log\pi_t(y),\frac{\mathrm d}{\mathrm d t} \ell_t(y))<0$ holds empirically.
This implies, based on Lemma~\ref{lem:entropy-increase}, that entropy must increase, leading to the eruption.

\begin{figure}
    \centering
    \includegraphics[width=0.32\linewidth]{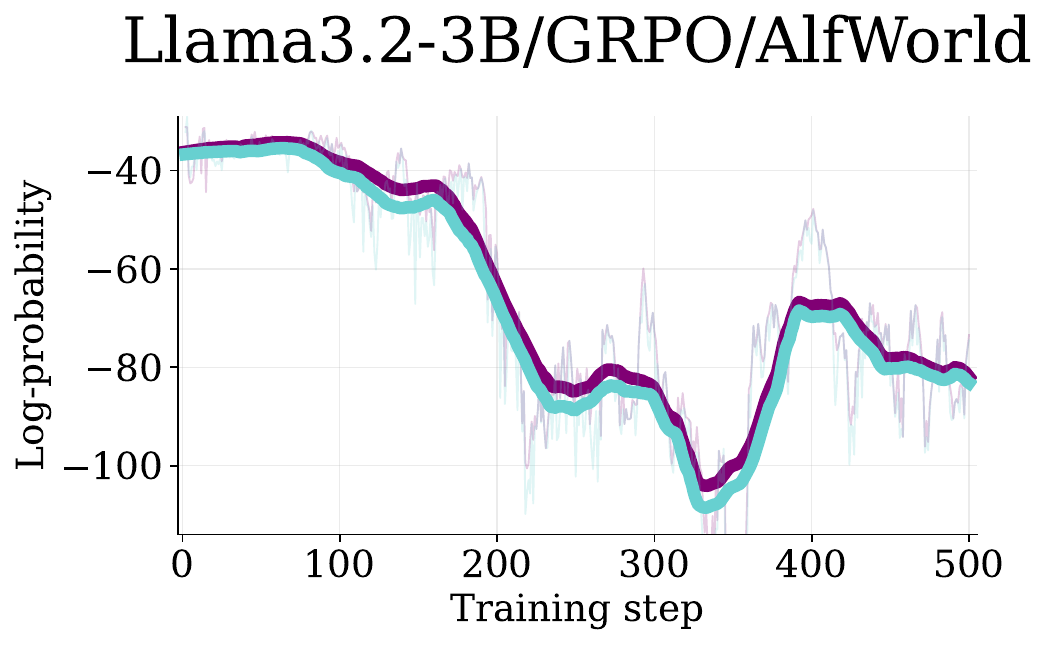}
    \includegraphics[width=0.32\linewidth]{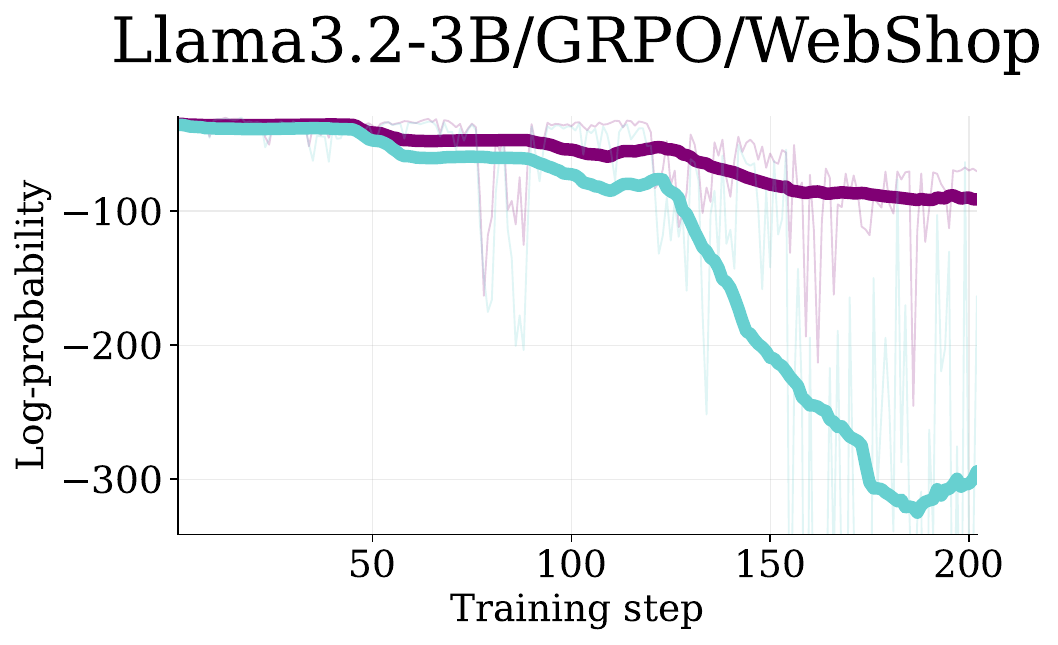}
    \includegraphics[width=0.32\linewidth]{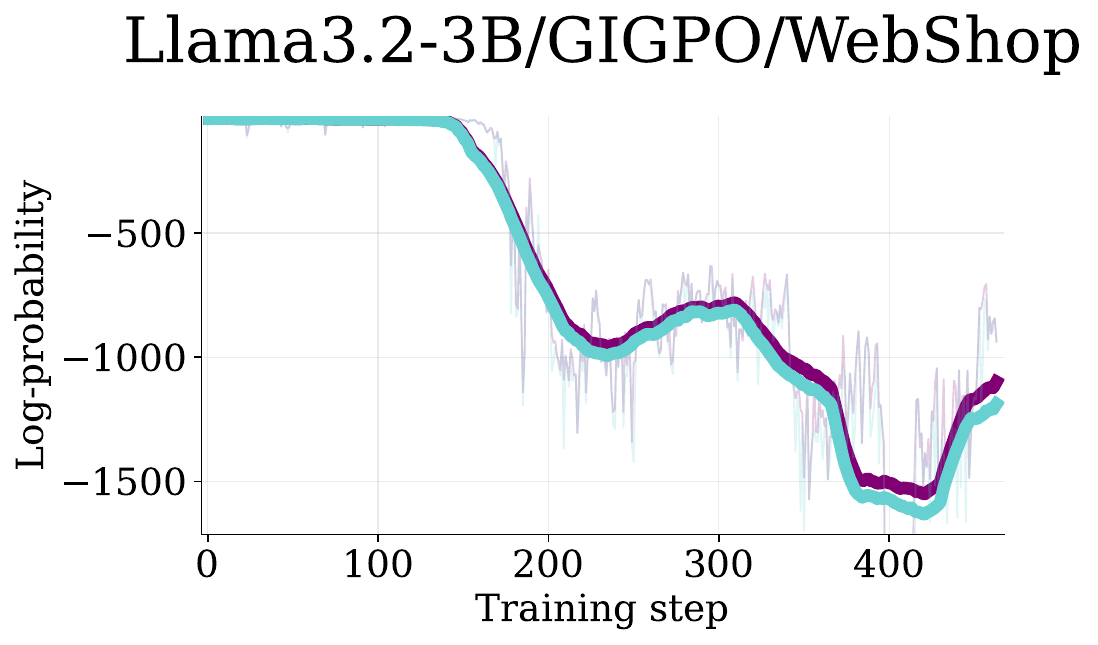}
    \includegraphics[width=0.32\linewidth]{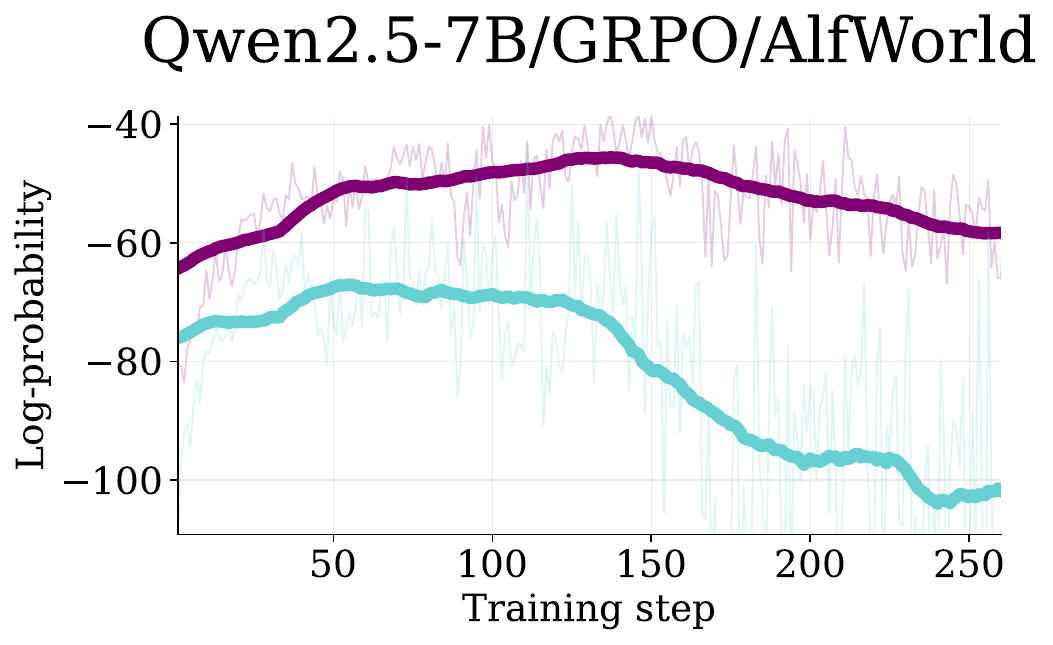}
    \includegraphics[width=0.32\linewidth]{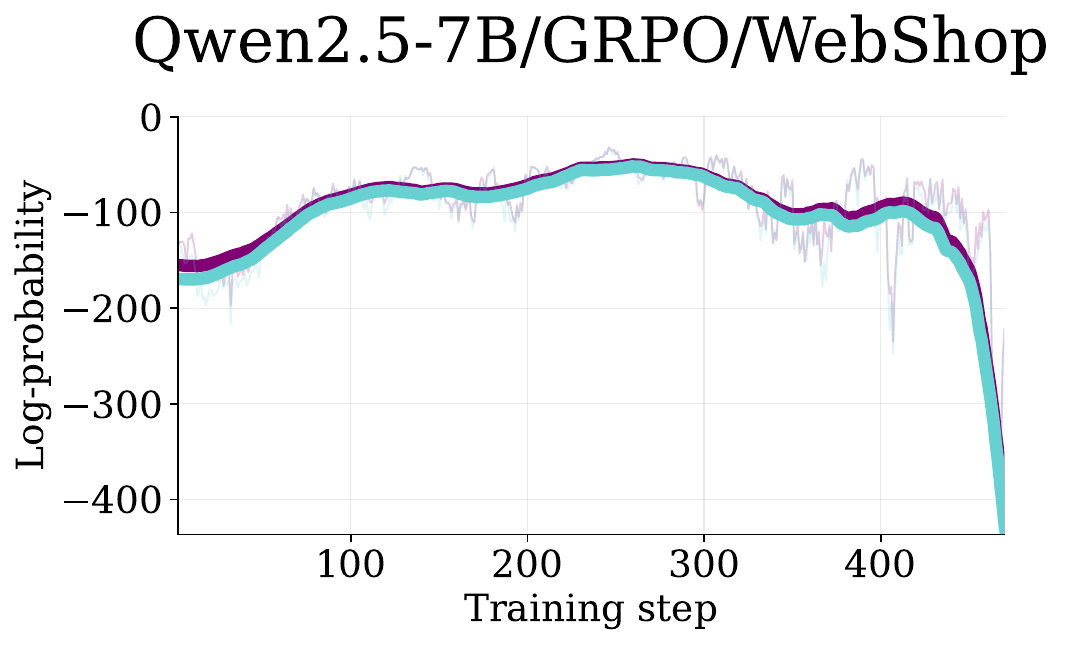}
    \includegraphics[width=0.32\linewidth]{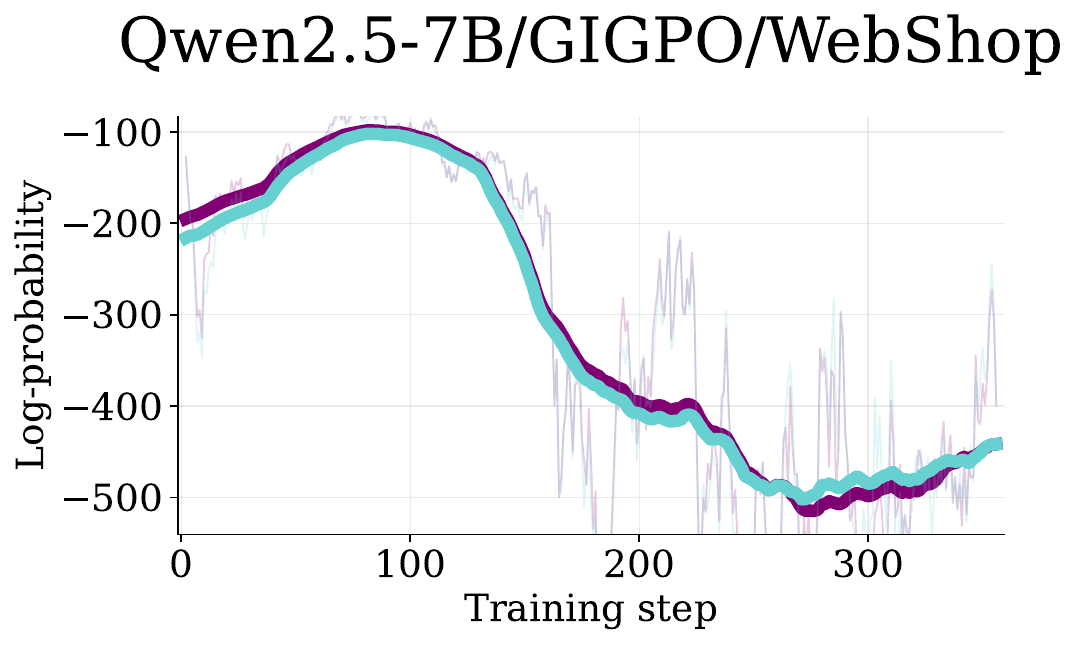}
    \vspace{-0.5em}
    \caption{The likelihood of \textcolor[HTML]{800074}{correct trajectories} and \textcolor[HTML]{67d0d0}{valid trajectories} decreases during agent RL. The reported log-probabilities are averages computed over all correct trajectories (purple) and all valid trajectories (orange) at each training step. More examples are shown in Appendix~\ref{apd:sec:likelihood-decrease}.}
    \label{fig:likelihood}
\end{figure}

\begin{figure}
    \centering
    \includegraphics[width=0.32\linewidth]{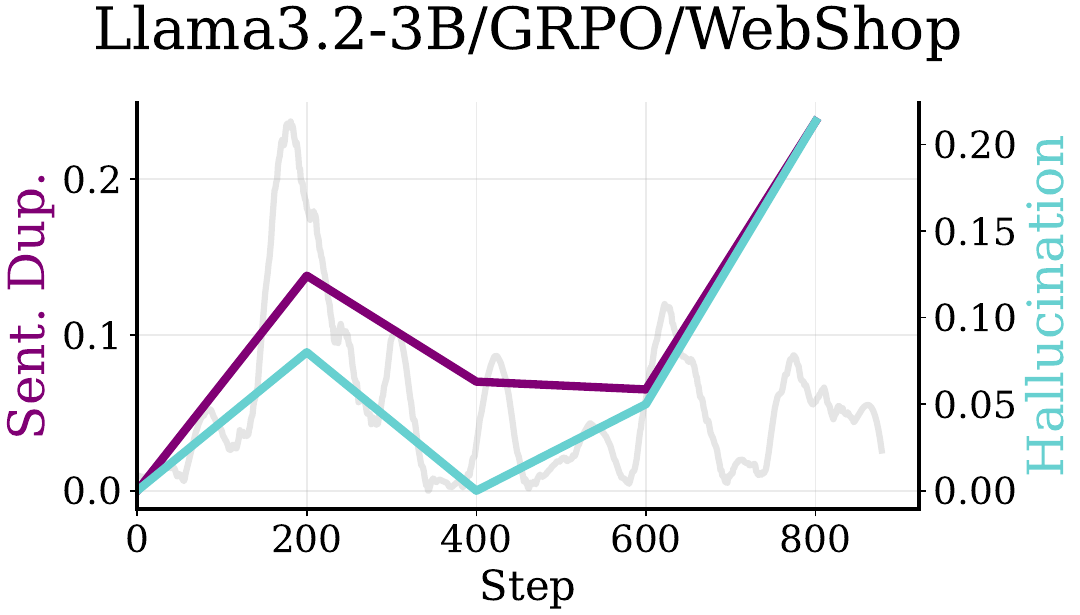}
    \includegraphics[width=0.32\linewidth]{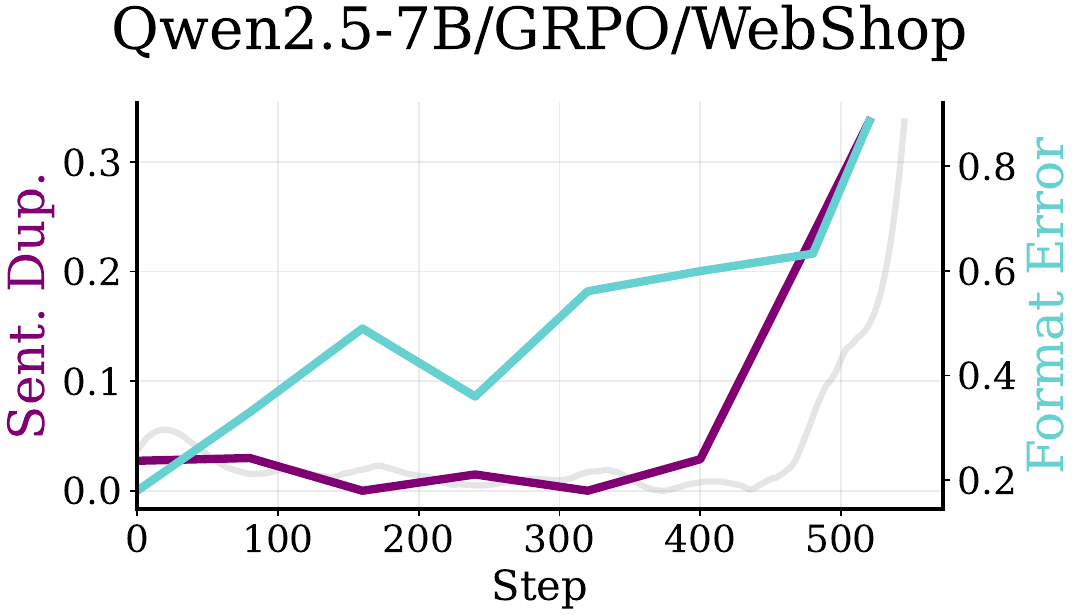}
    \includegraphics[width=0.32\linewidth]{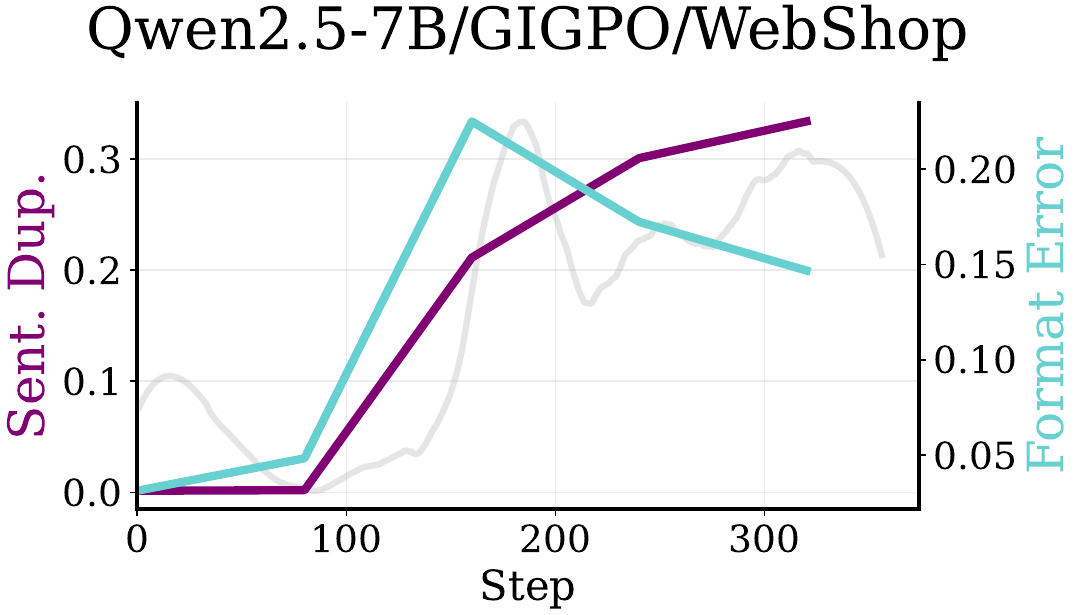}
    \vspace{-0.5em}
\caption{Degenerate patterns such as \textcolor[HTML]{800074}{Sentence Duplication} and \textcolor[HTML]{67d0d0}{Hallucination} occur in entropy eruption. Sent. Dup. represents sentence duplication. y-axis represents percentage. Gray background curve is the policy entropy.}
\vspace{-1.5em}
    \label{fig:toxic}
\end{figure}

\paragraph{Entropy eruption flattens the policy and amplifies degenerate patterns.}

As entropy erupts, the policy distribution becomes flatter, making low-quality yet non-negligible trajectory modes more likely to be sampled during training. In experiments, we find that this phase is strongly associated with the emergence of degenerate patterns, such as \emph{sentence duplication} and \emph{hallucination}. To evaluate this phenomenon, we collect 1024 sampled trajectories at different training checkpoints and compute the corresponding statistics. For sentence duplication, we split each response into sentence-level chunks and measure the fraction of sentences that are repeated within the same response. For hallucination, we measure the rate at which the model selects product attributes that do not actually appear on the webpage. As shown in Figure~\ref{fig:toxic}, across different backbones and both GRPO and GIGPO training, both metrics significantly rise during entropy-eruption periods, indicating that the entropy eruption regime systematically amplifies pathological behaviors and degrades the quality of sampled trajectories.

\subsection{Phase 3: Entropy Subsidence and the Recurring Cycle}
\begin{wrapfigure}{r}{0.41\textwidth}
    \centering
    \vspace{-1.5em}
\includegraphics[width=0.40\textwidth]{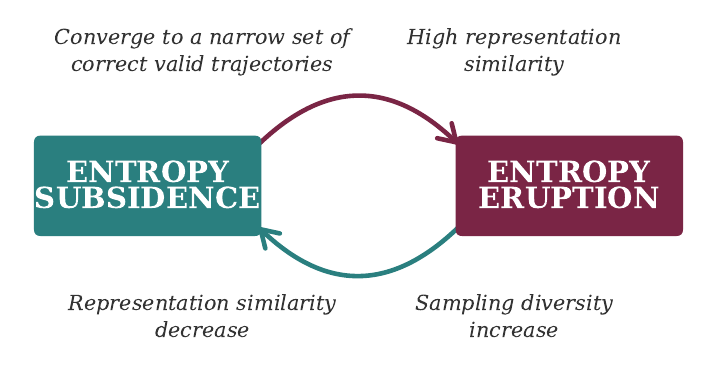}
\vspace{-1.em}
    \caption{The cyclical dynamics between entropy subsidence and entropy eruption. }
    \vspace{-1em} \label{fig:cyc}
\end{wrapfigure}
\paragraph{Why does the cycle repeat?}
Entropy eruption does not persist indefinitely. As the distribution flattens, the policy samples more diverse trajectories, which in turn reduces representation similarity among them.  
By Lemma~\ref{lem:gradient-interference} and Lemma~\ref{lem:likelihood-decrease}, this reduction in trajectory similarity means less gradient interference, so the suppressive effect of negatively-advantaged trajectories on correct ones is reduced and the likelihood of correct trajectories recovers.
As a result, training gradually shifts from a highly entropic exploration back toward reinforcing correct responses.
However, the subsidence phase carries the seed of the next eruption. As the policy reconcentrates on the correct trajectories discovered during the high-entropy phase, representation similarity rises once again, restoring the conditions that triggered Phase~2.
The dynamics therefore form a self-perpetuating cycle, illustrated in Figure~\ref{fig:cyc}: convergence raises similarity $\rightarrow$ similarity causes eruption $\rightarrow$ eruption increases diversity $\rightarrow$ diversity enables subsidence $\rightarrow$ and subsidence leads back to convergence.

\paragraph{The lasting damage of entropy eruption.}
Although entropy subsides and valid-trajectory likelihood recovers, the damage introduced during eruption can persist. When the distribution is flat, degenerate trajectories, such as those containing sentence duplication and hallucination, are sampled more frequently. If such trajectories still produce a correct final answer and receive a positive reward, the policy reinforces these patterns. As shown in Figure~\ref{fig:toxic}, degenerate behaviors acquired during eruption often persist into subsequent subsidence phases and can even be amplified as the policy concentrates on the (now-corrupted) correct trajectories.
We provide additional analysis and statistics in Appendix~\ref{apd:sec:lasting-effect}.
In more severe cases, entropy eruption can trigger a semantic collapse and completely derail RL training. As shown in Figure~\ref{fig:entropy-eruption-apd}, for Llama3.2-1B on WebShop, an entropy eruption is followed by a collapse of the reward curve to zero.

\section{How to Alleviate Entropy Eruption?}

While the primary focus of this paper is on understanding the training dynamics of agent RL, we show here that our analysis directly suggests a simple practical intervention---and that even a minimal approach motivated by the theory can meaningfully stabilize training and improve performance. We expect future work to deepen this line of inquiry with more sophisticated methods; our goal here is to demonstrate that the theoretical understanding developed above is actionable and guides practice.

\subsection{SEAL: Separation-Enhanced Agent Learning}

The analysis in Section~\ref{sec:dynamics} traces entropy eruption to a single root cause: high representation similarity between correct and incorrect trajectories amplifies gradient interference, suppressing valid-trajectory likelihood and flattening the distribution. A natural remedy is to push these two trajectory classes apart in representation space. We propose \textbf{SEAL} (\textbf{S}eparation-\textbf{E}nhanced \textbf{A}gent \textbf{L}earning), a lightweight modification to the original RL objective, with no additional cost at inference time.

For each sampled response \(y_i=(y_{i,1},\dots,y_{i,T_i})\), let \(h_{i,t}\in\mathbb{R}^d\) denote the hidden representation of the \(t\)-th response token produced by the policy model. We attach a binary classifier on top of these token representations, $p_{i,t}=\sigma\!\left(\mathrm{MLP}(h_{i,t})\right)$,
where \(p_{i,t}\) is the predicted probability that the token comes from a correct trajectory. We assign every token in trajectory \(y_i\) the same binary label \(z_i\in\{0,1\}\), where \(z_i=1\) if \(y_i\) is correct and \(z_i=0\) otherwise. The overall training objective is
\begin{equation}
    \mathcal{L} = \mathcal{L}_{\mathrm{RL}} +\alpha \mathcal{L}_{\mathrm{SEAL}},
\end{equation}
with the binary cross-entropy loss for the classification task as follows,
\begin{equation}
    \mathcal{L}_{\mathrm{SEAL}}
    =
    -\frac{1}{\sum_{i=1}^{G} T_i}
    \sum_{i=1}^{G}\sum_{t=1}^{T_i}
    \Big[
    z_i \log p_{i,t}
    +
    (1-z_i)\log(1-p_{i,t})
    \Big].
\end{equation}
This auxiliary objective encourages token representations from correct trajectories to be distinguishable from those of incorrect trajectories, thereby reducing harmful representation overlap between them. As a result, the interference from incorrect trajectories on correct ones is weakened, which mitigates the likelihood decrease of valid trajectories and helps suppress entropy eruption.

\subsection{Experimental Settings}
\label{sec:experimental-settings}
We first train the LLM agents on two challenging benchmarks: ALFWorld~\citep{alfworld} and
WebShop~\citep{webshop}. ALFWorld is an embodied environment designed to assess the ability of LLM agents to perform multi-step decision-making. In each episode, the agent receives a text goal and must accomplish it through multi-turn interaction with the environment. It includes 3,827 task instances across six categories of common household activities: Pick \& Place (Pick), Examine in Light (Look), Clean \& Place (Clean), Heat \& Place (Heat), Cool \& Place (Cool), and Pick Two \& Place (Pick2).
WebShop is a complex, web-based interactive environment designed to test the LLM agents in realistic online shopping scenarios. To complete the task, the agent must interact with a simulated HTML-based shopping website to search for, navigate to, and ultimately purchase a suitable item. It contains over 1.1 million products and 12k user instructions, providing a rich and diverse action space. We also conduct further experiments on search-augmented QA in Appendix~\ref{apd:sec:search}.
We use Qwen2.5-1.5B/7B-Instruct~\citep{qwen2.5} and Llama3.2-1B/3B-Instruct~\citep{LLama3.2} as our base models. Since Llama3.2 instruct models cannot sample valid trajectories throughout the RL training, we perform mid-training on Llama3.2 with a subset of OpenManus~\citep{OpenManus}. For hyperparameter settings, we mostly follow GIGPO~\citep{gigpo}. For fair comparison, all methods use the same learning rate, batch size, and training steps, and each pair of runs with and without SEAL is evaluated after the same amount of training. Implementation details are in Appendix~\ref{apd:sec:implementation-details}.

\definecolor{sealbg}{HTML}{EEF4FF}
\begin{table}[t]
\footnotesize
    \centering
    \caption{Comparison of GRPO and GIGPO performance with and without SEAL on AlfWorld and WebShop for Qwen2.5- and Llama3.2-series models.}
    \label{tab:main-alfworld} 
    \scalebox{0.9}{
    \begin{tabular}{cc|ccccccc|cc}
    \toprule 
       \multirow{2}{*}{Model} & \multirow{2}{*}{Methods} & \multicolumn{7}{c|}{\textsc{AlfWorld}} & \multicolumn{2}{c}{\textsc{WebShop}}\\
       & &Pick& Look& Clean &Heat& Cool& Pick2& {\bf Avg.} & Score& {\bf Succ.}  \\
       \midrule
       \multirow{4}{*}{\makecell{Llama3.2\\-1B}} & GRPO  & 81.48& 87.50& 92.86 &93.33&70.83 &69.57 &82.60  & 00.00 & 00.00  \\
        \rowcolor{sealbg}\cellcolor{white} Llama3.2 &GRPO + SEAL& 96.43 & 84.62 &100.00 &83.33 &87.50 &84.00 &\textbf{89.31} &92.42&\textbf{79.69}\\
        &GIGPO & 91.43&75.00 & 100.00&92.86&57.14& 95.83& 85.38&86.09&75.00   \\
        \rowcolor{sealbg}\cellcolor{white} &GIGPO + SEAL& 83.33& 94.25& 95.65&100.00 &95.46 &90.00 & \textbf{93.12}&89.39&\textbf{78.90}\\
       \midrule
        \multirow{4}{*}{\makecell{Qwen2.5\\-1.5B}} & GRPO & 90.91& 87.50& 83.87& 91.67& 84.85& 84.21 & 87.17 & 85.97 & 73.04\\
         \rowcolor{sealbg}\cellcolor{white} Qwen2.5&  GRPO + SEAL&97.22&88.89&96.00& 100.00&93.10& 91.30& \textbf{94.42} & 89.92& \textbf{81.25}\\
        &GIGPO &93.75& 86.77& 84.62&84.62& 95.00& 77.27& 87.01& 86.11&71.88\\
        \rowcolor{sealbg}\cellcolor{white} &GIGPO + SEAL & 100.00&85.00& 91.68& 75.00& 95.83& 91.30& \textbf{89.80} & 91.14& \textbf{80.47}\\
        \midrule
       \multirow{4}{*}{\makecell{Llama3.2\\-3B}} & GRPO & 96.88& 100.00& 96.15& 76.92& 90.00& 86.36& 91.05  &87.15&63.28\\
        \rowcolor{sealbg}\cellcolor{white} Llama3.2 &GRPO + SEAL& 97.22&95.00&95.46& 86.67& 100.00& 100.00&\textbf{95.73} & 89.86 & \textbf{77.73}\\
        &GIGPO & 94.44& 100.00 & 96.67& 90.00& 76.47 & 96.29 & 92.31&89.07&71.88  \\
        \rowcolor{sealbg}\cellcolor{white} &GIGPO + SEAL& 96.30& 100.00 & 100.00& 86.67& 100.00& 86.96& \textbf{94.99} &  90.26&\textbf{78.91} \\
        \midrule
        \multirow{4}{*}{\makecell{Qwen2.5\\-7B}} & GRPO &  100.00& 71.42& 91.30& 86.67& 80.95& 55.56 &80.98&90.81& 76.95 \\
        \rowcolor{sealbg}\cellcolor{white} Qwen2.5 &GRPO + SEAL & 94.60& 85.71 & 100.00 & 86.67 & 85.74 & 50.00 & \textbf{83.79}&91.13&\textbf{80.08}\\
        &GIGPO & 96.55& 84.61& 95.65 & 94.44&80.77&85.71&89.62 &89.61 & 75.00  \\
        \rowcolor{sealbg}\cellcolor{white} &GIGPO + SEAL& 94.60&85.71 &100.00 &86.67 &85.71 & 91.67& \textbf{90.72}&90.23& \textbf{79.69} \\
        \bottomrule
    \end{tabular}}
\end{table}

\begin{figure}[t!]
    \centering
    \vspace{-0.5em}
    \includegraphics[width=0.24\linewidth]{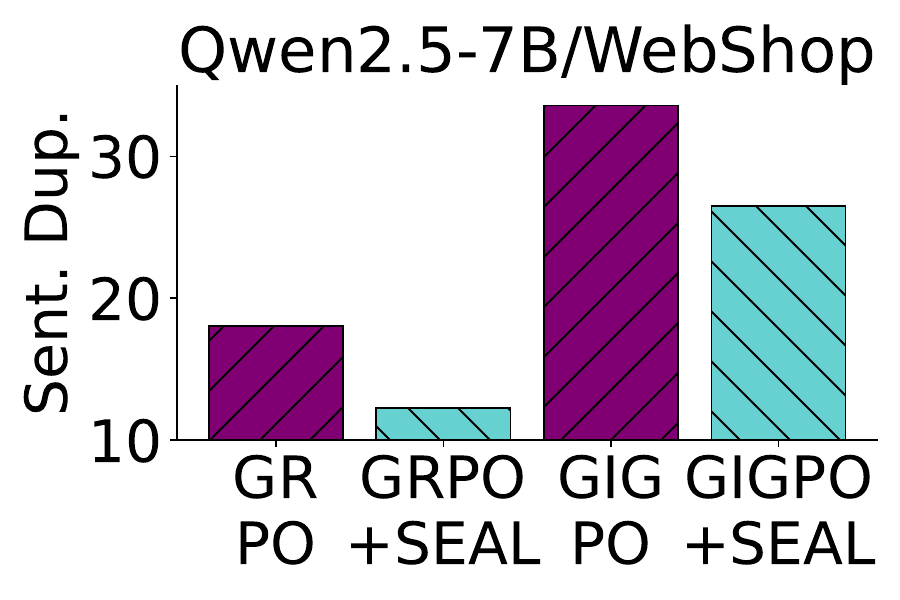}
    \includegraphics[width=0.24\linewidth]{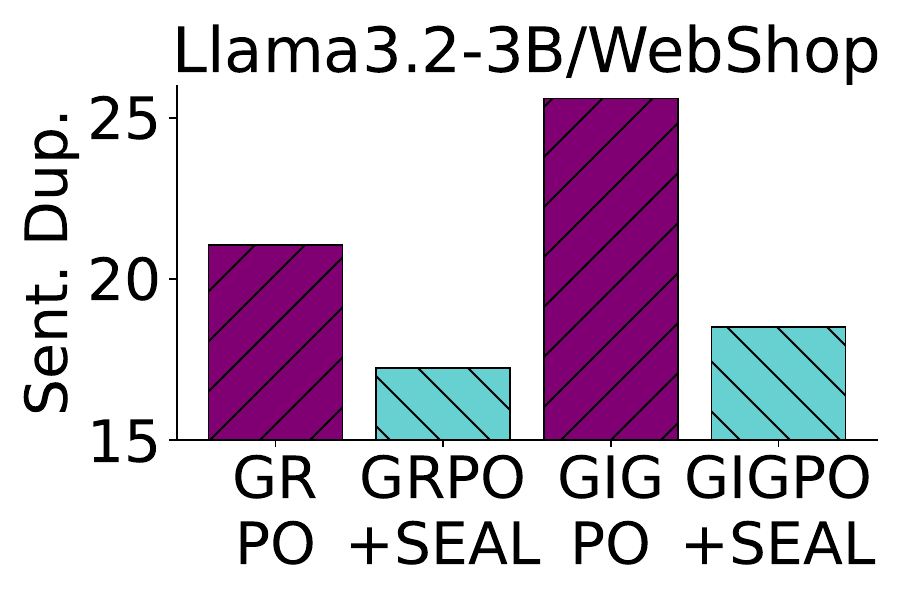}
    \includegraphics[width=0.24\linewidth]{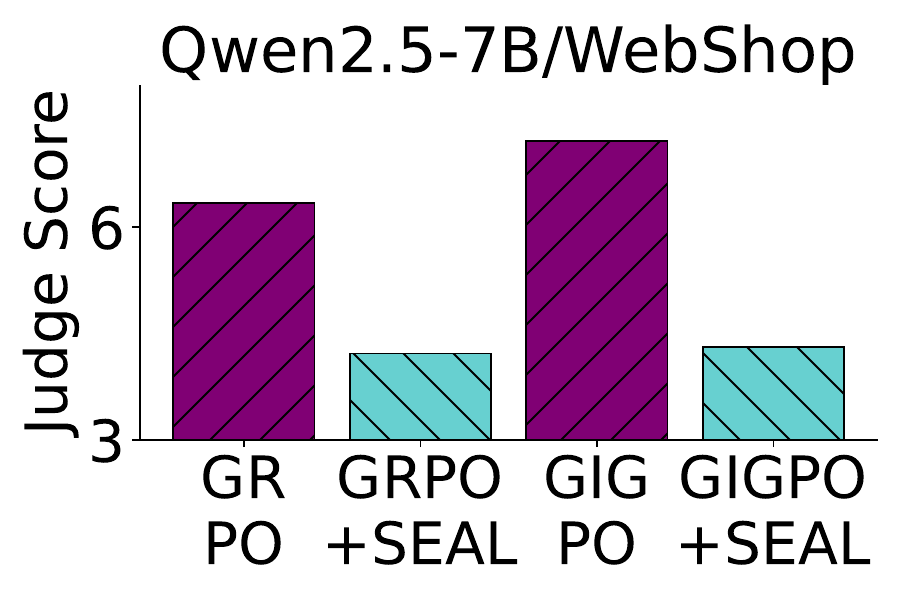}
    \includegraphics[width=0.24\linewidth]{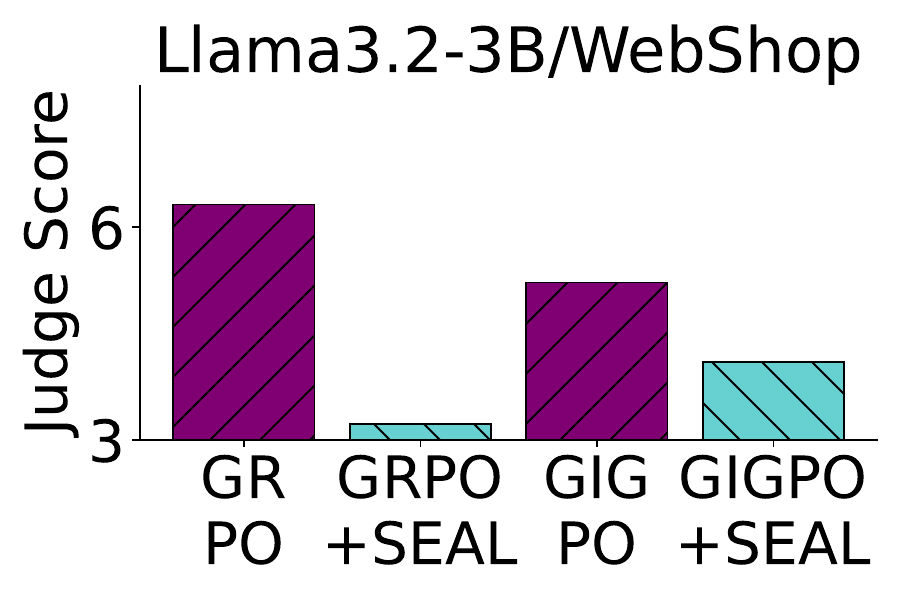}
    \vspace{-0.5em}
    \caption{Comparison of sampling quality in terms of sentence duplication ratio (Sent. Dup.\%) and GPT-4-based degeneracy score. SEAL reduces degenerate patterns in agent RL.}
    \vspace{-1em}
    \label{fig:toxic-contrast}
\end{figure}

\subsection{SEAL Improves the Performance and Suppresses Degenerate Patterns.}

\paragraph{SEAL improves the performance of agent RL across different tasks and backbones.} Table~\ref{tab:main-alfworld} shows that SEAL consistently improves agent RL performance across both tasks and model backbones. On AlfWorld, adding SEAL generally raises the average success over subtasks for both GRPO and GIGPO, with especially large gains for smaller models such as Llama3.2-1B and Qwen2.5-1.5B. On WebShop, SEAL also leads to clear improvements in both score and success rate across the Qwen and Llama series. Notably, SEAL turns a near-failed run of Llama3.2-1B on WebShop into strong performance. Overall, these results indicate that SEAL is a robust and effective enhancement for agent RL, delivering gains across different backbones, training algorithms, and task environments.

\paragraph{SEAL alleviates entropy eruption and suppresses degenerate patterns.}
As shown in Appendix~\ref{apd:sec:alleviate-eruption}, SEAL consistently mitigates the entropy-eruption phenomenon across different tasks, backbones, and RL algorithms. 
This stabilization also translates into better sample quality. In Figure~\ref{fig:toxic-contrast}, we evaluate degenerate behavior on WebShop using two complementary metrics: the sentence-duplication ratio and an LLM-as-a-judge degenerate score with GPT-4. Across both Qwen2.5-7B and Llama3.2-3B, SEAL consistently reduces sentence duplication for both GRPO and GIGPO. These results suggest that by preventing the policy from entering an excessively diffuse high-entropy regime, SEAL not only improves final task performance but also suppresses pathological generation behaviors during training.

\paragraph{Ablations.}
SEAL is not overly sensitive to the auxiliary-loss weight $\alpha$. As shown in Appendix~\ref{apd:sec:hyperparameter}, SEAL remains effective over a broad range of $\alpha$, with the best or near-best performance typically achieved at intermediate values. Importantly, across a wide range of $\alpha$, SEAL generally remains better than the corresponding baselines without SEAL, suggesting that its gains are robust rather than dependent on delicate hyperparameter tuning.

\section{Related Works}

\paragraph{Reinforcement learning for LLM agents.}
LLM agents \citep{react,toolformer,reflexion,sweagent} extend static text generation to sequential decision making, requiring models to plan, invoke tools, interact with environments, and maintain state over long horizons. Thus, reinforcement learning is a natural framework for improving agent behavior from outcome feedback. Recent work \citep{ragen,searchr1,toolrl,webagentr1} has largely built on PPO/GRPO-style objectives \citep{ppo,grpo} originally developed for single-turn reasoning and alignment, while adapting them to the agent setting with techniques such as context compression \citep{compress1,compress2,compress3,compress4,compress5}, decoupled execution-training pipelines, and memory-aware optimization \citep{memagent,memr1,mem1,mem2}. These advances have substantially improved empirical performance, but the optimization dynamics of agent RL remain poorly understood. Our work complements this line of research by studying how the policy distribution evolves during training and by identifying the phenomenon of cyclical entropy eruption.

\paragraph{Entropy dynamics in LLM RL.}
Policy entropy is a central quantity in policy-gradient RL, governing the exploration-exploitation trade-off: higher entropy supports broader search, whereas rapid entropy reduction can lead to overly deterministic behavior \citep{entropycontrol_llmrl,haarnoja2018soft, messaoud2024s}. In LLM RL, prior work has shown that entropy often collapses early in training and then remains persistently low, reducing response diversity and contributing to performance saturation \citep{entropy_mechanism_llm,entropycontrol_llmrl,invisibleleash}. To address this issue, existing methods have introduced entropy regularization \citep{entropy_mechanism_llm,reasoning_with_exploration} or alternative distribution-matching objectives \citep{flowrl,lad}. A few recent works observe entropy fluctuations in LLM RL \citep{epo_agents,quantile}, but do not explain the mechanism behind them. In contrast, we focus on the agent setting, where long-horizon trajectories, tool interactions, and validity constraints induce qualitatively different dynamics. We provide a theoretical and empirical account linking entropy eruption to representation similarity and gradient interference, and show that the resulting behavior is inherently cyclical.

\section{Conclusion}

In this work, we study the training dynamics of agent RL and identify a previously underexplored phenomenon where agent RL exhibits recurring cycles of entropy descent, eruption, and subsidence. We term it \emph{cyclical entropy eruption}. Through theoretical and empirical analysis, we show that this behavior is closely related to high representation similarity among trajectories, which induces harmful gradient interference, suppresses the likelihood of valid trajectories, and amplifies degenerate patterns such as sentence duplication and hallucination.
Building on this understanding, we propose \textsc{SEAL}, a simple training objective that encourages better separation between correct and incorrect trajectories in representation space. Despite its simplicity, SEAL effectively stabilizes training, reduces entropy eruption and degenerate behaviors, improving downstream agent performance across multiple benchmarks, model families, and RL algorithms. 

\section*{Acknowledgement}
This work is supported in part by the AFOSR Young Investigator Program under award number FA9550-23-1-0184, the Office of Naval Research (ONR) under award number N000142612508, the National Science Foundation under awards IIS-2237037 and IIS-2331669, Schmidt Sciences Foundation, Open Philanthropy (now Coefficient Giving), Alfred P. Sloan Fellowship, and gifts from Google and Amazon. Shawn Im is also supported by the National Science Foundation Graduate Research Fellowship Program under Grant No. 2137424. Any opinions, findings, conclusions or recommendations expressed in this material are those of the authors and do not necessarily reflect the views of the National Science Foundation. 

{\small
\bibliographystyle{unsrt}
\bibliography{reference}
}

\newpage
\appendix
\newcommand{\softmax}{\mathrm{softmax}}
\section{Proofs}
\label{apd:sec:proofs}

\begin{lemma}[Format-gated learning dominates training at the beginning.]
\label{lem:format_gated_learning}
Suppose the trajectory reward factorizes as $r(y)=r_{\mathrm{fmt}}(y)\,r_{\mathrm{sem}}(y)$,
where \(r_{\mathrm{fmt}}(y)\in\{0,1\}\) and \(r_{\mathrm{sem}}(y)\in\{0,1\}\) denote the format and semantic correctness. Let $p_\theta^{\mathrm{fmt}}(x) =\sum_{y:\,r_{\mathrm{fmt}}(y)=1} \pi_\theta (y\mid x)$, and $p_\theta^{\mathrm{sem}}(x)=\sum_{y:\,r_{\mathrm{sem}}(y)=1}\pi_\theta(y\mid x)$ denote the total probability the policy assigns to format-valid and semantically correct trajectories, respectively. If $\mathbb E_x p_\theta^{\mathrm{sem}}(x) > \mathbb E_x p_\theta^{\mathrm{fmt}}(x)$, \emph{i.e.}, the policy is weaker on agent-specific formatting than on semantic correctness, RL updates are dominated by increasing the format validity of trajectories. 
\end{lemma}
\begin{proof}
Let
\[
F:=\{y:\,r_{\mathrm{fmt}}(y)=1\},
\qquad
S:=\{y:\,r_{\mathrm{sem}}(y)=1\}.
\]
Since \(r_{\mathrm{fmt}},r_{\mathrm{sem}}\in\{0,1\}\), we have
\[
r(y)=r_{\mathrm{fmt}}(y)\,r_{\mathrm{sem}}(y)=\mathbf{1}\{y\in F\cap S\}.
\]
Hence, for RL objective $J_x(\theta)
:=
\mathbb{E}_{y\sim \pi_\theta(\cdot\mid x)}[r(y)]
=
\mathbb{E}_{y\sim \pi_\theta(\cdot\mid x)}[r_{\mathrm{fmt}}(y)\,r_{\mathrm{sem}}(y)]$, we have $J_x(\theta)=\sum_{y\in F\cap S}\pi_\theta(y\mid x)$. Note \(p_\theta^{\mathrm{fmt}}(x)>0\) and \(p_\theta^{\mathrm{sem}}(x)>0\), and define the conditional success rates
\[
q_\theta^{\mathrm{fmt}\mid\mathrm{sem}}(x)
:=
\mathbb{P}_\theta\!\big(r_{\mathrm{fmt}}=1 \mid r_{\mathrm{sem}}=1, x\big),
\qquad
q_\theta^{\mathrm{sem}\mid\mathrm{fmt}}(x)
:=
\mathbb{P}_\theta\!\big(r_{\mathrm{sem}}=1 \mid r_{\mathrm{fmt}}=1, x\big).
\]
By elementary probability,
\[
J_x(\theta)
=
\mathbb{P}_\theta(F\cap S\mid x)
=
\mathbb{P}_\theta(S\mid x)\,\mathbb{P}_\theta(F\mid S,x)
=
p_\theta^{\mathrm{sem}}(x)\,q_\theta^{\mathrm{fmt}\mid\mathrm{sem}}(x),
\]
and similarly
\[
J_x(\theta)
=
\mathbb{P}_\theta(F\mid x)\,\mathbb{P}_\theta(S\mid F,x)
=
p_\theta^{\mathrm{fmt}}(x)\,q_\theta^{\mathrm{sem}\mid\mathrm{fmt}}(x).
\]
Differentiating both factorizations gives
\[
\nabla_\theta J_x
=
q_\theta^{\mathrm{fmt}\mid\mathrm{sem}}\,\nabla_\theta p_\theta^{\mathrm{sem}}
+
p_\theta^{\mathrm{sem}}\,\nabla_\theta q_\theta^{\mathrm{fmt}\mid\mathrm{sem}},
\]
and
\[
\nabla_\theta J_x
=
q_\theta^{\mathrm{sem}\mid\mathrm{fmt}}\,\nabla_\theta p_\theta^{\mathrm{fmt}}
+
p_\theta^{\mathrm{fmt}}\,\nabla_\theta q_\theta^{\mathrm{sem}\mid\mathrm{fmt}}.
\]
The first identity shows that changes in format validity enter through the term $p_\theta^{\mathrm{sem}}\,\nabla_\theta q_\theta^{\mathrm{fmt}\mid\mathrm{sem}}$, so format improvement is weighted by the current semantic-correctness mass \(p_\theta^{\mathrm{sem}}\). Likewise, the second identity shows that changes in semantic correctness enter through $p_\theta^{\mathrm{fmt}}\,\nabla_\theta q_\theta^{\mathrm{sem}\mid\mathrm{fmt}}$,
so semantic improvement is weighted by the current format-valid mass \(p_\theta^{\mathrm{fmt}}\). Therefore, when $E_x p_\theta^{sem}(x) > E_x p_\theta^{fmt}(x)$, the gating of semantic improvement would be more severe than format improvement, and RL updates are thus dominated by increasing the probability of protocol-valid trajectories. 
\end{proof}

\begin{lemma}[Restatement of Lemma~\ref{lem:gradient-interference}]
Fix a prompt $x$, and let $\theta = (W,\phi)$, where $W \in \mathbb{R}^{|V|\times d}$ is the output matrix and $\phi$ denotes all remaining model parameters. For any trajectory $y_i = (y_{i,1},\dots,y_{i,|y_i|})$, define $\ell_i(\theta) := \sum_{k=1}^{|y_i|} \log \pi_\theta(y_{i,k}\mid x, y_{i,<k})$. 
For each token position $(i,k)$, let $h_{i,k} := h_\phi(x,y_{i,<k}) \in \mathbb{R}^d$ denote the output hidden states, $p_{i,k} := \pi_\theta(\,\cdot \mid x,y_{i,<k}) \in \mathbb{R}^{|V|}$ denote the output distribution, 
and $e_{i,k} \in \mathbb{R}^{|V|}$ be the one-hot label of token $y_{i,k}$.
If $g_i := \nabla_\theta \ell_i(\theta)$ is the full gradient of the trajectory log-likelihood, then for any two trajectories $y_i$ and $y_j$,
$$\langle g_i, g_j \rangle=
\sum_{k=1}^{T_i}\sum_{k'=1}^{T_j} \alpha_{i,k;j,k'} \, \langle h_{i,k}, h_{j,k'} \rangle+\sum_{k=1}^{T_i}\sum_{k'=1}^{T_j}\Bigl\langle J_{i,k}^{\top} W^{\top} q_{i,k},\,J_{j,k'}^{\top} W^{\top} q_{j,k'}\Bigr\rangle.$$
where $\alpha_{i,k;j,k'} := \langle q_{i,k}, q_{j,k'}\rangle = \langle  e_{{i,k}}-p_{i,k},\, e_{{j,k'}}-p_{j,k'} \bigr\rangle.$, and $J_{i,k} := \frac{\partial h_{i,k}}{\partial \phi} \in \mathbb{R}^{d \times \dim(\phi)}$ denote the Jacobian of the hidden state with respect to the non-output parameters.
\end{lemma}

\begin{proof}
For a fixed token position $(i,k)$, write $z_{i,k} := W h_{i,k} \in \mathbb{R}^{|V|}.$ 
Then $\pi_\theta(\,\cdot \mid x,y_{i,<k}) = \softmax(z_{i,k})$,
and therefore
$$\nabla_{z_{i,k}} \log \pi_\theta(y_{i,k}\mid x,y_{i,<k})=e_{y_{i,k}} - \softmax(z_{i,k})=e_{y_{i,k}} - p_{i,k}=q_{i,k}.$$
By the chain rule, the gradient with respect to the output matrix $W$ is 
$$\nabla_W \log \pi_\theta(y_{i,k}\mid x,y_{i,<k})=q_{i,k} h_{i,k}^{\top}.$$ 
Similarly, since $z_{i,k} = W h_{i,k}$ and $h_{i,k} = h_\phi(x,y_{i,<k})$, we have
$$\nabla_\phi \log \pi_\theta(y_{i,k}\mid x,y_{i,<k})=J_{i,k}^{\top} W^{\top} q_{i,k}.$$
Summing over all token positions of trajectory $y_i$ yields
$$
\nabla_W \ell_i(\theta)=\sum_{k=1}^{T_i} q_{i,k} h_{i,k}^{\top},
\qquad
\nabla_\phi \ell_i(\theta)=\sum_{k=1}^{T_i} J_{i,k}^{\top} W^{\top} q_{i,k}.$$
Under the decomposition $\theta=(\text{vec}(W),\phi)$, the inner product between the full gradients decomposes as
\[
\langle g_i, g_j \rangle
=
\left\langle \nabla_W \ell_i(\theta), \nabla_W \ell_j(\theta) \right\rangle
+
\left\langle \nabla_\phi \ell_i(\theta), \nabla_\phi \ell_j(\theta) \right\rangle.
\]
where we define the inner product between two matrices $\langle A, B \rangle$ as $\text{tr}(A^T B)$. Expanding the first term bilinearly gives
$$
\left\langle \nabla_W \ell_i(\theta), \nabla_W \ell_j(\theta) \right\rangle_F
=
\sum_{k=1}^{T_i}\sum_{k'=1}^{T_j}
\left\langle q_{i,k} h_{i,k}^{\top},\, q_{j,k'} h_{j,k'}^{\top} \right\rangle = \sum_{k=1}^{T_i}\sum_{k'=1}^{T_j} \langle q_{i,k}, q_{j,k'} \rangle \, \langle h_{i,k}, h_{j,k'} \rangle.
$$
where we use the identity $\langle ab^{\top}, cd^{\top} \rangle = \langle a,c\rangle \langle b,d\rangle$.
Likewise, expanding the $\phi$-block gives
$$ \left\langle \nabla_\phi \ell_i(\theta), \nabla_\phi \ell_j(\theta) \right\rangle=\sum_{k=1}^{T_i}\sum_{k'=1}^{T_j} \Bigl\langle J_{i,k}^{\top} W^{\top} q_{i,k},\,J_{j,k'}^{\top} W^{\top} q_{j,k'}\Bigr\rangle.$$

Adding the two contributions proves the claim.
\end{proof}

\paragraph{Meaning of the terms.}
The first double sum,
\[
\sum_{k=1}^{T_i}\sum_{k'=1}^{T_j}
\alpha_{i,k;j,k'} \, \langle h_{i,k}, h_{j,k'} \rangle,
\]
is exactly the $W$-block interaction from the constrained / output-layer-only analysis. Here
\[
\alpha_{i,k;j,k'} =
\bigl\langle e_{y_{i,k}}-p_{i,k},\, e_{y_{j,k'}}-p_{j,k'} \bigr\rangle
\]
measures how aligned the two token-level residual vectors are, while
$\langle h_{i,k}, h_{j,k'} \rangle$ measures alignment between the corresponding hidden states.
The second double sum,
\[
\sum_{k=1}^{T_i}\sum_{k'=1}^{T_j}
\Bigl\langle
J_{i,k}^{\top} W^{\top} q_{i,k},
\,
J_{j,k'}^{\top} W^{\top} q_{j,k'}
\Bigr\rangle,
\]
is the additional contribution coming from the trainable non-output parameters $\phi$.
It captures how token-level errors backpropagate through the network into the hidden states.
The factor $W^{\top} q_{i,k}$ is the gradient of the token log-probability with respect to the hidden state $h_{i,k}$,
and $J_{i,k}^{\top}$ maps that hidden-state gradient back to parameter space through the Jacobian of the hidden state with respect to $\phi$.

\begin{lemma}[Restatement of Lemma~\ref{lem:likelihood-decrease}]
Fix a prompt $x$ and a group of sampled responses $y_1,\dots,y_G$.
For each response $y_i$, define its log-likelihood under the current policy by $\ell_i(\theta) := \log \pi_\theta(y_i \mid x)$. Let $A_i$ denote the (normalized) advantage assigned to response $y_i$, and consider the local RL objective $L_{\mathrm{RL}}(\theta):=-\sum_{i=1}^G A_i \,\ell_i(\theta)$.
For each $i$, let $g_i := \nabla_\theta \ell_i(\theta)$ be the gradient of the log-likelihood of response $y_i$, and define $c_j := \frac{1}{G}\sum_{i=1}^G \langle g_i, g_j \rangle$. The quantity $c_j$ measures how strongly the update direction induced by response $y_j$ aligns, on average, with the gradients of all responses in the group; equivalently, it is a group-level \emph{similarity centrality} of response $y_j$. Suppose we perform a gradient descent step $\frac{\mathrm d}{\mathrm d t} \theta = -\eta \nabla_\theta L_{\mathrm{RL}}(\theta)$ with $\eta > 0$. The average log-likelihood change of the group is negative when $\sum_{j=1}^G A_j c_j <0$
\end{lemma}
\begin{proof}
Denote $K_{ij} := \langle g_i, g_j \rangle$. The matrix $K$ is the trajectory-level interaction kernel.  
From the objective, $$\nabla_\theta L_{\mathrm{RL}}(\theta)=-\sum_{j=1}^G A_j \nabla_\theta \ell_j(\theta)=-\sum_{j=1}^G A_j g_j.$$
Hence under gradient flow, $$\frac{\mathrm d }{\mathrm d t}{\theta}=-\eta \nabla_\theta L_{\mathrm{RL}}=\eta \sum_{j=1}^G A_j g_j.$$
Therefore, for each trajectory $i$,
$$\frac{\mathrm d }{\mathrm d t} \log\pi(y_i|x)= \langle g_i, \frac{\mathrm d \theta}{\mathrm d t} \rangle= \eta \sum_{j=1}^G A_j \langle g_i, g_j \rangle.$$
Averaging over $i$ gives
$$ \frac{1}{G}\sum_{i=1}^G \frac{\mathrm d }{\mathrm d t}\log \pi(y_i|x) =\frac{1}{G}\sum_{i=1}^G \eta \sum_{j=1}^G A_j \langle g_i, g_j \rangle= \eta \sum_{j=1}^G A_j \left(\frac{1}{G}\sum_{i=1}^G \langle g_i, g_j \rangle\right)= \eta \sum_{j=1}^G A_j c_j.$$
This proves the claim.
\end{proof}

\paragraph{Empirical evidence consistent with the mechanism.}
To further probe the causal role of gradient interference, we compare the relationship between trajectory advantage and gradient interference immediately before and after an entropy-eruption event in Figure~\ref{fig:grad-correlation}. Before the eruption, trajectories with negative advantage exhibit systematically larger gradient interference than those with positive advantage, as shown by the clear negative trend in the fitted line and the binned statistics. This indicates that poorly rewarded trajectories are more strongly coupled to the rest of the sampled group and therefore exert a larger suppressive effect on other trajectories, including valid or correct ones. This observation is consistent with our theory: when negatively advantaged trajectories are more central in the interaction kernel, they can drive a decrease in the average likelihood of the group and trigger entropy eruption. In contrast, after the eruption, this pattern is substantially weakened and even partially reversed: the dependence of gradient interference on advantage becomes much flatter, and negatively advantaged trajectories no longer dominate the interaction structure. This change is consistent with the role of entropy eruption in increasing trajectory diversity and reducing representation similarity, which in turn alleviates harmful interference. Taken together, the before-and-after comparison provides temporal evidence supporting our causal claim: strong interference from negatively advantaged trajectories appears before entropy eruption and is mitigated after the eruption occurs.

\begin{corollary}[Negatively-Advantaged trajectories reduce all trajectory likelihoods] 
Let $\mathcal{Y^+}$ be the set of indices corresponding to valid and correct trajectories and $\mathcal{Y^-}$ be the remaining trajectory indices. Suppose that the advantages are centered, $A_j > 0$ for all $j \in \mathcal{Y}^+$ and $A_j < 0$ for all $j \in \mathcal{Y}^-$. Given that all the gradients are similar, $\langle g_i, g_j\rangle > \lVert g_i \rVert \lVert g_j \rVert \delta$ for $1 \geq \delta > 0$ for any $i, j \in [G]$ and the negatively-advantaged trajectories dominate the gradient, $\delta \min_{j \in \mathcal{Y}^-}  \lVert g_j \rVert > \max_{j \in \mathcal{Y}^+} \lVert g_j \rVert$, $\sum_{j=1}^G A_j \langle g_i, g_j \rangle < 0$ for all $i \in [G]$ and the likelihoods of the all trajectories decrease. 
\end{corollary}

\begin{proof}
    From the proof of Lemma 3.2, we know that 
    $$\frac{\mathrm d }{\mathrm d t} \log\pi(y_i|x) = \eta \sum_{j=1}^G A_j \langle g_i, g_j \rangle.$$
    We will decompose this sum into a $\mathcal{Y}^+$ component and a $\mathcal{Y}^-$ component. 
    $$\frac{\mathrm d }{\mathrm d t} \log\pi(y_i|x) = \eta \sum_{j\in \mathcal{Y}^+} A_j \langle g_i, g_j \rangle + \eta \sum_{j\in \mathcal{Y}^-} A_j \langle g_i, g_j \rangle.$$
    We can upper bound the first term using Holder's inequality
    $$\frac{\mathrm d }{\mathrm d t} \log\pi(y_i|x) \leq \eta (\sum_{j\in \mathcal{Y}^+} A_j) \max_{j \in \mathcal{Y}^+}\langle g_i, g_j \rangle + \eta \sum_{j\in \mathcal{Y}^-} A_j \langle g_i, g_j \rangle.$$
    and since all the $A_j$ in second term are negative and all the inner products are positive, we can upper bound the second term as
    $$\eta (\sum_{j\in \mathcal{Y}^-} A_j) \min_{j \in \mathcal{Y}^-}\langle g_i, g_j \rangle$$
    Then, as the advantages are centered, we know that $\sum_{j\in \mathcal{Y}^+} A_j = -\sum_{j\in \mathcal{Y}^-} A_j$ and we have
    $$\frac{\mathrm d }{\mathrm d t} \log\pi(y_i|x) \leq \eta (\sum_{j\in \mathcal{Y}^+} A_j) (\max_{j \in \mathcal{Y}^+}\langle g_i, g_j \rangle - \min_{j \in \mathcal{Y}^-}\langle g_i, g_j \rangle).$$
    By Cauchy-Schwarz and by the assumptions, we have that
    $$\frac{\mathrm d }{\mathrm d t} \log\pi(y_i|x) \leq \eta (\sum_{j\in \mathcal{Y}^+} A_j) (\lVert g_i \rVert \max_{j \in \mathcal{Y}^+}\lVert g_j \rVert - \delta \lVert g_i \rVert \min_{j \in \mathcal{Y}^-}\lVert g_j \rVert).$$
    Then, as $\delta \min_{j \in \mathcal{Y}^-}  \lVert g_j \rVert > \max_{j \in \mathcal{Y}^+} \lVert g_j \rVert$, 
    $$\frac{\mathrm d }{\mathrm d t} \log\pi(y_i|x) \leq \eta (\sum_{j\in \mathcal{Y}^+} A_j) \lVert g_i \rVert( \max_{j \in \mathcal{Y}^+}\lVert g_j \rVert - \delta \min_{j \in \mathcal{Y}^-}\lVert g_j \rVert) < 0,$$
    completing the proof.
\end{proof}

\newcommand{\Cov}{\mathrm{Cov}}

\begin{lemma}[Restatement of Lemma~\ref{lem:entropy-increase}] 
Fix a prompt $x$, and let $\pi_t(y):=\pi_{\theta_t}(y\mid x)$ be the distribution at the training step $t$. If $\Cov(\log\pi_t(y),\dot\ell_t(y))<0$, then $\frac{\mathrm d}{\mathrm dt}H_t>0$. 
\end{lemma}

\begin{proof}
    Differentiate the entropy:
    \begin{equation}
    \frac{d}{dt}H_t=-\sum_y \frac{d\pi_t(y)}{dt}\bigl(1+\log\pi_t(y)\bigr).    
    \end{equation}
    Since $\sum_y \pi_t(y)=1$, we have $\sum_y \frac{d\pi_t(y)}{dt}=0$. Define $\dot\ell_t(y) = \frac{\mathrm d}{\mathrm d t} \log \pi(y|x)$, we have
    \begin{equation}
    \frac{d}{dt}H_t=-\sum_y \frac{d\pi_t(y)}{dt}\log\pi_t(y)=-\sum_y \pi_t(y)\dot\ell_t(y)\log\pi_t(y) = -\Cov(\log\pi_t(y),\dot\ell_t(y)).    
    \end{equation}
    Then, as we have assumed $\Cov(\log\pi_t(y),\dot\ell_t(y))<0$, $\frac{\mathrm d}{\mathrm dt}H_t>0$.
\end{proof}

\section{Further Analyses}

\subsection{Further Evidence Supporting the Training Dynamic of Entropy Eruption in Agent RL}
\label{apd:sec:entropy-eruption}
Beyond the examples we show in Figure~\ref{fig:entropy-eruption}, we exhibit more experimental instances with Llama3.2-1B and Qwen2.5-1.5B on AlfWorld, WebShop and Search tasks. For AlfWorld and WebShop, we adopt mid-training on Llama3.2-1B similarly to our main experiments with corpora from Openmanus-RL~\citep{OpenManus}. As shown in Figure~\ref{fig:entropy-eruption-apd}, Llama3.2-1B and Qwen2.5-1.5B consistently exhibit the same entropy-eruption pattern across all three agent tasks, indicating that this phenomenon is not specific to a single model or environment. In AlfWorld, the entropy remains relatively controlled in the early stage and then rises noticeably in the middle-to-late stage, while the reward has already reached a plateau. In WebShop, the effect is even more pronounced: entropy undergoes a sharp late-stage escalation, accompanied by reward stagnation or even collapse, especially for Llama3.2-1B where the reward drops abruptly after the entropy spike. In Search, models also show a clear entropy burst whenever reward enters a plateau. Overall, these results provide further evidence that entropy eruption is a robust training dynamic in agent RL and is often associated with reward instability or degradation once training enters the later phase.

\subsection{Mid-training shrinks the length of phase 1.}
\label{apd:sec:mid-training}

Figure~\ref{fig:entropy-eruption-apd} shows that models initiated with mid-training have a substantially shorter, or even nearly nonexistent, Phase 1. 
This is expected because mid-training already instills the basic formatting and validity constraints required for function calling. Meanwhile, it also improves the policy’s overall trajectory quality. As a result, the policy enters RL from a substantially better initialization, with a higher valid-action ratio and a stronger average reward. Consequently, early RL no longer needs to spend many updates correcting format errors, and much of the coarse entropy-reduction process that typically characterizes Phase 1 has already occurred before RL begins. 

\subsection{Representations of Agent Trajectories are Significantly More Concentrated than Those of Alignment Trajectories.}
\label{apd:sec:representation-similarity}
From Phase 1, the policy model would quickly learn the basic rules of output format and function calling, thus making the sampling focus on valid trajectories. Here, we show that the representation of valid agent trajectories shares a very similar representation in the LLM semantic space. Figure~\ref{fig:geometry} visualizes the last-layer hidden representations of Qwen2.5-7B-Instruct, averaged over response tokens, for trajectories sampled from agent tasks (ALFWorld, WebShop, Search) and from a non-agent alignment task (UltraFeedback~\citep{ultrafeedback}). Agent trajectories form significantly tighter clusters than alignment trajectories.  Next, we show that the high representation similarity among agent trajectories leads to entropy eruption.

\begin{figure}[t!]
    \centering
    \includegraphics[width=0.32\linewidth]{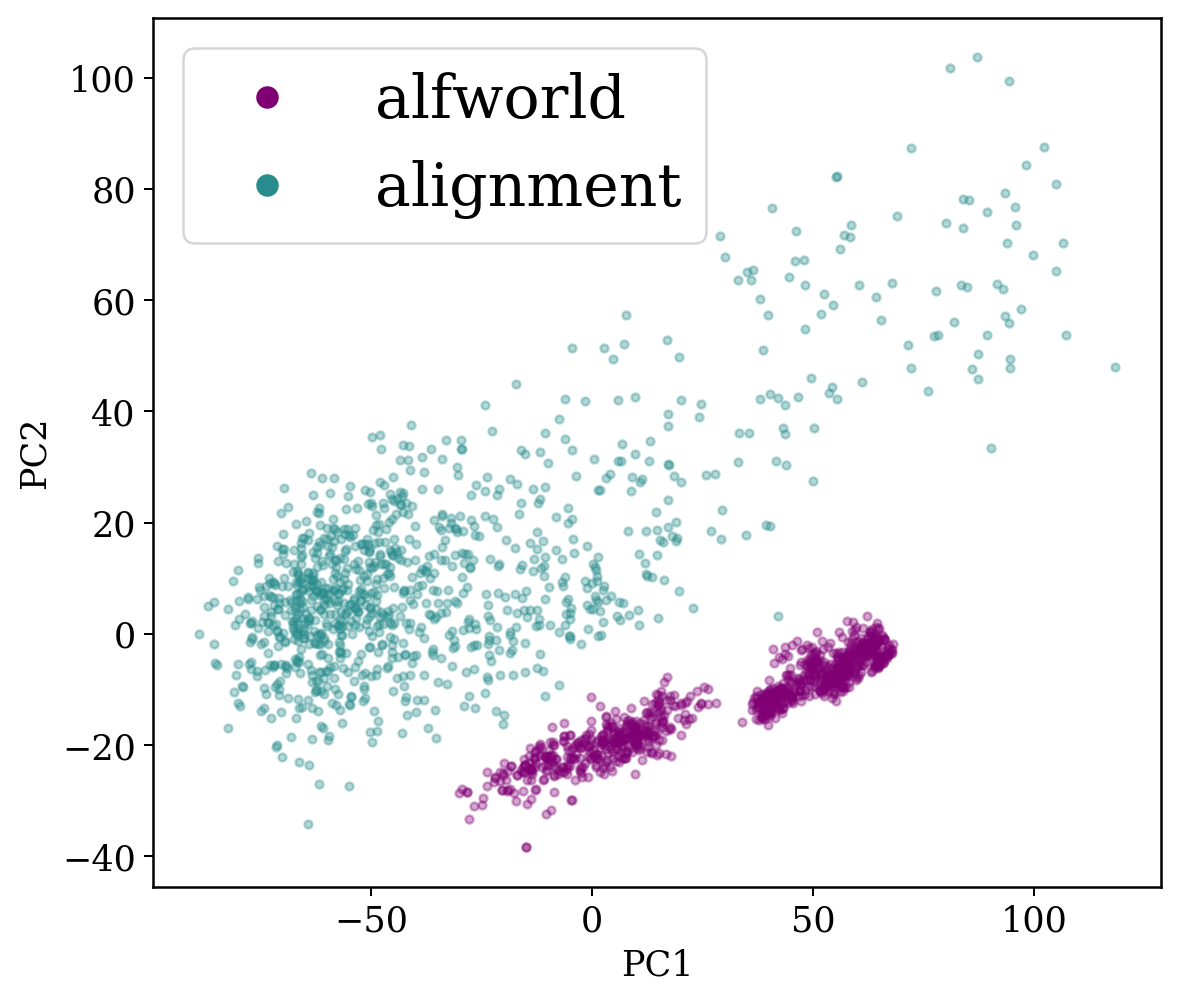}
    \includegraphics[width=0.32\linewidth]{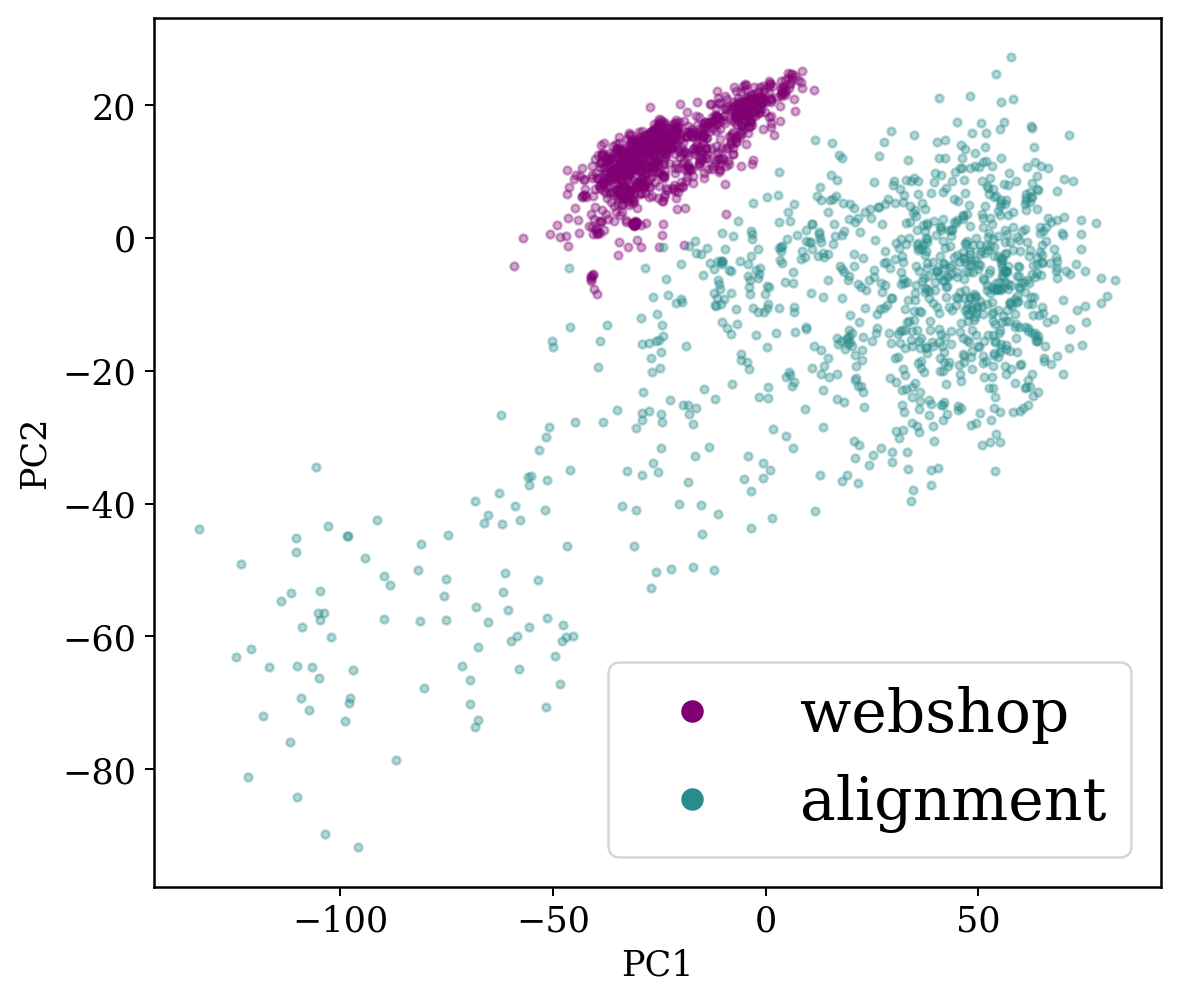}
    \includegraphics[width=0.32\linewidth]{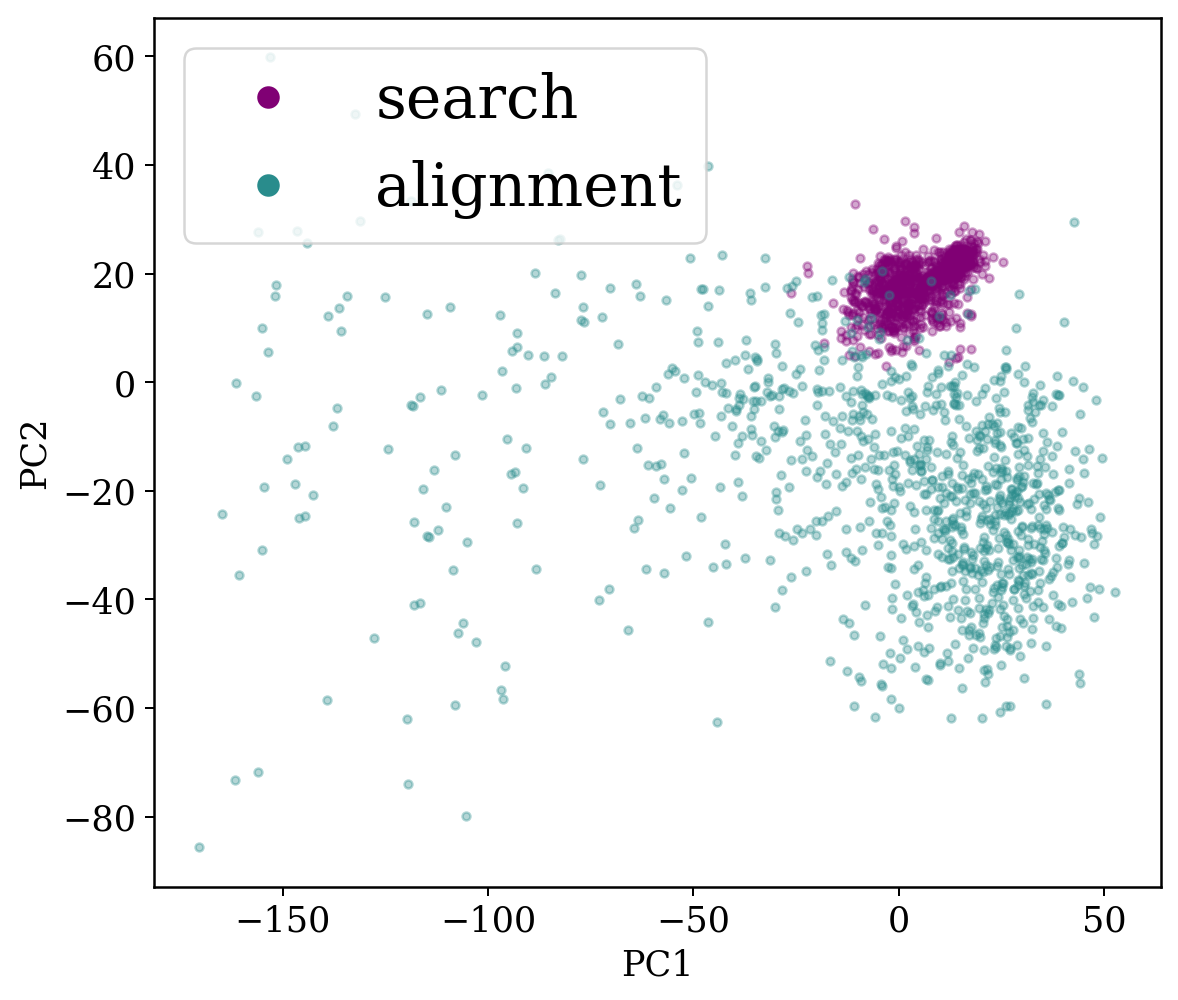}
    \vspace{-1em}
    \caption{Representation distribution of trajectories in \textcolor{violet}{agent} and \textcolor{teal}{alignment} tasks. The trajectory representation in the agent task is significantly more concentrated.}
    \label{fig:geometry-apd}
\end{figure}

\subsection{Further Evidence Supporting the Likelihood Decrease of Valid Trajectories}
\label{apd:sec:likelihood-decrease}
To complement our main results, we further track the training-time log-probability of \emph{correct}, \emph{wrong}, and \emph{valid} responses for Llama3.2-1B and Qwen2.5-1.5B on AlfWorld, WebShop, and Search. As shown in Figure~\ref{fig:likelihood-apd}, the likelihood assigned to valid trajectories surprisingly does not increase along with training, instead, across all three agent tasks and both model families, it either gradually declines after an early improvement phase or undergoes an abrupt late-stage collapse. In AlfWorld, both models show a gentle fluctuation in the log-probability of valid responses, with a clear decreasing overall tendency. In Search, the effect is even more pronounced: after a relatively stable phase, the log-probability of valid trajectories drops sharply in the later stage and only partially recovers. In WebShop, the phenomenon is strongest, especially for Llama3.2-1B, where the likelihood of valid responses collapses dramatically during late training.

\subsection{Further Evidence on Occurrence of Degenerate Patterns during Entropy Eruption.}
\label{apd:sec:toxic-pattern}

As the entropy erupts, the distribution is flattened, which would make the degenerate patterns more likely to be sampled in the training. We found that the entropy eruption is strongly synchronized with the emergence of some degenerate patterns such as sentence duplication. To empirically show this, we collect the 1024 sampled trajectories at different checkpoints during training, and analyze their sentence duplication ratio. Specifically, we first split the responses into sentence chunks, and measure the fraction of sentences that appear multiple times in the same response. As shown by the purple lines in Figure~\ref{fig:toxic-apd}, the sentence duplication ratio surges during the entropy eruption, and has the same fluctuation patterns as the entropy dynamics. 
Meanwhile, as the orange line shown in Figure~\ref{fig:toxic-apd}, other kinds of error are also likely to appear during the entropy eruption, such as the format error.

\begin{figure}[t!]
    \centering
    \includegraphics[width=0.32\linewidth]{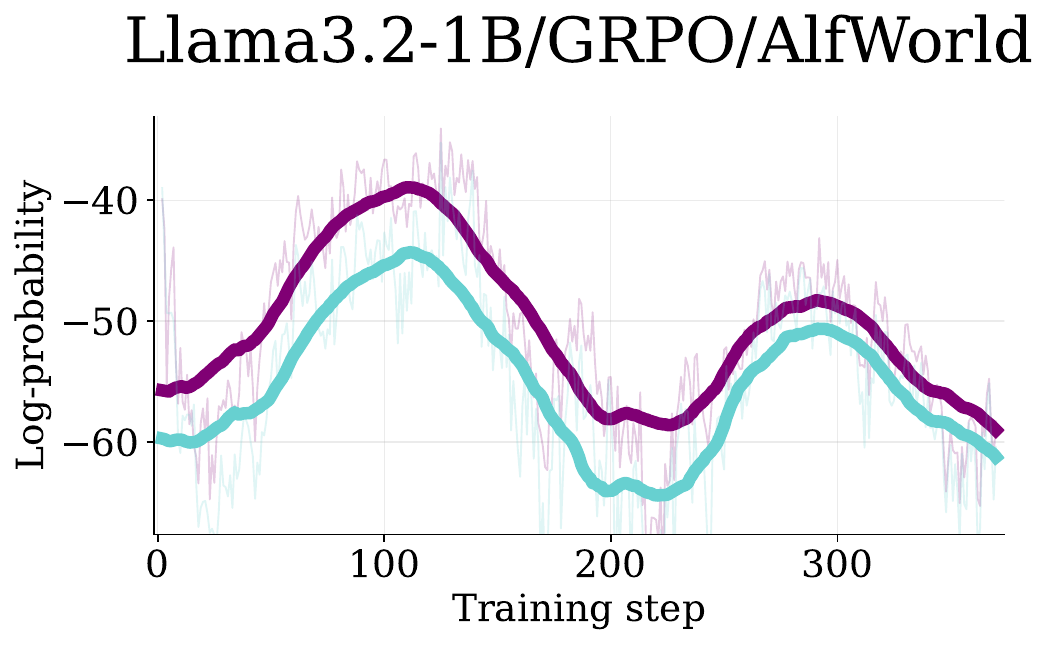}
    \includegraphics[width=0.32\linewidth]{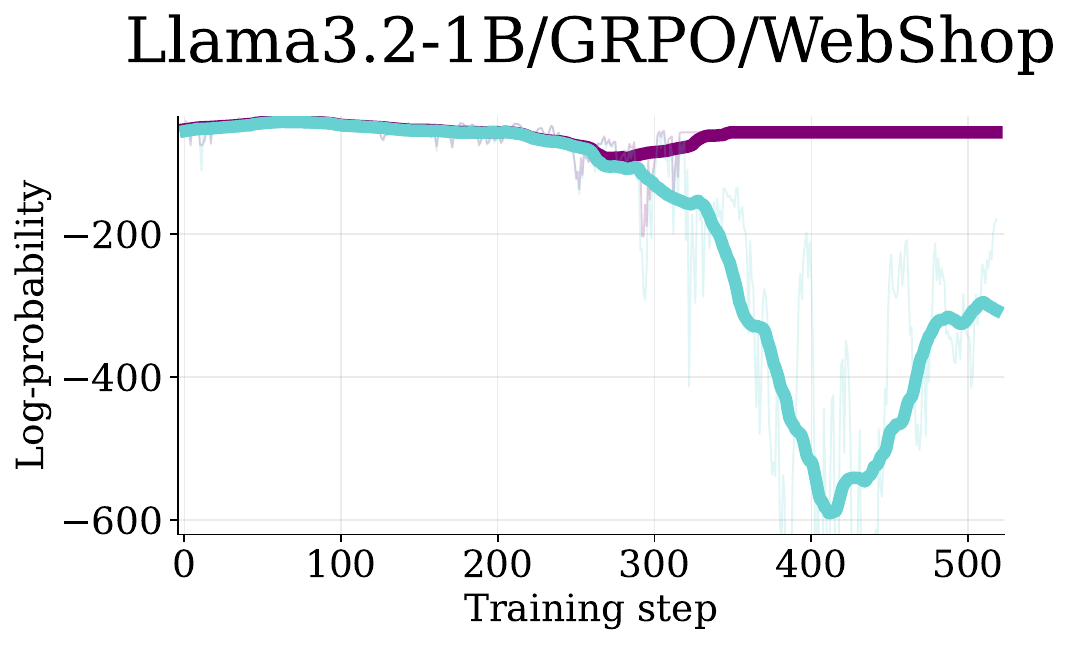}
    \includegraphics[width=0.32\linewidth]{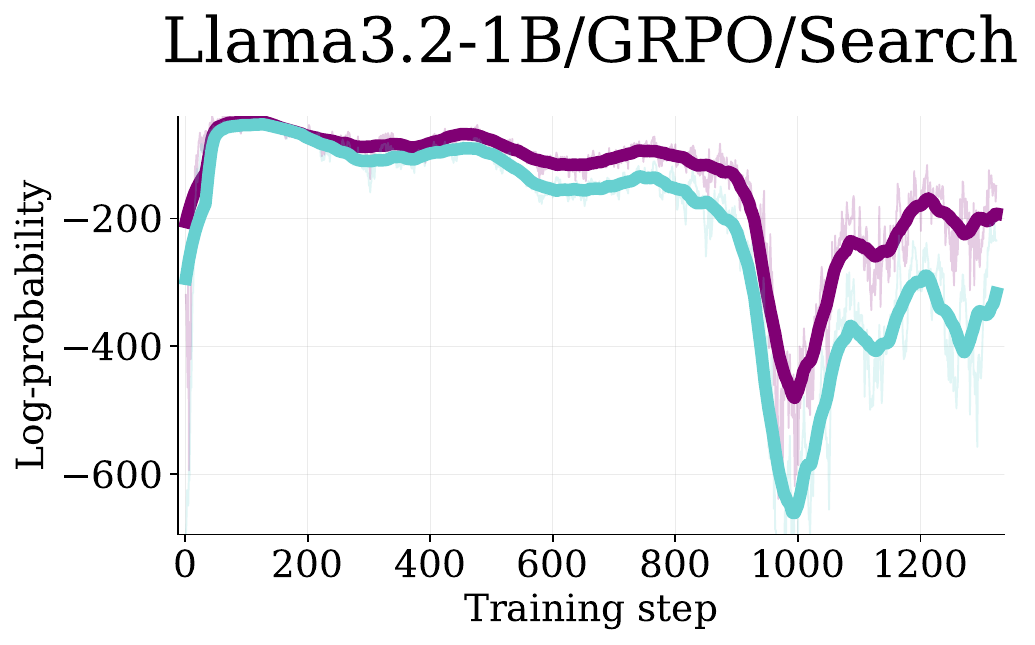}
    \includegraphics[width=0.32\linewidth]{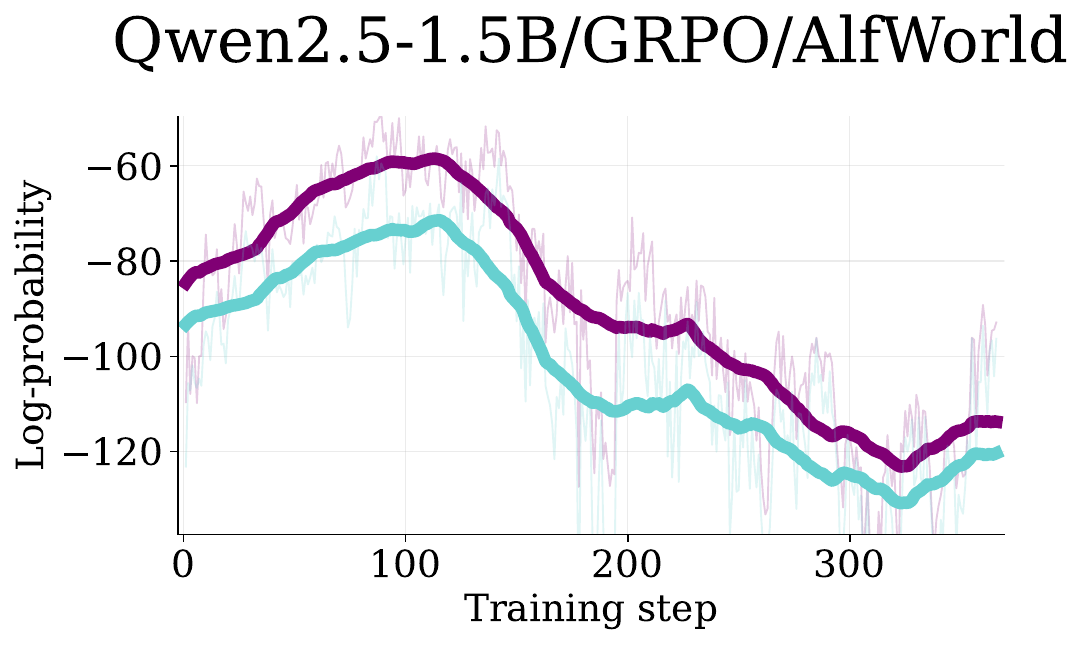}
    \includegraphics[width=0.32\linewidth]{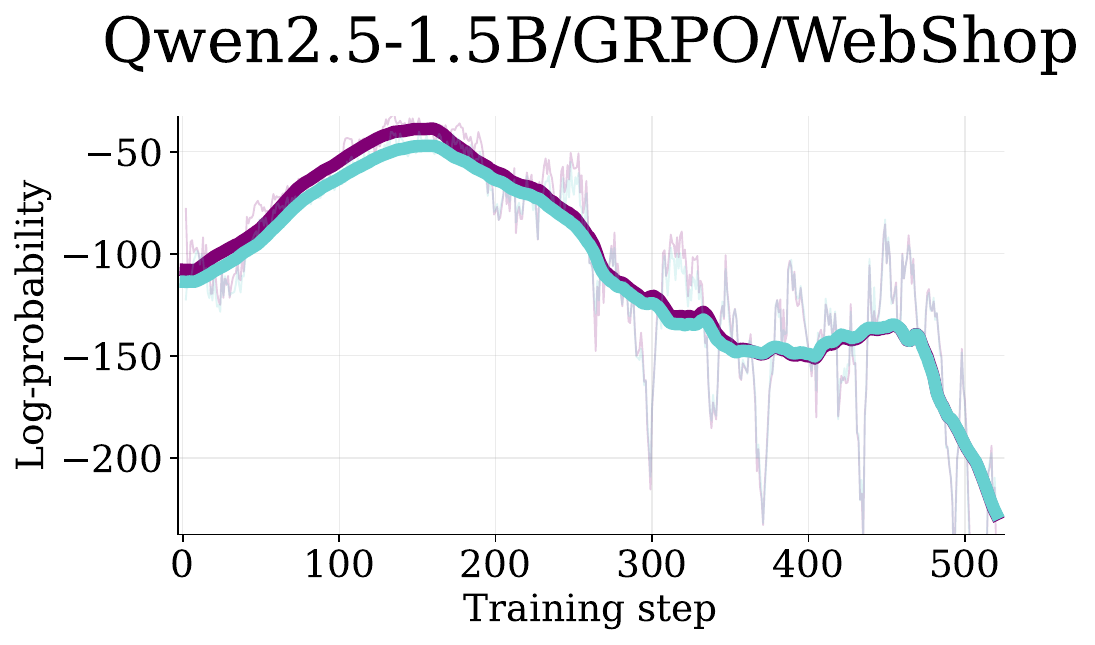}
    \includegraphics[width=0.32\linewidth]{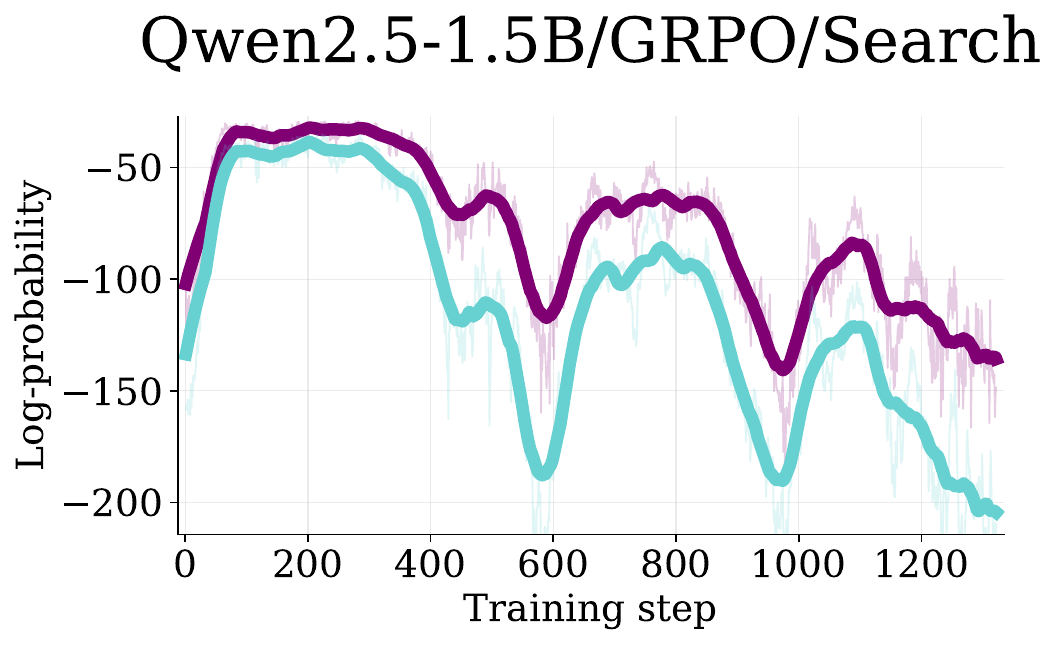}
    
    \caption{Further evidence supporting the likelihood decrease of valid trajectories during agent RL}
    \label{fig:likelihood-apd}
\end{figure}
\begin{figure}[t!]
    \centering
\includegraphics[width=0.38\textwidth]{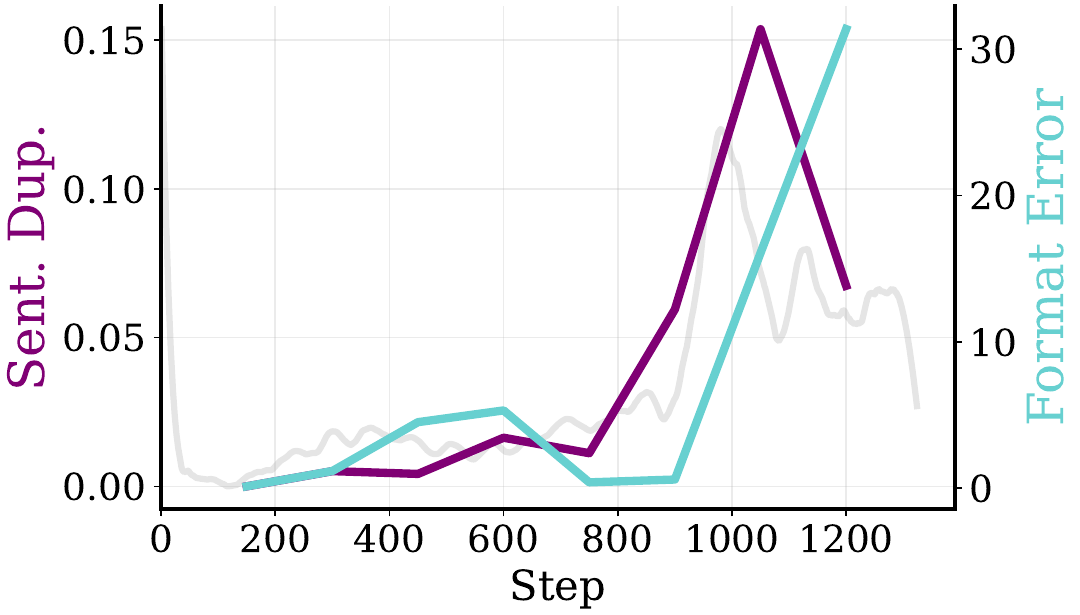}
    \includegraphics[width=0.38\textwidth]{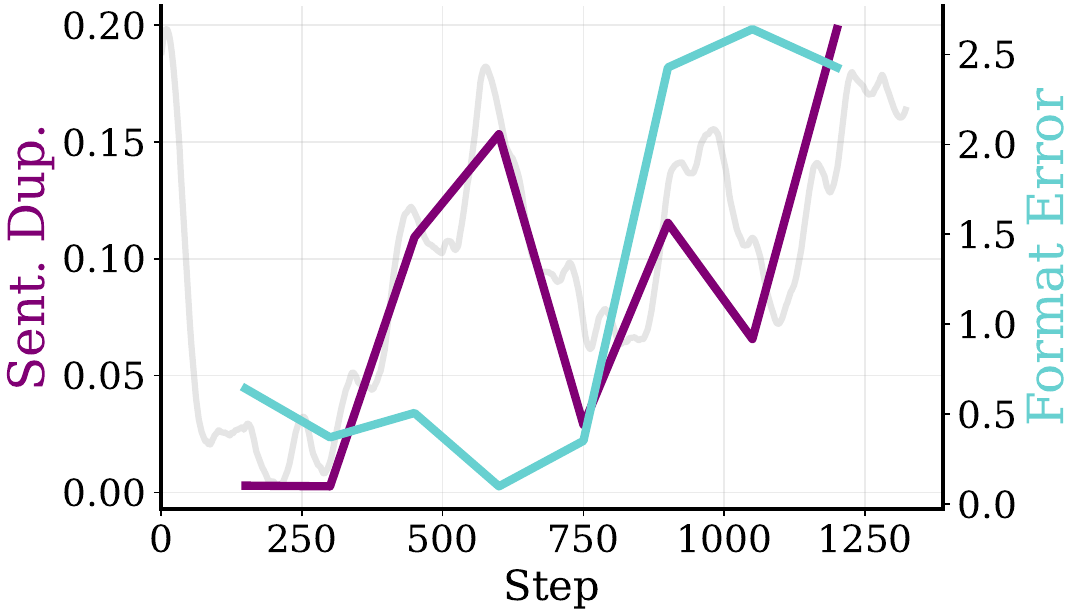}

    \caption{\textcolor[HTML]{800074}{Sentence Duplication Ratio} rises with entropy eruption, which also increases the occurrence of \textcolor[HTML]{67d0d0}{format error}. The gray line in the background is the entropy curve.}
    \label{fig:toxic-apd}
\end{figure}

\subsection{Further Evidence on Lasting Effect of Entropy Eruption.}
\label{apd:sec:lasting-effect}
As shown in Figure~\ref{fig:toxic-apd}, although training temporarily becomes normal after entropy eruption subsides, the damage introduced during the eruption can remain. When entropy erupts, the policy distribution flattens and valid trajectories receive lower probability, so sampled trajectories are more likely to contain degenerate patterns even among responses that still obtain the correct answer. 
If these degenerate trajectories receive positive reward, the policy can inadvertently learn and reinforce such garbage patterns. Consequently, the harmful effects of entropy eruption may persist after the eruption subsides and can even be amplified during entropy subsidence, as shown by the format-error curves (light blue) in Figure~\ref{fig:toxic-apd}. We provide additional analysis and statistics in Appendix~\ref{apd:sec:lasting-effect}.
In more severe cases, entropy eruption can lead to semantic collapse and completely break RL training. As shown in Figure~\ref{fig:entropy-eruption-apd}, for Llama3.2-1B on WebShop, entropy eruption is followed by a collapse of the reward curve to zero.

\section{Further Validation For SEAL}

\begin{figure}[t]
    \centering
    \includegraphics[width=0.32\linewidth]{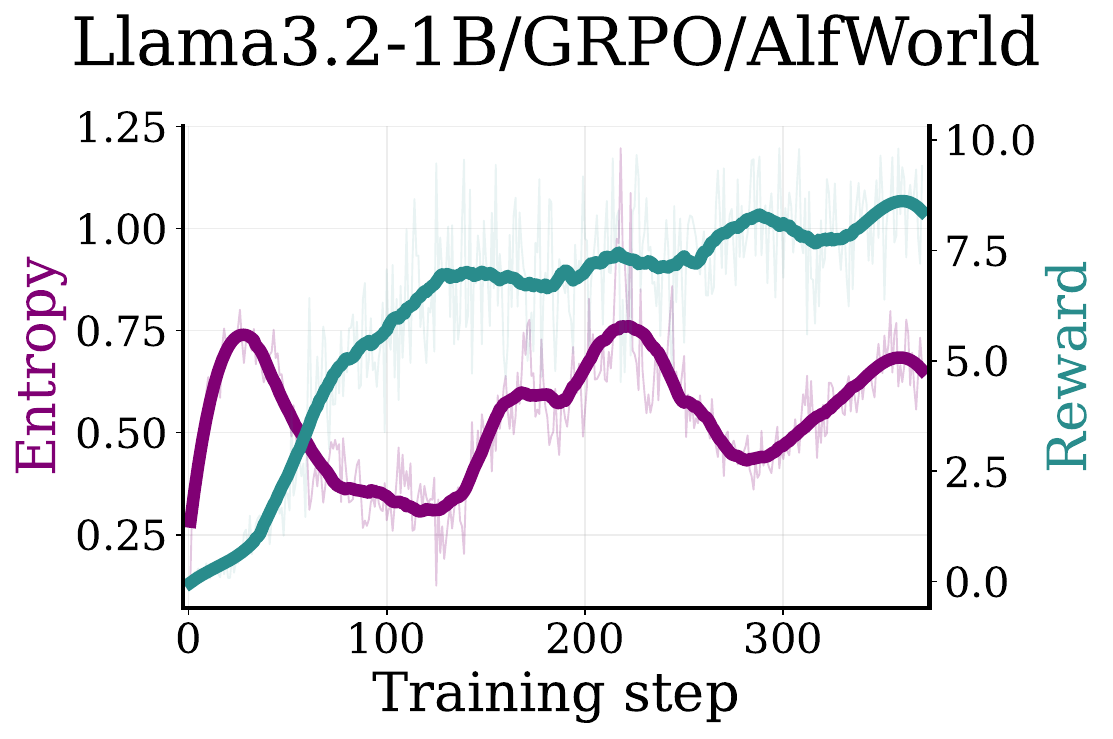}
    \includegraphics[width=0.32\linewidth]{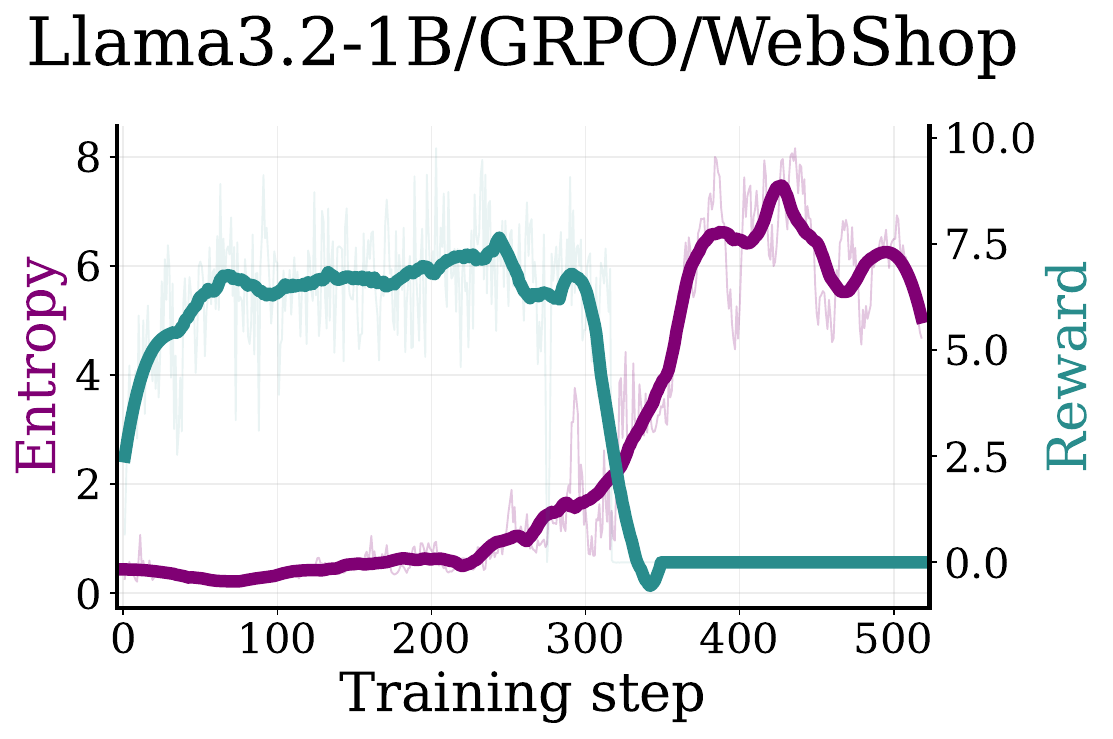}
    \includegraphics[width=0.32\linewidth]{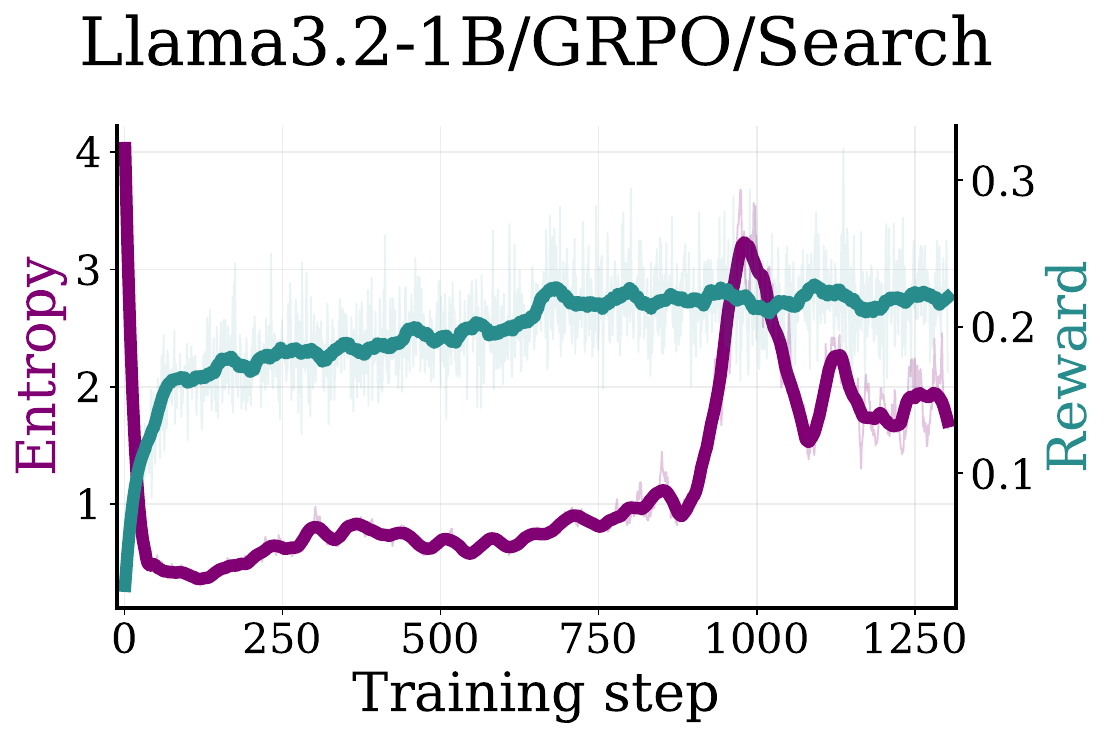}
    \includegraphics[width=0.32\linewidth]{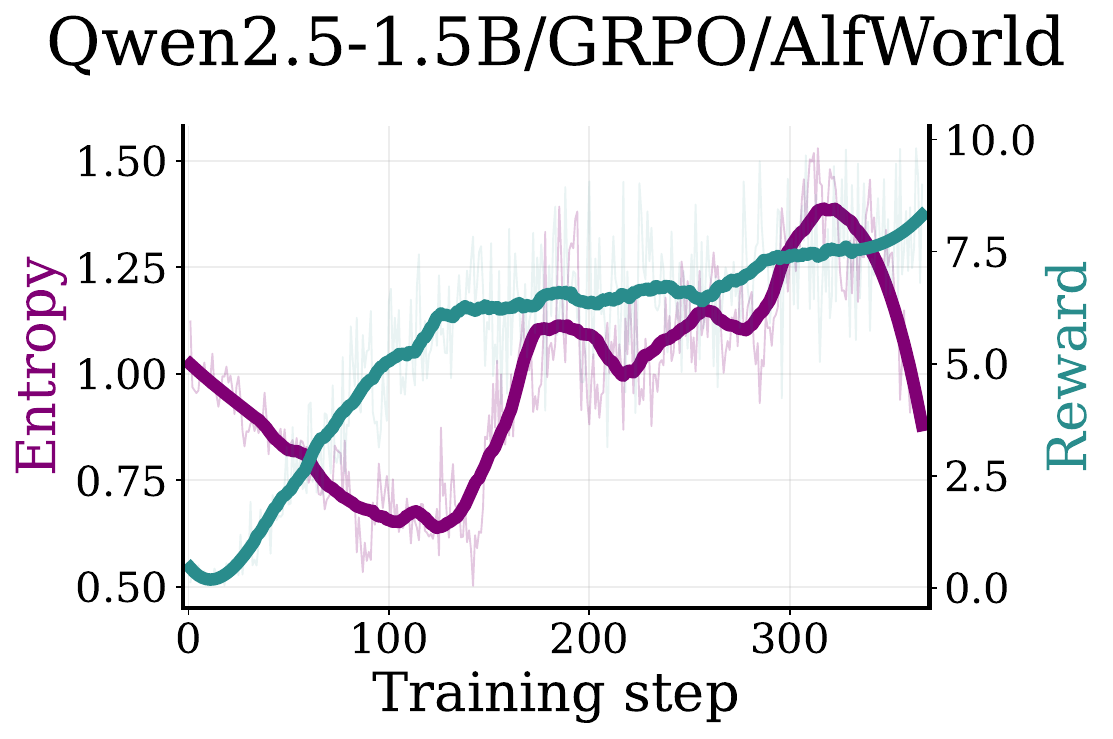}
    \includegraphics[width=0.32\linewidth]{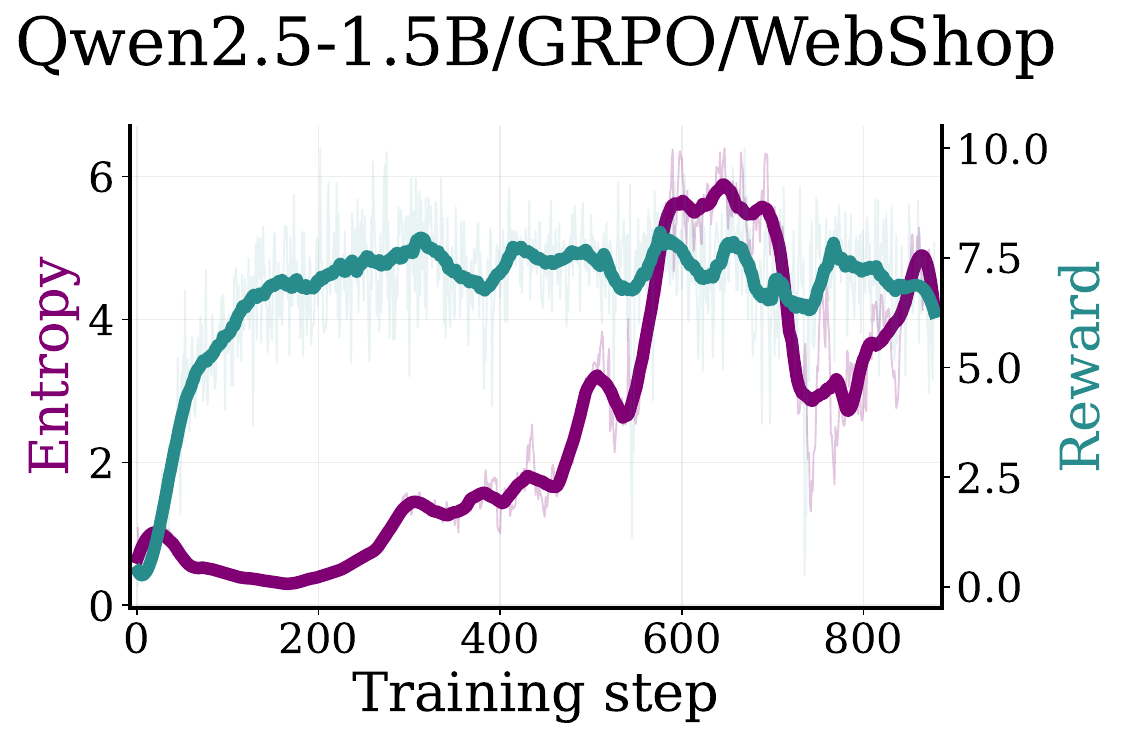}
    \includegraphics[width=0.32\linewidth]{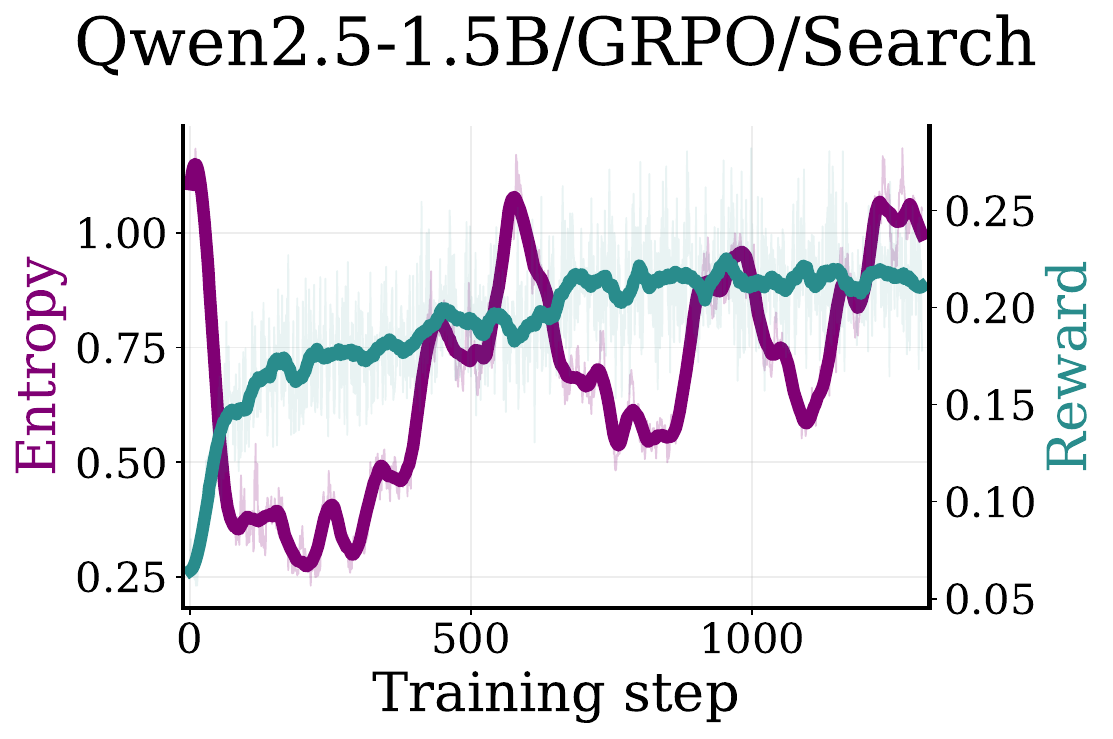}
    \includegraphics[width=0.32\linewidth]{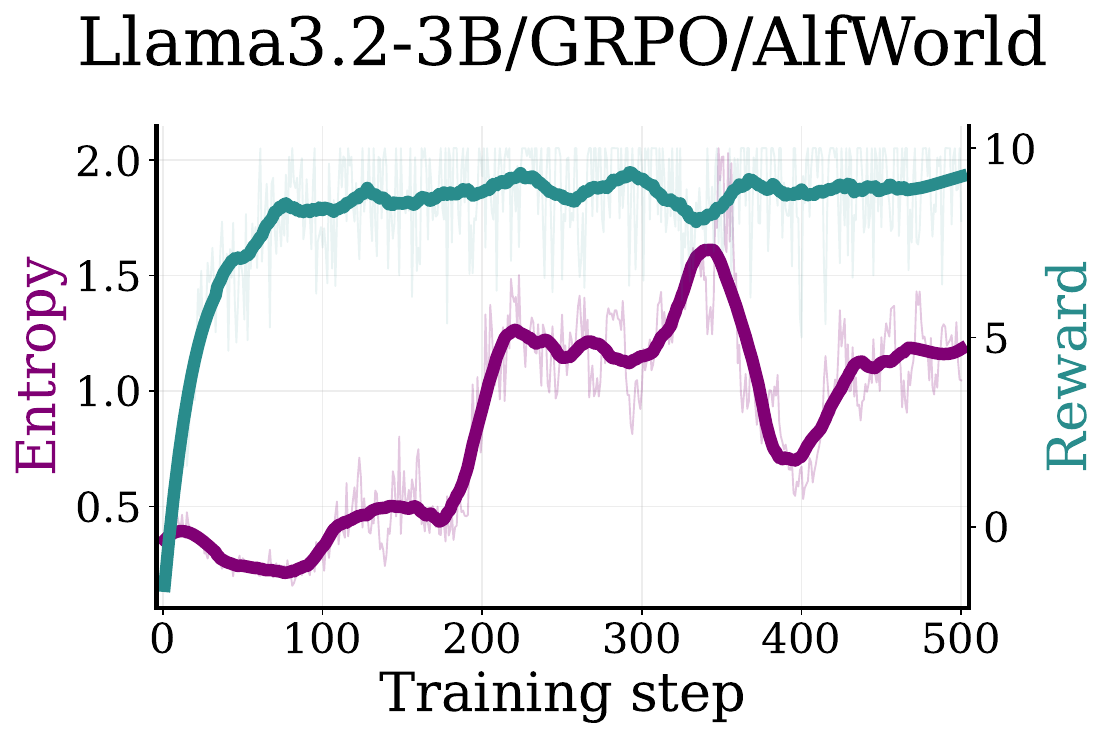}
    \includegraphics[width=0.32\linewidth]{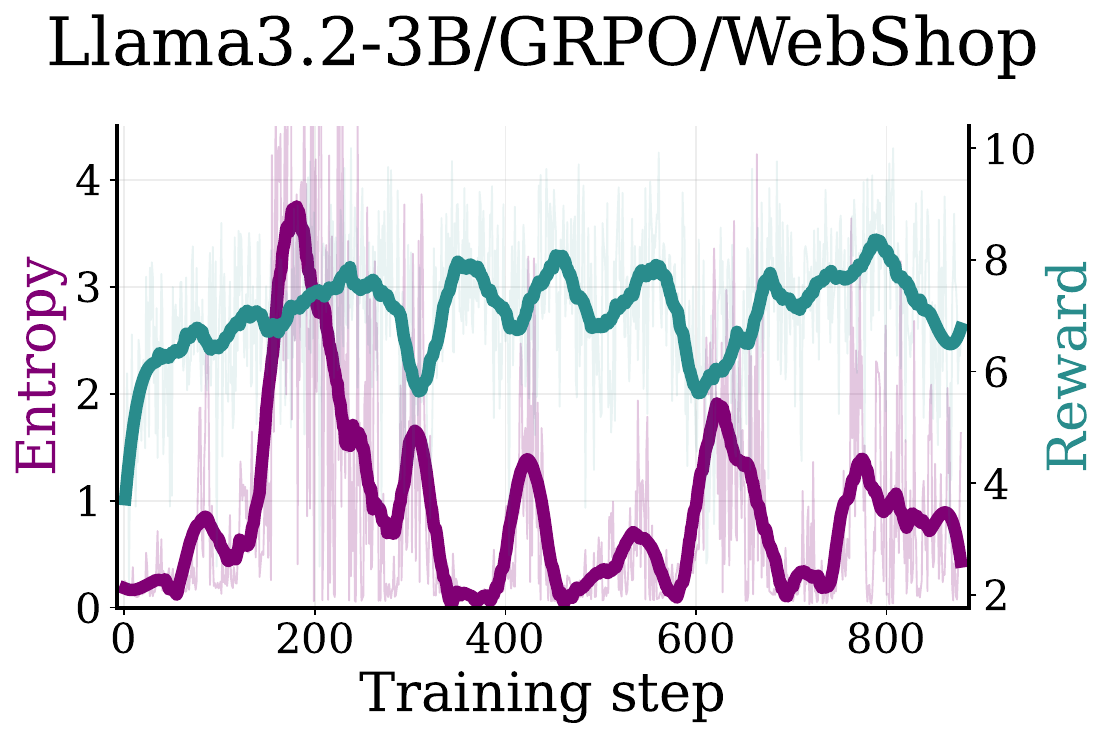}
    \includegraphics[width=0.32\linewidth]{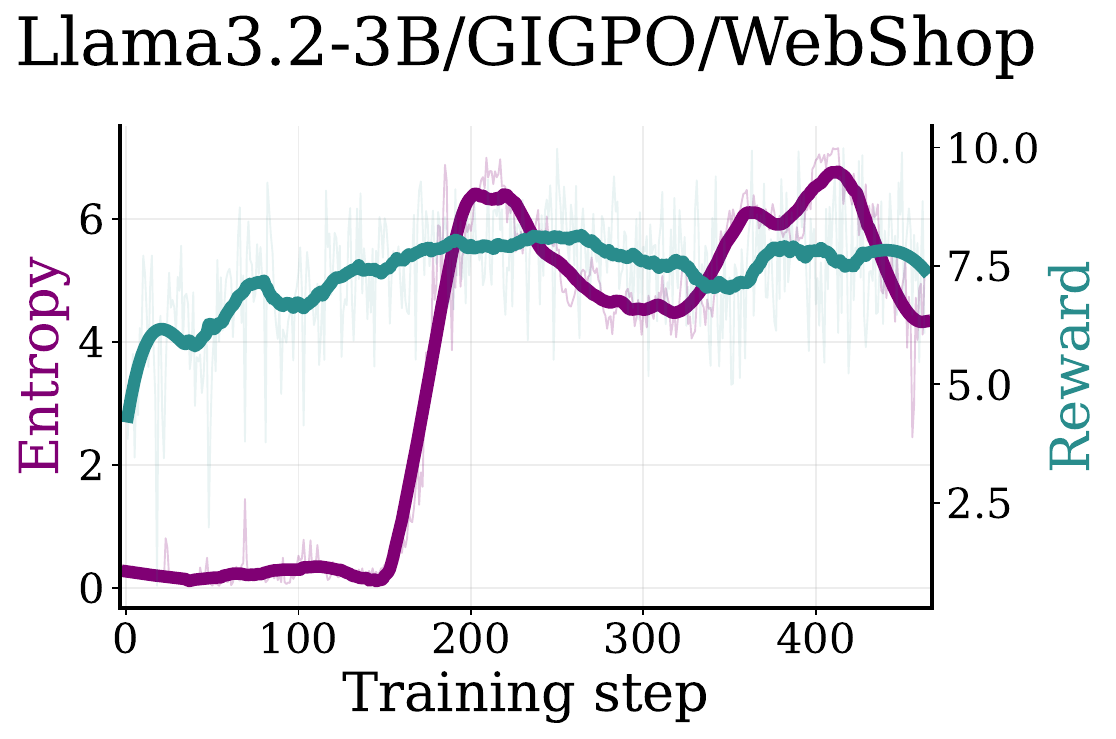}
    \includegraphics[width=0.32\linewidth]{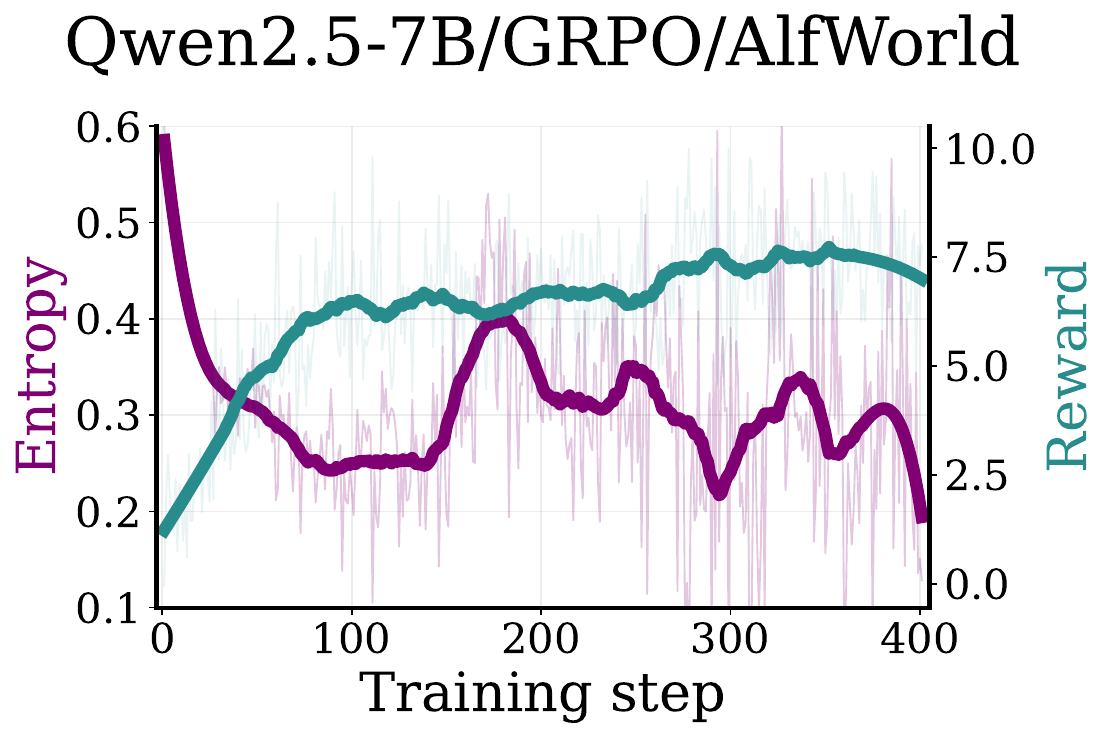}
    \includegraphics[width=0.32\linewidth]{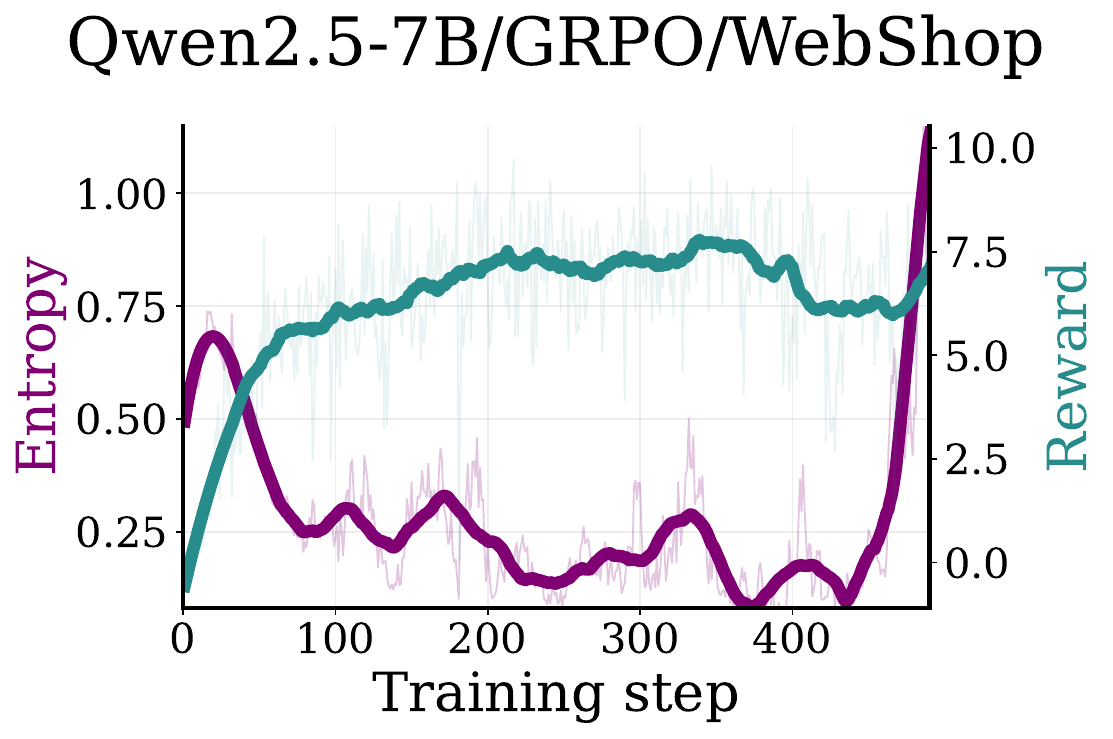}
    \includegraphics[width=0.32\linewidth]{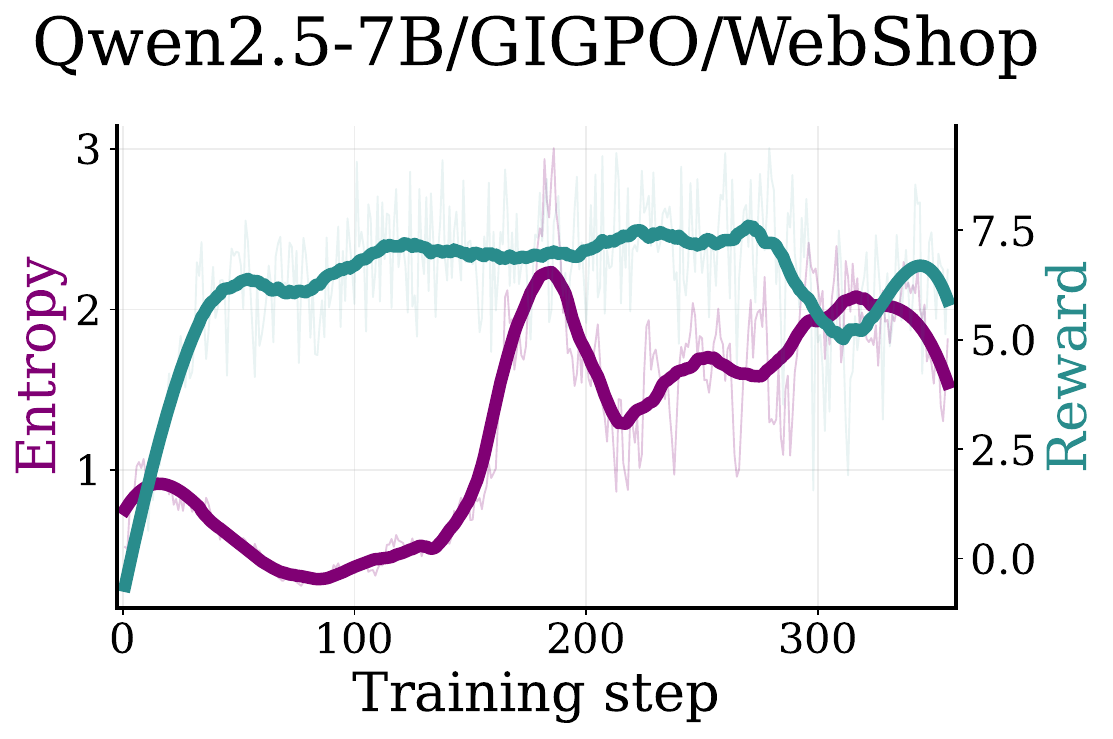}
    \caption{Further evidence supporting the training dynamic of entropy eruption in agent RL. In the AlfWord and WebShop tasks, the length of Phase 1 in Llama-series models is shortened due to mid-training as demonstrated in Section~\ref{sec:dynamics}.}
    \label{fig:entropy-eruption-apd}
\end{figure}

\subsection{Implementation Details}
\label{apd:sec:implementation-details}
The experiments of Qwen2.5-1.5B and Llama3.2-1B are conducted on 2 NVIDIA A100-SXM4-80GB GPUs, and the experiments of Qwen2.5-7B and Llama3.2-3B are conducted on 2 NVIDIA H200 GPUs. 
We mostly follow experimental settings in \cite{gigpo}. 
We fix the random seed of the environment to 0 for reproducibility.

For the AlfWorld task, the maximum prompt length is 2048 tokens, and the maximum response length is 512 tokens. Each
episode allows up to 25 environment steps. The learning rate is set to 2e-6 for the actor. The reward is set to 10 for success
and to 0 for failure. We also apply a reward penalty of -0.1 to punish the format errors. We use a group size of 8. The rollout temperature is set to 1.0, while the validation temperature is set to 0.4. The mini-batch size is 256, and the KL-divergence loss coefficient is set to 0.01. For GIGPO, we also keep the best hyperparameter according to the original paper \citep{gigpo}, where the weighting coefficient is fixed at 1 without further tuning, and the discount factor is set to 0.95.

For the WebShop task, the maximum prompt length is 4096 tokens, and the maximum response length is 512 tokens. Each episode is limited to 10 environment steps. The learning rate is 2e-6 for the actor. As in ALFWorld, reward is set to 10 for success and 0 for failure. Invalid actions are penalized with a reward of -0.1. Group size is set to 8. The rollout temperature is set to 1.0, while the validation temperature is set to 0.4. The mini-batch size is 64, and the KL-divergence loss coefficient is set to 0.01. For GIGPO, we also keep the best hyperparameter according to the original paper \citep{gigpo}, where the weighting coefficient is fixed at 1 without further tuning, and the discount factor is set to 0.95.

For the search-augmented QA task, the maximum prompt length is 4096 tokens, and the maximum response length is 512 tokens. The max turn is set to 3. The learning rate is 2e-6 for
the actor. Reward is set to 1 for success and 0 for failure. Invalid actions are penalized with a reward of -0.01. We set the train data size to 256 and use a group size of 5 following GIGPO \citep{gigpo} and Search-R1 \citep{searchr1}. Rollout and validation temperatures are set to 1.0 and 0.0, respectively. The mini-batch size is 512, and the KL-divergence loss coefficient is set to 0.001. The weighting coefficient is fixed at 1 without further tuning, and the discount factor is set to 0.95.

To evaluate the degenerate score with GPT-4, we use the prompt in Box~\ref{box:gptprompt}.

\begin{figure*}[t]
\centering
\begin{mybox}[GPT-4 prompt for degeneration judgment]{box:gptprompt}
<System Prompt>\\
You are an expert evaluator of language-model generations.\\
\\
Your task is to judge how much a model output contains:\\
1. semantic garbage / nonsensical content\\
2. repetition / looping / duplicated phrases\\
3. off-task filler that does not help solve the task\\
\\
You are NOT judging whether the output is polite, safe, or toxic in the moderation sense.\\
Here, "toxicity" means degeneration: nonsense, repetition, meaningless filler, or severe loss of grounding.\\
\\
Return a score from 0 to 10, where a HIGHER score means WORSE generation quality:\\
- 0 = clean, coherent, non-repetitive\\
- 2.5 = mild repetition or mild filler\\
- 5 = noticeable nonsense/repetition affecting usefulness\\
- 7.5 = severe degeneration; much of the text is repetitive or semantically broken\\
- 10 = almost entirely garbage, repetitive looping, or incoherent\\
\\
Important rules:\\
- Penalize repeated phrases, repeated constraints, repeated boilerplate, repeated clauses, and low-information loops.\\
- Penalize semantic garbage: irrelevant claims, contradictory statements, hallucinated filler, strange topic drift, or sentences that are grammatical but meaningless in context.\\
- Penalize generic filler when it replaces task-relevant reasoning.\\
- Do NOT penalize short necessary restatements once.\\
- Do NOT penalize correct task-specific repetition if it is needed and minimal.\\
- If the output includes actions, also penalize actions that are inconsistent with the stated action space or observation.\\
\\
Scoring principle:\\
Estimate the proportion and severity of degenerate content in the output.\\
Higher score = larger portion of nonsense/repetition and/or stronger severity.\\
\\
Return ONLY a number from 0 to 10, with no additional text or explanation.\\
\\
<User Prompt>\\
PROBLEM:\\
{problem}\\
SOLUTION:\\
{response}
\end{mybox}
\end{figure*}

\begin{figure}[t!]
    \centering
    \includegraphics[width=0.32\linewidth]{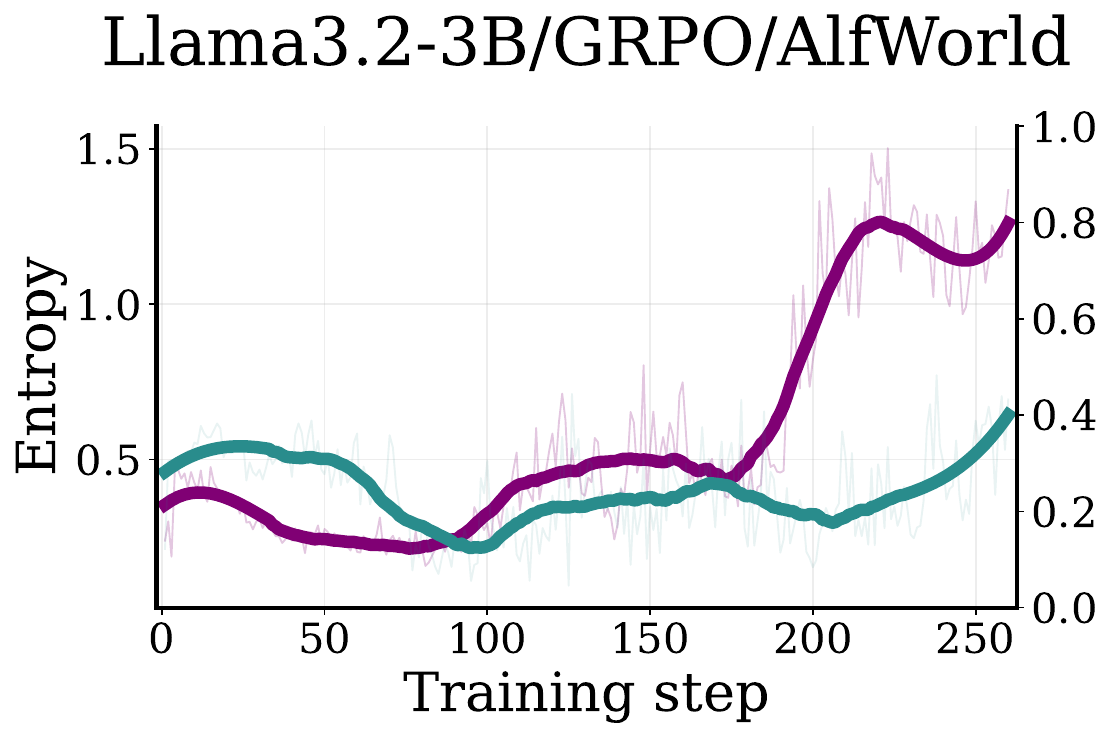}
    \includegraphics[width=0.32\linewidth]{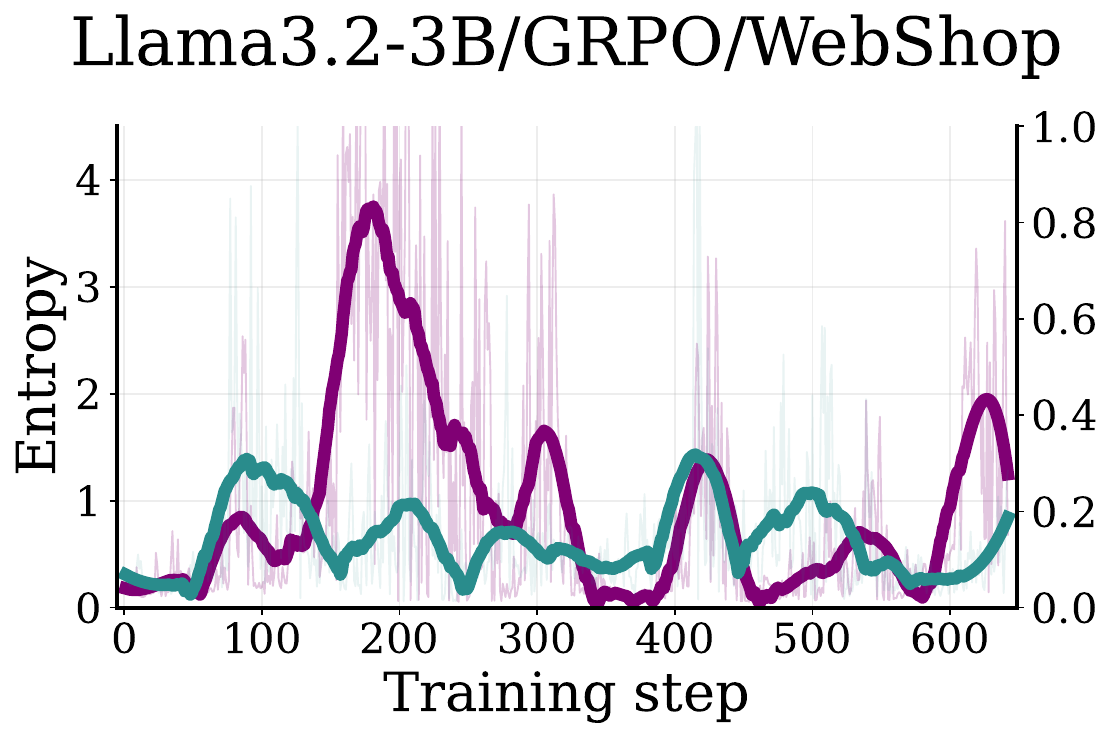}
    \includegraphics[width=0.32\linewidth]{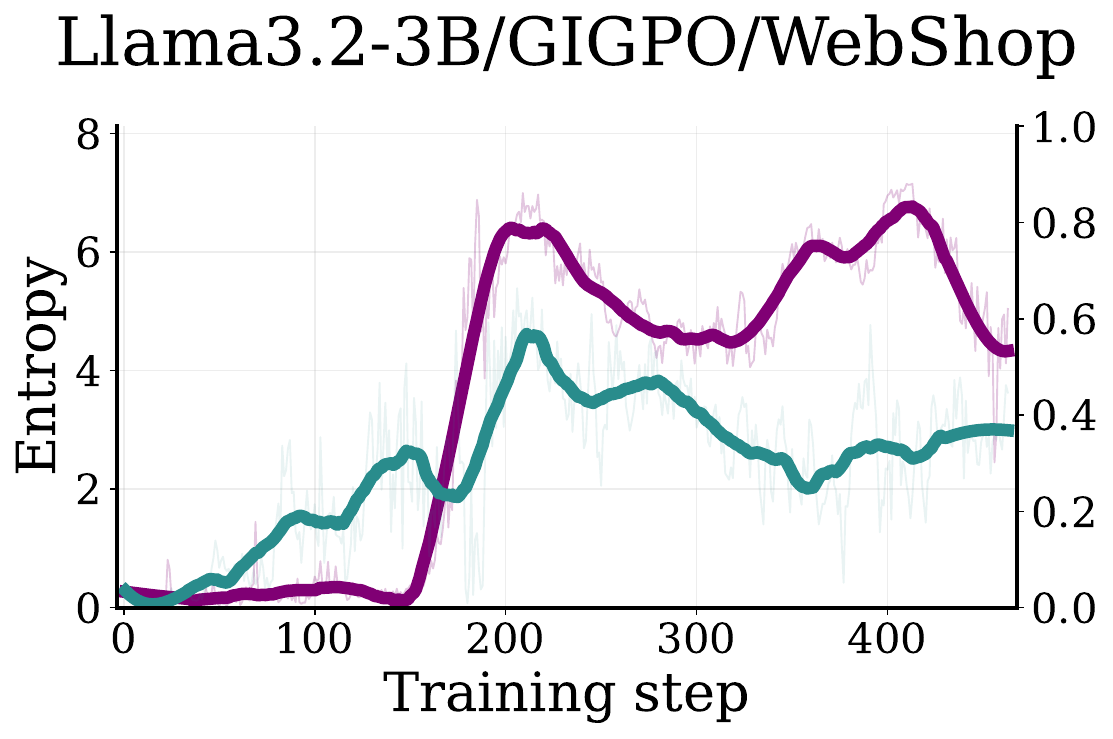}
    \includegraphics[width=0.32\linewidth]{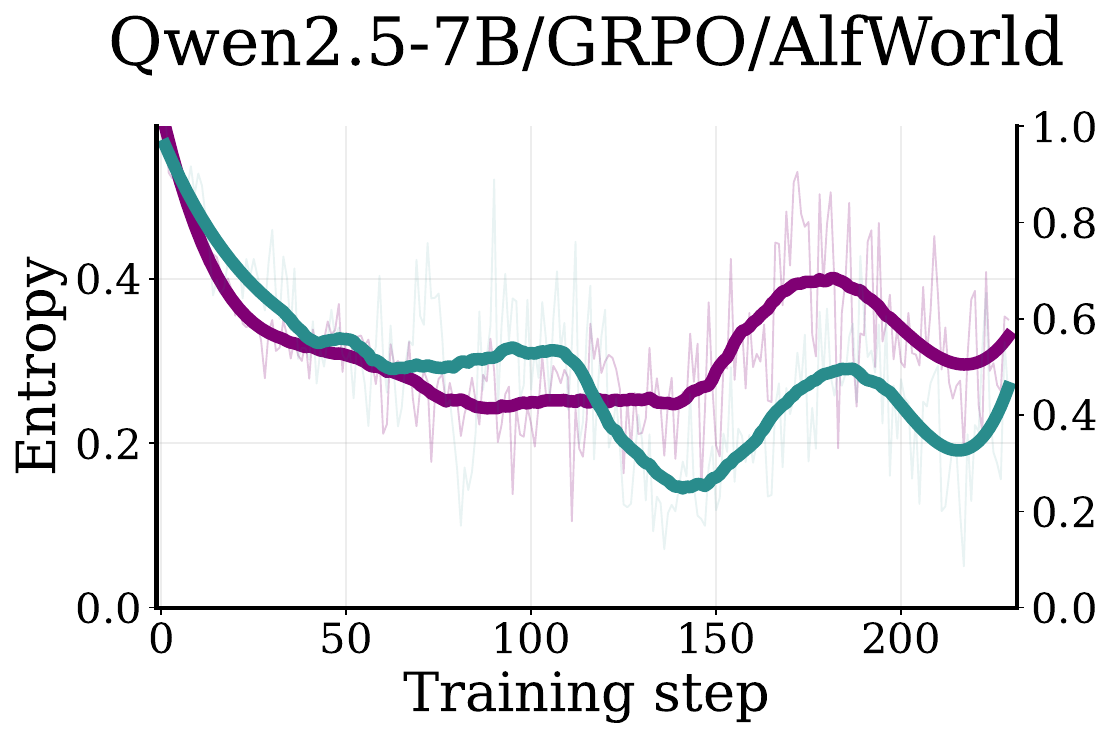}
    \includegraphics[width=0.32\linewidth]{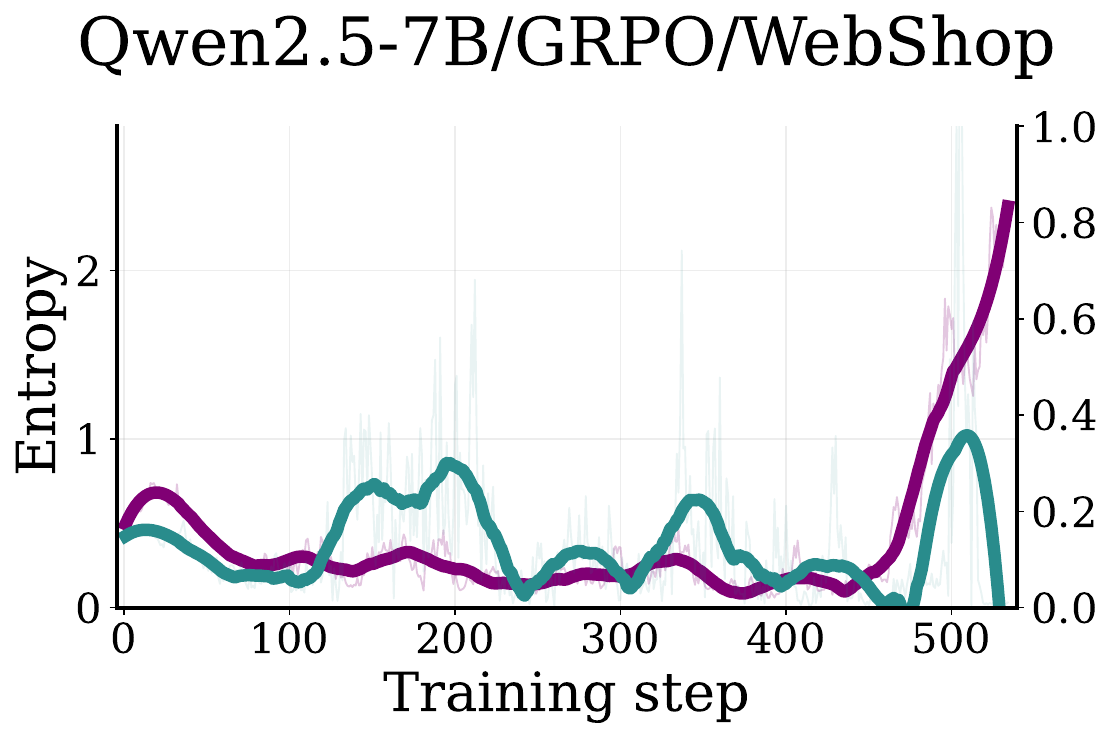}
    \includegraphics[width=0.32\linewidth]{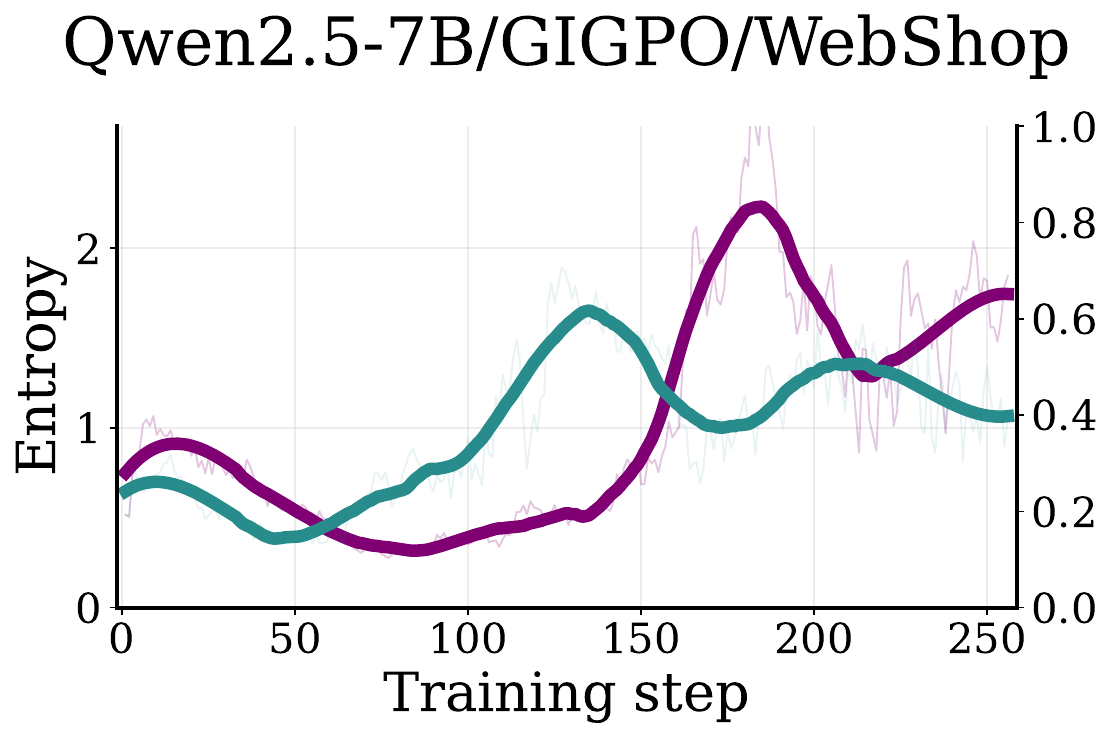}
    \vspace{-0.8em}
    \caption{Entropy dynamics comparison between \textcolor[HTML]{800074}{GRPO/GIGPO} and \textcolor[HTML]{298c8c}{GRPO/GIGPO+SEAL}.}
    \vspace{-1.5em}
    \label{fig:exp-entropy}
\end{figure}

\subsection{SEAL Successfully Alleviates Entropy Eruption}
\label{apd:sec:alleviate-eruption}
As shown in Figure~\ref{fig:exp-entropy}, SEAL consistently mitigates the entropy-eruption phenomenon across different tasks, backbones, and RL algorithms. Compared with standard GRPO/GIGPO, the corresponding runs with SEAL exhibit substantially smoother entropy trajectories, with smaller spike magnitudes and weaker recurrent oscillations. As shown in Figure~\ref{fig:toxic-contrast}, suppressing entropy eruption not only makes training more robust and improves performance, but also reduces degenerate patterns.

\subsection{Hyperparameter Sensitivity}
\label{apd:sec:hyperparameter}
We study the sensitivity of SEAL to the auxiliary-loss weight \(\alpha\) in Figure~\ref{fig:hyperparameter}. Across both Llama3.2-3B and Qwen2.5-7B on WebShop, the method remains relatively stable over a broad range of \(\alpha\), indicating that SEAL does not require highly delicate tuning to be effective. In terms of task performance, both GRPO+SEAL and GIGPO+SEAL achieve their best or near-best accuracy at intermediate values of \(\alpha\), while still maintaining competitive results under smaller or larger settings. At the same time, the LLM-as-a-Judge score on degenerate patterns also remains favorable across most choices of \(\alpha\), with lower scores typically obtained in the same intermediate range. Overall, these results suggest a consistent trade-off: moderate auxiliary-loss strength is usually sufficient to improve representation separation without over-regularizing the policy. Importantly, across a wide range of \(\alpha\), the accuracy and degeneracy metrics of SEAL remain generally better than those of the corresponding baselines without SEAL (see Table~\ref{tab:main-alfworld} and Figure~\ref{fig:toxic-contrast}), showing that the benefit of SEAL is robust rather than tied to a single hyperparameter choice.

\begin{table}[t]
\footnotesize
    \centering
    \caption{Comparison of GRPO performance with and without SEAL on Search-augmented QA for Qwen2.5-7B and Llama3.2-3B.}
    \label{tab:search-apd} 
    \scalebox{0.95}{    \begin{tabular}{cc|cc|ccccc|c}
    \toprule
       \multirow{2}{*}{Model} & \multirow{2}{*}{Methods} & \multicolumn{2}{c|}{In-domain} & \multicolumn{5}{c|}{Out-of-domain}& \multirow{2}{*}{\bf Avg.}\\
       & &NQ & HotpotQA& TriviaQA & PopQA&  2Wiki & MuSiQue& Bamboo.&  \\
       \midrule
       
       \multirow{2}{*}{\makecell{Llama3.2\\-3B}} & GRPO & 48.22&37.98&64.60&45.19& 35.21&11.62&67.74& 44.37  \\

        \rowcolor{sealbg}\cellcolor{white} -3B&GRPO + SEAL& 47.51&44.21 &63.95& 47.31& 44.56&17.99&72.98& 48.36\\

        \midrule
        \multirow{2}{*}{\makecell{Qwen2.5\\-7B}} & GRPO & 46.08& 40.12 & 63.78& 45.55&38.77&16.32&65.72&45.19 \\

        \rowcolor{sealbg}\cellcolor{white} -7B&GRPO + SEAL & 46.08 & 40.46& 65.20&46.53 &41.58&17.72&69.76&46.76\\

        \bottomrule
    \end{tabular}}
\end{table}

\definecolor{sealbg}{HTML}{EEF4FF}
\begin{table}[t]
\footnotesize
    \centering
    \caption{Statistics of GRPO performance with and without SEAL on AlfWorld and WebShop for Qwen2.5- and Llama3.2-series models. Results are the average of three runs.}
    \label{tab:statisitc} 
    \scalebox{0.88}{
    \begin{tabular}{cc|ccccccc|cc}
    \toprule 
       \multirow{2}{*}{Model} & \multirow{2}{*}{Methods} & \multicolumn{7}{c|}{\textsc{AlfWorld}} & \multicolumn{2}{c}{\textsc{WebShop}}\\
       & &Pick& Look& Clean &Heat& Cool& Pick2& {\bf Avg.} & Score& {\bf Succ.}  \\
       
        \midrule
       \multirow{2}{*}{\makecell{Llama3.2\\-3B}} & GRPO & 97.17 & 99.95 & 96.44 & 77.21 & 90.29 & 86.68& 91.29 \scriptsize{$\pm$ 2.35} &87.06& 67.11\scriptsize{$\pm$4.11}\\

        \rowcolor{sealbg}\cellcolor{white} 3B &GRPO + SEAL& 95.73 & 93.50 & 93.97 & 85.17 & 98.50 & 98.51&94.23\scriptsize{$\pm$ 2.57} & 89.81 & 76.95\scriptsize{$\pm$3.12}\\
       
        \midrule
        \multirow{2}{*}{\makecell{Qwen2.5\\-7B}} & GRPO &  99.95 & 78.89 & 92.77 & 88.14 & 82.42 & 51.03&82.20\scriptsize{$\pm$ 2.57} &89.63& 76.95 \scriptsize{$\pm$ 1.81}\\

        \rowcolor{sealbg}\cellcolor{white} 7B &GRPO + SEAL & 94.21 & 85.32 & 99.61 & 86.28 & 85.35 & 49.63 & 83.40\scriptsize{$\pm$ 2.76} &90.26&77.73\scriptsize{$\pm$ 1.17}\\
        
        \bottomrule
    \end{tabular}}
\end{table}

\begin{figure}[t!]
    \centering
    \includegraphics[width=0.24\linewidth]{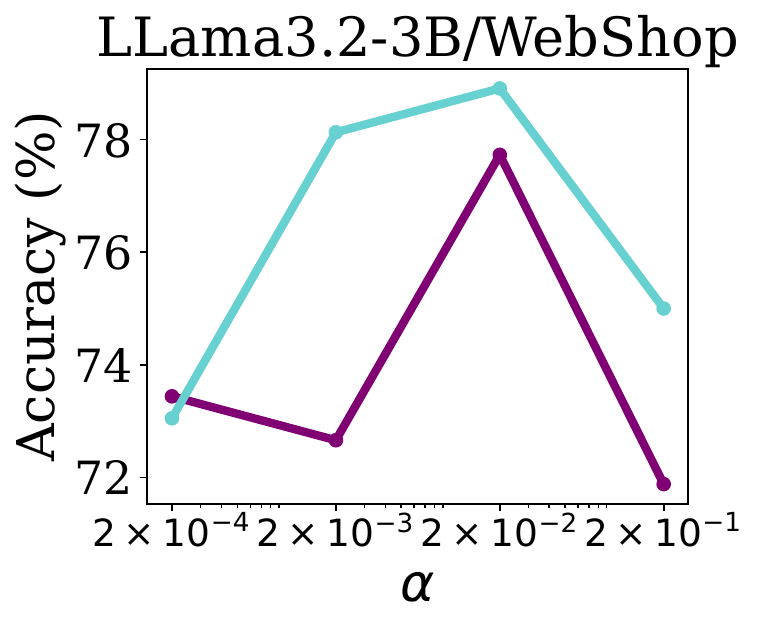}
    \includegraphics[width=0.24\linewidth]{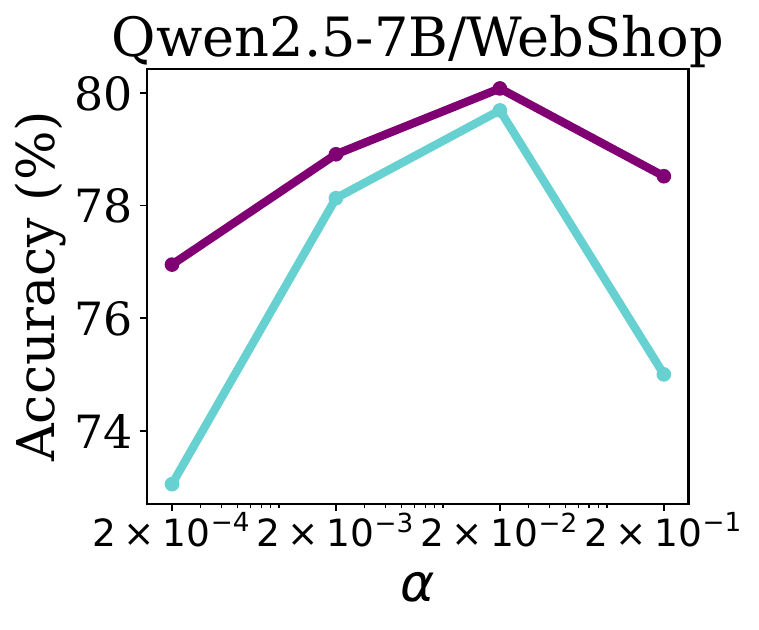}
    \includegraphics[width=0.24\linewidth]{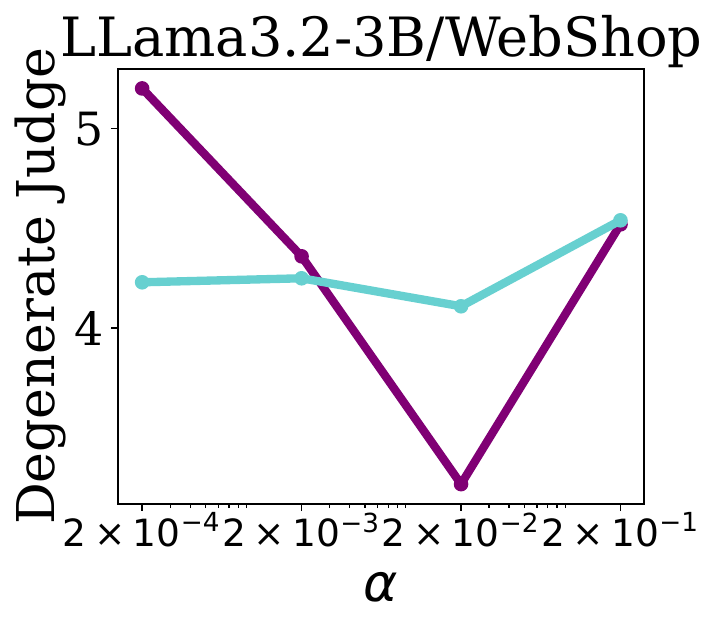}
    \includegraphics[width=0.24\linewidth]{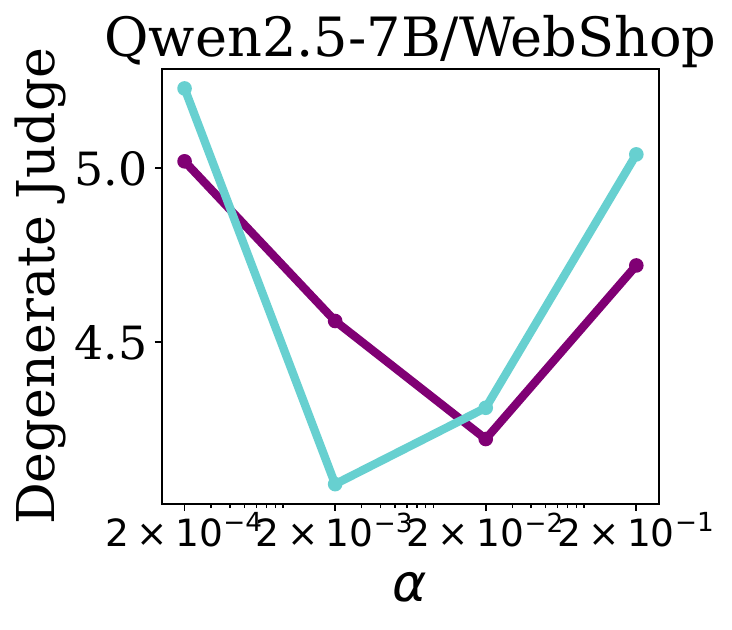}
    \caption{Accuracy and LLM-as-a-Judge score on degenerate patterns of \textcolor[HTML]{800074}{GRPO+SEAL} and \textcolor[HTML]{67d0d0}{GIGPO+SEAL} with different hyperparameters.}
    \label{fig:hyperparameter}
\end{figure}

\subsection{Statistics over three runs}
\label{apd:sec:statistics}
Table~\ref{tab:statisitc} reports the mean performance over three random seeds. The results show that the gains from SEAL remain consistent after averaging across runs. On Llama3.2-3B, adding SEAL improves the AlfWorld average from \(91.29\) to \(94.23\), while also increasing the WebShop score from \(87.06\) to \(89.81\) and the success rate from \(67.11\) to \(76.95 \). On Qwen2.5-7B, SEAL raises the AlfWorld average from \(82.20\) to \(83.40\), and also improves WebShop from \(89.63\) to \(90.26\) in score and from \(76.95\) to \(77.73\) in success rate. Overall, the multi-run statistics confirm that the benefits of SEAL are stable and reproducible, rather than arising from a favorable single run.

\subsection{Performance on search-augmented QA.}
\label{apd:sec:search}

Table~\ref{tab:search-apd} reports the performance of GRPO with and without SEAL on the search-augmented QA benchmark. Across both model backbones, SEAL consistently improves the average performance. These results suggest that the benefits of SEAL are not limited to interactive agent environments such as AlfWorld and WebShop, but also extend to search-augmented QA, where multi-step retrieval and reasoning are both required. Overall, the results provide further evidence that reducing harmful gradient interference can improve policy optimization across a broader range of agent-style tasks.

\section{Limitations and Future Work}
\label{apd:sec:limitations}

Our study is conducted on three representative agent tasks---AlfWorld, WebShop, and search-augmented QA---and on two widely used model families, the Llama and Qwen series. While these settings already cover multiple environments, reward structures, and backbone scales, they do not exhaust the full space of agent RL problems. It would therefore be valuable to examine whether the same entropy dynamics arise as broadly in other domains, such as longer-horizon agents, richer tool-use settings, or multimodal environments. At the same time, the consistency of our observations across all of our current settings suggests that cyclical entropy eruption is unlikely to be an isolated artifact of a single benchmark or model family.
In addition, our proposed mitigation, SEAL, is designed as a practical intervention during RL training. This choice matches the focus of this work, namely, to diagnose the training dynamics of entropy eruption and to show that explicitly reducing harmful representation overlap can alleviate it. However, it remains an open question whether similar benefits could be obtained even earlier in the training pipeline, for example, through pre-training, instruction tuning, or mid-training strategies that better separate correct and incorrect trajectories before RL begins. Exploring such earlier-stage interventions is a promising direction for future work.
More broadly, future work could investigate whether improved curricula, trajectory filtering, reward design, or representation regularization can further stabilize agent RL while preserving useful exploration. We hope this work provides both an empirical starting point and a theoretical lens for studying these questions.

\section{Broader Impact}

This work studies the training dynamics of reinforcement learning for LLM agents, with the goal of improving the stability and reliability of agent training. As discussed in the paper, agent systems are increasingly being deployed in settings such as web interaction, software engineering, scientific assistance, and other long-horizon decision-making tasks, making it important to understand not only how to improve benchmark performance, but also how and why training can become unstable. Our analysis of cyclical entropy eruption and our proposed mitigation, \textsc{SEAL}, are intended to contribute to safer and more controllable agent optimization by reducing pathological behaviors such as likelihood collapse on valid trajectories, degenerate generation patterns, and severe reward instability. 

At the same time, more capable and stable agent RL can have dual-use implications. Techniques that improve the robustness of agent training may also make it easier to build agents that are more persistent, autonomous, and effective in real-world environments. Such systems could be beneficial in domains like scientific discovery, productivity assistance, and automation. 

We therefore view this work as part of a larger effort toward principled and responsible agent development. A better understanding of training dynamics can support improved monitoring, evaluation, and intervention during RL, and future work should continue to study how stability-oriented methods interact with broader concerns such as alignment, robustness, and misuse prevention.


\newpage
\section*{NeurIPS Paper Checklist}

\begin{enumerate}

\item {\bf Claims}
    \item[] Question: Do the main claims made in the abstract and introduction accurately reflect the paper's contributions and scope?
    \item[] Answer: \answerYes{} 
    \item[] Justification: We explicitly mentioned our core contributions: We identify a unique entropy dynamics in agent RL and analyze its cause. At last, we propose an easy-to-implement approach based on our understanding of this dynamics.

\item {\bf Limitations}
    \item[] Question: Does the paper discuss the limitations of the work performed by the authors?
    \item[] Answer: \answerYes{} 
    \item[] Justification: We discussed limitations in Appendix~\ref{apd:sec:limitations}.

\item {\bf Theory assumptions and proofs}
    \item[] Question: For each theoretical result, does the paper provide the full set of assumptions and a complete (and correct) proof?
    \item[] Answer: \answerYes{} 
    \item[] Justification: We clearly exhibit the formulation of question in each lemma. For specific conditions that should be satisfied, we discussed why it is satisfied for real experiments in context. 

    \item {\bf Experimental result reproducibility}
    \item[] Question: Does the paper fully disclose all the information needed to reproduce the main experimental results of the paper to the extent that it affects the main claims and/or conclusions of the paper (regardless of whether the code and data are provided or not)?
    \item[] Answer: \answerYes{} 
    \item[] Justification: We provided implementation details in Appendix~\ref{apd:sec:implementation-details}. Meanwhile, we provide data and code.

\item {\bf Open access to data and code}
    \item[] Question: Does the paper provide open access to the data and code, with sufficient instructions to faithfully reproduce the main experimental results, as described in supplemental material?
    \item[] Answer: \answerYes{} 
    \item[] Justification: We released the data and code, and provide command scripts to directly run it. 

\item {\bf Experimental setting/details}
    \item[] Question: Does the paper specify all the training and test details (e.g., data splits, hyperparameters, how they were chosen, type of optimizer) necessary to understand the results?
    \item[] Answer: \answerYes{} 
    \item[] Justification: As shown in Appendix~\ref{apd:sec:implementation-details} and command scripts in code, we include all necessary implementation details and data sources.

\item {\bf Experiment statistical significance}
    \item[] Question: Does the paper report error bars suitably and correctly defined or other appropriate information about the statistical significance of the experiments?
    \item[] Answer: \answerYes{} 
    \item[] Justification: We include statistics in Appendix~\ref{apd:sec:statistics}. Meanwhile, we fix the random seed for reproducibility. Moreover, we conducted hyperparameter searches and experiments across multiple tasks/backbones, which can validate the consistent effectiveness of our approach. 

\item {\bf Experiments compute resources}
    \item[] Question: For each experiment, does the paper provide sufficient information on the computer resources (type of compute workers, memory, time of execution) needed to reproduce the experiments?
    \item[] Answer:  \answerYes{}
    \item[] Justification: we include information in Appendix~\ref{apd:sec:implementation-details}.
    
\item {\bf Code of ethics}
    \item[] Question: Does the research conducted in the paper conform, in every respect, with the NeurIPS Code of Ethics \url{https://neurips.cc/public/EthicsGuidelines}?
    \item[] Answer: \answerYes{} 
    \item[] Justification: We followed the terms in NeurIPS code of ethics.

\item {\bf Broader impacts}
    \item[] Question: Does the paper discuss both potential positive societal impacts and negative societal impacts of the work performed?
    \item[] Answer: \answerYes{} 
    \item[] Justification: We discussed potential impacts of our paper in introduction, approach, and limitation part.
    
\item {\bf Safeguards}
    \item[] Question: Does the paper describe safeguards that have been put in place for responsible release of data or models that have a high risk for misuse (e.g., pre-trained language models, image generators, or scraped datasets)?
    \item[] Answer: \answerNA{} 
    \item[] Justification: We do not release any new data source or pretrained language models.

\item {\bf Licenses for existing assets}
    \item[] Question: Are the creators or original owners of assets (e.g., code, data, models), used in the paper, properly credited and are the license and terms of use explicitly mentioned and properly respected?
    \item[] Answer: \answerYes{} 
    \item[] Justification: We cited every data source, pretrained models and code repository we based on. The license and terms of use are properly respected.

\item {\bf New assets}
    \item[] Question: Are new assets introduced in the paper well documented and is the documentation provided alongside the assets?
    \item[] Answer: \answerNA{} 
    \item[] Justification: We do not release any new assets.

\item {\bf Crowdsourcing and research with human subjects}
    \item[] Question: For crowdsourcing experiments and research with human subjects, does the paper include the full text of instructions given to participants and screenshots, if applicable, as well as details about compensation (if any)? 
    \item[] Answer: \answerNA{} 
    \item[] Justification: We do not conduct crowdsourcing experiments.

\item {\bf Institutional review board (IRB) approvals or equivalent for research with human subjects}
    \item[] Question: Does the paper describe potential risks incurred by study participants, whether such risks were disclosed to the subjects, and whether Institutional Review Board (IRB) approvals (or an equivalent approval/review based on the requirements of your country or institution) were obtained?
    \item[] Answer: \answerNA{} 
    \item[] Justification: Our paper does not involve crowdsourcing nor research with human subjects.

\item {\bf Declaration of LLM usage}
    \item[] Question: Does the paper describe the usage of LLMs if it is an important, original, or non-standard component of the core methods in this research? Note that if the LLM is used only for writing, editing, or formatting purposes and does \emph{not} impact the core methodology, scientific rigor, or originality of the research, declaration is not required.
    \item[] Answer: \answerNA{} 
    \item[] Justification: The core method development in this research does not involve LLMs.

\end{enumerate}

\end{document}